\documentclass[11pt]{article}
\usepackage[utf8]{inputenc}
\usepackage[left=2.0cm,right=2.0cm,top=2.0cm,bottom=2.0cm]{geometry}
\usepackage[T1]{fontenc} 
\usepackage{lmodern}
\usepackage{microtype}

\usepackage[english]{babel}              
\usepackage{graphicx}
\usepackage{mathtools,amsfonts,amssymb,amsthm}
\mathtoolsset{showonlyrefs}
\usepackage{xspace}
\usepackage{geometry}
\usepackage{graphicx}
\usepackage[overload]{empheq}
\usepackage{dsfont}
\usepackage{stmaryrd}
\usepackage{tikz}
\usepackage{minted}
\usepackage{array}
\usepackage{booktabs}
\usepackage{mathrsfs}
\usepackage{enumerate}
\usepackage{enumitem}
\usepackage{subfig}
\usepackage{float}

\usepackage{csquotes}

\usepackage[
  backend=bibtex,
  style=numeric,
  sorting=none,
  natbib=true,
  maxcitenames=2,
  mincitenames=1,
  maxbibnames=99
]{biblatex}

\usepackage{hyperref}
\hypersetup{
colorlinks = true,
urlcolor = blue,
linkcolor = blue,
citecolor = blue,
}

\makeatletter

\newif\ifappendixtoc
\appendixtocfalse

\let\oldaddcontentsline\addcontentsline
\renewcommand{\addcontentsline}[3]{%
  \ifappendixtoc
    \def\tempa{#1}%
    \def\tempb{toc}%
    \ifx\tempa\tempb
      \oldaddcontentsline{atoc}{#2}{#3}%
    \else
      \oldaddcontentsline{#1}{#2}{#3}%
    \fi
  \else
    \oldaddcontentsline{#1}{#2}{#3}%
  \fi
}

\newcommand{\appendixtableofcontents}{%
  \section*{Appendix contents}%
  \@starttoc{atoc}%
}
\newtheorem{theorem}{Theorem}
\newtheorem{proposition}{Proposition}
\newtheorem{lemma}{Lemma}
\newtheorem{remark}{Remark}
\newtheorem{definition}{Definition}

\usepackage{amsmath,amssymb,euscript ,yfonts,psfrag,latexsym,dsfont,graphicx,bbm,color,amstext,wasysym,epsfig,subfig,parskip,textcomp}

\usepackage{pgfplots}
\usepackage{mathrsfs}
\usetikzlibrary{arrows}

\DeclareMathOperator*{\argmin}{arg min}

\usepackage{xcolor}

\newcommand{\R}{\mathbb{R}}
\newcommand{\E}{\mathbb{E}}
\newcommand{\Err}{\operatorname{Err}}
\DeclareMathOperator{\Cov}{Cov}

\makeatletter
\def\thm@space@setup{%
    \thm@preskip=6pt plus 2pt minus 2pt 
    \thm@postskip=1pt plus 2pt minus 2pt 
}
\makeatother

\newcommand{\T}{\mathbb{T}}

\newcommand{\dd}{\mathrm{d}}
\newcommand{\KL}{\mathrm{KL}}
\newcommand{\Ent}{\mathrm{Ent}}

\newcommand{\diver}{\nabla\!\cdot}

\newcommand{\Lip}{\operatorname{Lip}}
\newcommand{\osc}{\operatorname{osc}}

\title{
Convergence rates for generative drifting flows: \\ fixed-scale obstructions and multihead acceleration
}

\author{%
Arthur St\'ephanovitch\\[-0.2em]
{\small Centre de recherche en économie et statistique}\\
{\small ENSAE, IP Paris}\\
{\small\href{mailto:arthur.stephanovitch@ensae.fr}
{\texttt{arthur.stephanovitch@ensae.fr}}}
\and
Eddie Aamari\\[-0.2em]
{\small Département de Mathématiques et Applications}\\
{\small CNRS, École Normale Supérieure, PSL}\\
{\small\href{mailto:eddie.aamari@ens.fr}
{\texttt{eddie.aamari@ens.fr}}}
}

\date{\today}

\pgfplotsset{compat=1.18}

\begin{document}
\maketitle

\begin{abstract}%
Drifting models offer a promising route to faster generative AI: they perform gradual transport during training, while generating new samples in a single step. This paper asks whether the underlying drifting process can converge rapidly to a target distribution under ideal conditions, before finite-data or optimization effects are introduced. We show that its convergence rate depends critically on how it handles spatial scale. With a single fixed resolution, fine-scale features of the target can become nearly invisible, leading to extremely slow convergence. We introduce a multihead approach that combines scale-normalized information across a continuum of resolutions. We prove that this multihead approach restores exponential convergence near standard reference distributions. These results identify fixed resolution as a key bottleneck and provide a simple route to faster one-step generative models.
\end{abstract}

\section*{Contents}

\section{Introduction}
\label{sec:introduction}

Modern generative models, such as diffusion and flow-matching, typically generate samples by numerically solving a stochastic or ordinary differential equation over many sequential steps. Drifting models~\cite{deng2026drifting} take a different approach: they shift this iterative computation to training, while generation requires only a single forward pass.  This separation between training and sampling makes the generation phase much faster. During training, the successive neural networks transport the model distribution along a path toward the target distribution. At inference time, however, a new sample does not need to retrace this path: it is generated by applying a learned transport map once.

Informally, drifting constructs this transport by repeatedly comparing the current model distribution with the target distribution. At each stage, a velocity field is obtained from a kernel-smoothed difference between the two distributions. The field attracts model particles toward regions of high target density while simultaneously repelling particles from one another. A neural network is then trained to approximate the corresponding short transport step, and these steps are accumulated during training.

Before studying finite-sample estimation, neural-network approximation, or optimization error, one must first determine whether the underlying population dynamics itself approaches the target at a fast rate. This leads to the central question of the paper:
\begin{quote}
\emph{Does the idealized drifting dynamics, in the infinite-data limit with no algorithmic error, converge to the target distribution exponentially fast?}
\end{quote}
This question matters for both theoretical and algorithmic reasons. The population flow is the best-case version of the method: if this idealized evolution is intrinsically slow, then increasing the dataset size or enlarging the neural network  cannot remove the underlying obstruction. Moreover, the convergence rate of the flow determines the time horizon over which a practical algorithm must approximate it. A slowly converging flow requires more discretization steps to reach a prescribed accuracy, creating a longer interval over which time-discretization and learning errors must be controlled and may accumulate or be amplified. 

We show that exponential convergence of the idealized drifting dynamics
depends critically on how the dynamics treats spatial scales.
 During training, the choice of the kernel determines which differences between the two distributions are visible to the velocity field. In particular, its bandwidth sets the spatial resolution of the comparison. A large bandwidth emphasizes broad differences in mass, but smooths out fine structure. A small bandwidth reveals finer local differences, but the resulting correction only acts over a small surrounding area. Classical drifting uses one fixed bandwidth and therefore observes the discrepancy between the model and the target at one prescribed resolution throughout the evolution.

The following informal statement summarizes our main theoretical result.

\begin{theorem}[Informal]
Classical fixed-bandwidth drifting dynamics does not admit a uniform exponential rate over natural classes of smooth targets. Averaging the drift over a continuum of bandwidths, while placing greater emphasis on the finer scales, removes this obstruction and yields local exponential convergence.
\end{theorem}

At the empirical level, this suggests a simple modification: compute the normalized drifting field at a finite collection of scales and average the resulting fields before the regression step. This finite-scale version does not inherit the continuum guarantee, but it approximates the same scale-averaging mechanism over the range of resolutions represented in a finite computation. The rest of the training procedure is unchanged.  The analysis conducted in this paper explains why the continuum modification changes the convergence rate.

We work with two spatial settings for complementary purposes. On the torus $\T^d$, we develop the main intuition, construct counterexamples, and prove convergence results in a setting where compactness removes technical issues related to tails and behavior at infinity and makes the relevant Fourier structure particularly transparent. This allows us to isolate the mechanisms governing the convergence rate without obscuring them with additional analytic difficulties. We then establish the corresponding convergence theory on $\R^d$, with a Gaussian reference measure and Ornstein-Uhlenbeck smoothing. While the torus provides a convenient and informative model setting, the Euclidean result is the one most directly relevant to the generative-modeling problem, where the state space is naturally unbounded.

\subsection{The drifting algorithm and its population limit}

Drifting generative models were introduced by
\cite{deng2026drifting} as a one-step alternative to diffusion and flow-based generative models.  Let
$\pi$ be a simple latent prior and let $f_\theta$ be a generator.  Its
output distribution is the pushforward
\begin{equation}
    q_\theta := (f_\theta)_\#\pi .
\end{equation}
The generator is trained so that $q_\theta$ matches the target law $\rho^\star$. Since the target is observed only through data, the algorithm uses a finite training set
\begin{equation}
    Y_1,\ldots,Y_n \sim \rho^\star,
    \qquad
    \widehat\rho_n^\star := \frac1n\sum_{j=1}^n \delta_{Y_j}.
\end{equation}
The training procedure is iterative. At each iteration, the current generator first produces a collection of samples. An attraction-repulsion field is then computed by comparing these generated samples with the data samples, and is used to construct drifted target points. Finally, the generator parameters are updated by fitting the generator to these targets. Repeating these steps progressively transports the generated distribution toward the data distribution.

\subsubsection{The finite-sample drifting step.}\label{sec:drifting-algorithm-step}
At training iteration $\ell$, the current generator is $f_{\theta_\ell}$.
For simplicity, let us first write the full-batch version with $n$ latent
samples
\begin{equation}
    Z_1,\ldots,Z_n\sim\pi,
    \qquad
    X_i^\ell := f_{\theta_\ell}(Z_i),
    \qquad
    \widehat q_{\ell,n}:=\frac1n\sum_{i=1}^n\delta_{X_i^\ell}.
\end{equation}
The points $Y_j$ are the data samples, while
$X_j^\ell$ are the generated samples at step $\ell$ of the training.  Given a nonnegative kernel
$K_\sigma$, the algorithm compares a generated point $x$ with these two clouds
through the normalized weights
\begin{equation}
    a_j^\ell(x)
    :=
    \frac{K_\sigma(x,Y_j)}{\sum_{k=1}^n K_\sigma(x,Y_k)},
    \qquad
    r_j^\ell(x)
    :=
    \frac{K_\sigma(x,X_j^\ell)}{\sum_{k=1}^n K_\sigma(x,X_k^\ell)}.
    \label{eq:empirical-drifting-weights}
\end{equation}
The corresponding local attraction fields
are
\begin{equation}
    m_{\widehat\rho_n^\star}^\sigma(x)
    :=
    \sum_{j=1}^n a_j^\ell(x)(Y_j-x),
    \qquad
    m_{\widehat q_{\ell,n}}^\sigma(x)
    :=
    \sum_{j=1}^n r_j^\ell(x)(X_j^\ell-x),
    \label{eq:empirical-barycentres}
\end{equation}
and the empirical drifting field is the difference between these two fields:
\begin{align}
    \widehat V_{\ell,n}^\sigma(x)
    &:=
    m_{\widehat\rho_n^\star}^\sigma(x)
    -
    m_{\widehat q_{\ell,n}}^\sigma(x)\nonumber\\
    &=
    \sum_{j=1}^n a_j^\ell(x)(Y_j-x)
    -
    \sum_{j=1}^n r_j^\ell(x)(X_j^\ell-x).
    \label{eq:empirical-drifting-field}
\end{align}
The first term pulls $x$ toward nearby data points, and the
second term pushes $x$ away from nearby generated points.
When the positive and negative empirical distributions coincide, these two
terms are identical and the drift vanishes.

The drift is then turned into a regression target.  For each generated sample
$X_i^\ell$, define
\begin{equation}
    \widehat X_i^{\ell,{\rm drift}}
    :=
        X_i^\ell+\tau_\ell\widehat V_{\ell,n}^\sigma(X_i^\ell)
    ,
    \label{eq:empirical-drifted-target}
\end{equation}
where $\tau_\ell>0$ is a drift
scale, often absorbed into the definition of the field.  The next network is
trained to map the same latent variables toward these new points:
\begin{equation}
    \theta_{\ell+1}\approx \argmin_\theta
    \frac1n\sum_{i=1}^n
    \|
        f_\theta(Z_i)-\widehat X_i^{\ell,{\rm drift}}
    \|^2.
    \label{eq:empirical-drifting-loss}
\end{equation}

In practice one does not usually use the whole dataset at every iteration.
One draws a mini-batch of data samples and a mini-batch of latent variables,
computes the same weights and targets inside the batch, and performs an SGD
step. In image experiments, the
kernel may also be evaluated after a feature extractor, possibly at several
scales and spatial locations.  These implementation choices change how the
finite sums in~\eqref{eq:empirical-drifting-weights} are computed, but they
do not change the basic structure: data attraction plus repulsion among generated particles, followed by regression of a transport map from latent samples to the new drifted points.

This explains the difference with diffusion or flow-based models.  Diffusion
and flow models move a sample through many inference-time updates.  Drifting
models instead use the many optimization steps of training to move the
pushforward distribution $q_{\theta_\ell}$.  After training, generation is still a single forward pass $z\mapsto f_{\theta_L}(z)$.  The intermediate parameters $\theta_1,\ldots,\theta_{L-1}$ are used only during training and are not traversed when a new sample is generated.

\subsubsection{The population version.}
This paper studies the idealized dynamics obtained by removing
three finite-sample effects: the finite dataset, the mini-batch
approximation, and the finite optimization step.  Formally, the empirical
measures $\widehat\rho_n^\star$ and $\widehat q_{\ell,n}$ are replaced by
the target density $\rho^\star$ and the current density $q$, and the discrete training iteration is replaced by a continuous
transport equation.

For a probability density $\rho$, define the population local attraction field
\begin{equation}
    m_\rho^\sigma(x)
    :=
    \frac{\int  K_\sigma(x,y)(y-x)\rho(y)\,\dd y}
    {\int K_\sigma(x,y)\rho(y)\,\dd y}.
    \label{eq:kernel-barycentre}
\end{equation}
The population analogue of~\eqref{eq:empirical-drifting-field} is then
\begin{equation}
    V_{\rho^\star,q}^\sigma(x)
    :=
    m_{\rho^\star}^\sigma(x)-m_q^\sigma(x).
    \label{eq:population-drifting-field}
\end{equation}

The population dynamics associated with drifting is obtained by transporting
the current model density along this field.  If $q_t$ denotes the idealized
model density at training time $t$, then the formal velocity is
\begin{equation}
    v_t(x)=V_{\rho^\star,q_t}^\sigma(x),
\end{equation}
and $q_t$ evolves through the continuity equation
\begin{equation}
    \partial_t q_t+\diver(q_t v_t)=0.
    \label{eq:abstract-drifting-continuity-equation}
\end{equation}
The goal of this paper is to study the convergence of the solution $q_t$ to~\eqref{eq:abstract-drifting-continuity-equation} toward the target distribution $\rho^\star$, to identify the limitations of the fixed-bandwidth dynamics, and to introduce a multihead modification with improved convergence rates.

\subsubsection{Comparing convergence rates of continuous flows}
\label{sec:pde-rate-comparison}

Before comparing two drifting equations, one has to specify what is meant by a
``rate of convergence'' for a PDE.  If $(\rho_t)_{t\geq0}$ converges to $\rho^\star$ and $\phi:[0,\infty)\to[0,\infty)$ is increasing, then
$
    \widetilde \rho_t=\rho_{\phi(t)}
$
has the same image as $(\rho_t)$, but a different speed.  With a fast enough time change, a logarithmic decay can be made to look exponential.  Thus, a rate is meaningful only after specifying the time parameter with respect to which the dynamics is defined.

For an autonomous population rule, this time parameter is part of the rule.    A population drifting method is specified by an autonomous vector-field rule.  Given a target measure $\rho^\star$ and a current density $q$, a
formulation $\mathsf A$ assigns a velocity field $
    V_{\mathsf A}[\rho^\star,q],$
and the corresponding population equation is
\begin{equation}
    \partial_t q_t
    + \nabla\!\cdot\bigl(q_t V_{\mathsf A}[\rho^\star,q_t]\bigr)=0.
    \label{eq:autonomous-population-drifting}
\end{equation}
Let $S_t^{\mathsf A,\rho^\star}$ denote the solution map, i.e.
\[
    q_t = S_t^{\mathsf A,\rho^\star}(q_0).
\]
Because $V_{\mathsf A}$ does
not depend explicitly on time, the maps $(S_t^{\mathsf A,\rho^\star})_{t\geq0}$ have the
semigroup property
\begin{equation}
    S_{t+s}^{\mathsf A,\rho^\star}
    =
    S_t^{\mathsf A,\rho^\star}\circ S_s^{\mathsf A,\rho^\star}.
    \label{eq:semigroup-property}
\end{equation}
The semigroup property rules out arbitrary time changes.  Suppose that $
    T_t := S_{\phi(t)}^{\mathsf A,\rho^\star}$,
then $(T_t)_{t\geq0}$ is again a semigroup only if $
    \phi(t+s)=\phi(t)+\phi(s).$
Under the usual continuity assumption, this forces
\[
    \phi(t)=ct
\]
for some constant $c\geq0$.  Thus the only reparametrizations compatible with
the autonomous-flow structure are constant changes of units.  A nonlinear time
change produces instead
\[
    \frac{d}{dt}q_{\phi(t)}
    =
    \phi'(t)F(q_{\phi(t)}),
\]
where $F$ denotes the autonomous right-hand side of the original PDE.  This is a
different equation: it is non-autonomous and contains a prescribed
time-dependent multiplier. Let us make the convention explicit.

\begin{definition}[Convergence rates for semigroups]
Let $\mathcal P$ be a target class, let $q_0$ be the initialization, and let
$\Err$ be an error functional such as relative entropy or Wasserstein distance.
For an autonomous formulation $\mathsf A$, define the worst-case population-flow
error by
\begin{equation}
    R_{\mathsf A}(t)
    :=
    \sup_{\rho^\star\in\mathcal P}
    \Err\bigl(S_t^{\mathsf A,\rho^\star}q_0,\rho^\star\bigr).
    \label{eq:population-flow-rate-definition}
\end{equation}
A convergence rate for the PDE is an estimate on $R_{\mathsf A}(t)$ in the
canonical time of the semigroup $S_t^{\mathsf A,\rho^\star}$.
\end{definition}

With this convention, a natural benchmark is exponential convergence. For many gradient flows, such a rate follows from suitable convexity assumptions or from an inequality that controls the decay of the energy~\citep{ambrosio2008gradient,bakry2014analysis}. A bound of the form
\[
    R_{\mathsf A}(t) \leq C e^{-ct}
\]
means that the time required to reach an error level $\varepsilon$ grows only logarithmically with $1/\varepsilon$. For a stable numerical discretization, this is the behavior one would hope for in practice. Indeed, the more time the continuous flow needs to reach a prescribed accuracy, the more discretization steps are required and the longer the time interval over which discretization, statistical, and map-learning errors must be controlled. We therefore ask whether drifting converges exponentially in its canonical time.

\begin{remark}[Drifting is not a Wasserstein gradient flow]
Several recent works relate drifting to Wasserstein gradient flows or to approximations of such flows~\citep{cao2026gradient,turan2026secretly,gretton2026wasserstein}. These interpretations do not, however, provide an exact gradient-flow formulation of the autonomous drifting equation studied here. Depending on the formulation, the smoothing is applied differently or the velocity field is frozen during one update.

In Appendix~\ref{app:no-hidden-energy}, we prove that, at any positive smoothing scale, the Gaussian Drifting field cannot be the first variation of any $C^2$ functional on smooth positive densities. Consequently, the corresponding drifting dynamics are not the Wasserstein gradient flow of such an energy. Exponential convergence therefore does not follow from the standard gradient-flow theory and must be established, or ruled out, by analyzing the drifting equation directly.
\end{remark}

\subsubsection{Notation}\label{subsec:notation}

Throughout the paper, $d\geq1$ denotes the ambient dimension and
$\mathbb N_0:=\{0,1,2,\ldots\}$. For a multi-index
$\kappa=(\kappa_1,\ldots,\kappa_d)\in\mathbb N_0^d$, we write
$|\kappa|:=\kappa_1+\cdots+\kappa_d$ and
$\partial^\kappa:=\partial_1^{\kappa_1}\cdots\partial_d^{\kappa_d}$.
The Euclidean norm and inner product are denoted by $|\cdot|$ and
$\cdot$, respectively. The flat torus is
\[
    \T^d:=(\R/2\pi\mathbb Z)^d,
\]
and is always equipped with normalized Lebesgue measure; in particular,
$\int_{\T^d}1\,\dd x=1$. On $\R^d$, $\dd x$ denotes ordinary Lebesgue
measure, while $\gamma$ denotes the standard Gaussian probability measure.
The symbol $d_D$ denotes the Euclidean distance when $D=\R^d$ and the
geodesic distance when $D=\T^d$.

For simplicity, throughout the manuscript we identify probability densities
with the probability measures they represent whenever the reference measure is
clear. Thus, for example, a density $q$ on $\T^d$ also denotes the measure
$q\,\dd x$, and a density $q$ relative to $\gamma$ also denotes
$q\,\dd\gamma$. We write $T_\#\mu$ for the pushforward of a measure $\mu$
by a measurable map $T$, $\delta_x$ for the Dirac mass at $x$, and
$\Pi(\mu,\nu)$ for the set of couplings of $\mu$ and $\nu$. If
$\mu\ll\nu$, then
\begin{equation}
    \KL(\mu\mid\nu)
    :=
    \int \log\!\left(\frac{\dd\mu}{\dd\nu}\right)\,\dd\mu,
    \label{eq:notation-kl}
\end{equation}
with the value $+\infty$ when $\mu\not\ll\nu$. The quadratic Wasserstein
distance is
\begin{equation}
    W_2^2(\mu,\nu)
    :=
    \inf_{\pi\in\Pi(\mu,\nu)}
    \int d_D(x,y)^2\,\dd\pi(x,y).
    \label{eq:notation-w2}
\end{equation}
For a positive density $\rho$, its score is $\nabla\log\rho$.

For an integrable function $f$ and a probability measure $\mu$, we use
\[
    \mathbb E_\mu[f]:=\int f\,\dd\mu,
    \qquad
    \operatorname{Var}_\mu(f)
    :=\mathbb E_\mu\!\left[(f-\mathbb E_\mu f)^2\right],
\]
and, for vector-valued $F,G$ of compatible dimensions,
\[
    \operatorname{Cov}_\mu(F,G)
    :=
    \mathbb E_\mu[F\otimes G]
    -\mathbb E_\mu[F]\otimes\mathbb E_\mu[G].
\]
We write $\|\cdot\|_{L^p(\mu)}$ for the usual $L^p$ norm,
$\Lip(f)$ for the Lipschitz seminorm, and
$\osc(f):=\sup f-\inf f$. Convolution on either $\R^d$ or $\T^d$ is
normalized according to the corresponding reference measure:
\begin{equation}
    (f*g)(x):=\int f(x-y)g(y)\,\dd y.
    \label{eq:notation-convolution}
\end{equation}

On $\T^d$ we use the Fourier convention
\begin{equation}
    e_n(x):=e^{i n\cdot x},
    \qquad
    \widehat f(n):=\int_{\T^d}f(x)e^{-i n\cdot x}\,\dd x,
    \qquad n\in\mathbb Z^d,
    \label{eq:notation-fourier}
\end{equation}
so that $f=\sum_{n\in\mathbb Z^d}\widehat f(n)e_n$ whenever the series is
classically meaningful. The heat semigroup is denoted by
$P_s:=e^{s\Delta}$; hence $P_se_n=e^{-s|n|^2}e_n$ on $\T^d$. In the
Gaussian-reference setting, $Q_s:=e^{sL_\gamma}$ denotes the
Ornstein-Uhlenbeck semigroup with generator
$L_\gamma:=\Delta-x\cdot\nabla$.

We use the standard H\"older spaces $C^{k,\alpha}$, with $0<\alpha\le1$.
For statements with a single, possibly non-integer regularity index, set
\begin{equation}
    m_\eta:=\max\{m\in\mathbb N_0:m<\eta\}.
\end{equation}
Write $\mathcal C^\eta:=C^{m_\eta,\eta-m_\eta}$; at integer indices,
this means Lipschitz regularity of the preceding derivative, not a Sobolev
space. If $f=(f_1,\ldots,f_r):D\to\R^r$, where
$D\in\{\R^d,\T^d\}$, define
\begin{equation}\label{eq:Holder-norm}
    \|f\|_{\mathcal C^\eta(D)}
    :=
    \max_{1\le j\le r}\left\{
        \sum_{0\le |\kappa|\le m_\eta}
        \|\partial^\kappa f_j\|_{L^\infty(D)}
        +
        \sum_{|\kappa|=m_\eta}\sup_{x\ne y}
        \frac{
            |\partial^\kappa f_j(x)-\partial^\kappa f_j(y)|
        }{
            d_D(x,y)^{\eta-m_\eta}
        }
    \right\}.
\end{equation}
For vector and tensor fields, the norm is understood componentwise, and the
domain is omitted from the notation when it is clear from context. In
particular, $\mathcal C^1=C^{0,1}$, $\mathcal C^2=C^{1,1}$, and
$\mathcal C^3=C^{2,1}$. We use the explicit $C^{k,1}$ notation in the local
convergence theorems and their proofs. We write
\begin{equation}
    \mathcal C_C^\eta(D;\R^r)
    :=
    \left\{
        f:D\to\R^r:
        \|f\|_{\mathcal C^\eta(D)}\le C
    \right\},
\end{equation}
and omit $\R^r$ when $r=1$.

Finally, $A\lesssim_\theta B$ means that $A\le C_\theta B$ for a constant
$C_\theta$ depending only on the displayed parameters $\theta$;
$A\gtrsim_\theta B$ has the analogous meaning, and
$A\asymp_\theta B$ means that both inequalities hold. Constants denoted by
$c$ or $C$ may change from line to line, and their subscripts record the
parameters on which they are allowed to depend.

\subsection{Contributions and paper organization}
\label{subsec:main-contributions}
\subsubsection{Main results}

To the best of our knowledge, this work provides the first explicit
convergence-rate results for the full nonlinear population dynamics underlying
drifting. Previous analyses identified a slowdown through linearized or
frozen-field calculations; here we establish rate statements for the
nonlinear equation itself, uniformly over natural classes of targets. The population dynamics assumes 
infinite data and no approximation or optimization error. Slow convergence at
this level is therefore an intrinsic limitation of classical drifting, rather
than an artifact of implementation.

\paragraph{Classical fixed-bandwidth drifting is intrinsically slow.}
We first prove quantitative lower bounds showing that a single fixed bandwidth
cannot correct discrepancies at all spatial scales rapidly. For Gaussian
drifting on the $d$-dimensional torus,
Proposition~\ref{prop:nonlinear-log-lower-bound} shows that, for every dimension
$d\geq1$ and every regularity exponent $\beta>0$, the worst-case error over a bounded
$\mathcal C^\beta$ class of target potentials satisfies (with
$\rho_a^\star=e^a/\int_{\T^d}e^a\,\dd x$ and $q_t^a$ its associated flow)
\begin{equation}
    \sup_{a\in\mathcal C_C^\beta(\T^d)}
    \KL(q_t^a\mid\rho_a^\star)
    \gtrsim(\log t)^{-\beta}.
    \label{eq:main-results-gaussian-summary}
\end{equation}
Thus no polynomial, and hence no
exponential, convergence rate can hold uniformly over finite-smoothness target
classes. For the original Laplace displacement field,
Proposition~\ref{prop:true-laplace-polynomial-lower-bound} gives, for every
$\beta>0$,
\begin{equation}
\begin{aligned}
    \sup_{a\in\mathcal C_C^\beta(\T^d)}
    \KL(q_t^a\mid\rho_a^\star)
    \gtrsim\left(1+\frac{t}{\tau^{d+1}}\right)^{-2\beta/(d+1)}.
\end{aligned}
\label{eq:main-results-laplace-summary}
\end{equation}
The Laplace field is less severely affected than the Gaussian field, but no
uniform rate faster than polynomial can hold over these target classes.

Together, these results give a nonlinear and quantitative explanation for why
classical fixed-bandwidth drifting can be prohibitively slow: sufficiently fine
discrepancies remain only weakly corrected, regardless of how accurately the
population field is evaluated. 

\paragraph{Multihead drifting restores fast convergence.}
We then introduce \emph{multihead drifting}, an autonomous dynamics that
averages ordinary drifting fields over a range of smoothing scales. The method
allows the flow to correct coarse and fine discrepancies simultaneously,
without committing to a single bandwidth and without annealing the bandwidth
over time. At the empirical level, it can be implemented by averaging standard
drifting fields evaluated at a finite collection of scales.

Theorems~\ref{thm:local-torus-heat} and~\ref{thm:local-ou} show that this
modification changes the qualitative convergence behavior. For smooth targets $\rho^\star$
sufficiently close to the uniform law on the torus or to the standard Gaussian
law on $\mathbb{R}^d$, the multihead flow converges exponentially: 
\begin{equation}
    \mathrm{KL}(q_t \mid \rho^\star)
    \lesssim e^{-2c t}.
    \label{eq:main-results-exponential-summary}
\end{equation}
This is the main contrast of the paper: fixed-bandwidth drifting suffers
logarithmic or polynomial worst-case obstructions, whereas the proposed
multihead dynamics restores exponential convergence, at least locally.

\paragraph{Evidence beyond the perturbative theory.}
Finally, the numerical study tests the same mechanism outside the local regime.
Section~\ref{sec:numerical-experiments} presents Experiment~1, a controlled
smooth target in $\R^2$ studied at both the population-PDE and finite-data
levels. We report in
\hyperref[app:additional-numerical-experiments]{Appendix~\ref*{app:additional-numerical-experiments}} two more complex finite-data
tests, a regularized curve-supported target and a high-dimensional stratified
target in $\R^{20}$. 
Across all three settings, a small collection of well-placed scales yields a
better distributional match, suggesting
that the multihead advantage persists for irregular and high-dimensional
targets.

\subsubsection{Related work}
\label{subsec:related-work}

Drifting models were introduced by~\citet{deng2026drifting} as a training-time transport mechanism for one-step generation.  Several recent papers explain the structure of this field.  For Gaussian kernels, the population drift is, up to a constant time scaling, the difference between the scores of Gaussian-smoothed target and model densities~\citep{lai2026unified,turan2026secretly,cao2026gradient}.  The long-short flow-map interpretation of~\citet{li2026longshort} gives a complementary derivation from a short terminal transport step.

For non-Gaussian displacement kernels, the geometry is more delicate. ~\citet{franz2026nonconservative} show that the original normalized field is not in general a gradient field and identify the Gaussian as the exceptional radial case; they also propose a sharp normalization that restores a potential structure. ~\citet{estebancasadevall2026kernel} replace Euclidean displacement directions by kernel gradients, obtaining a score-difference field for general characteristic kernels.  For the original displacement construction,~\citet{lee2026identifiability} prove identifiability for a companion-elliptic family containing Gaussian and Laplace kernels, while also showing that a small field alone does not prevent escape of mass. ~\citet{he2026sinkhorn} replace one-sided normalization by two-sided Sinkhorn scaling, which produces a related identifiable attraction-repulsion dynamics.

Several works connect drifting to variational transport. ~\citet{cao2026gradient} derive KDE-based procedures from Wasserstein gradient flows of approximated divergences. ~\citet{gretton2026wasserstein} distinguish the proposed and implemented Gaussian updates and relate them to fixed points of smoothed-divergence flows. ~\citet{turan2026secretly} give a frozen-field variational interpretation and use a linearized Fourier analysis to expose the high-frequency bottleneck of a fixed Gaussian bandwidth.  The exact placement of the smoothing operator is important, however.  As explained in Subsection~\ref{sec:Wgradientflow} and proved in Appendix~\ref{app:no-hidden-energy}, the plug-in map
\[
    q\longmapsto
    \log\frac{G_\sigma*q}{G_\sigma*\rho^\star}
\] 
 is not the first  variation of the smoothed relative entropy, and at positive scale it is not the first variation of any $C^2$ energy on smooth positive densities.  We therefore analyze the autonomous transport equation directly rather than relying on an exact energy-dissipation identity.

The closest prior convergence analysis is~\citet{turan2026secretly}.  Their mode-by-mode calculation predicts exponentially large time scales for high frequencies under a fixed Gaussian bandwidth and motivates a bandwidth that decreases with time.  Our lower bounds establish the corresponding obstruction for the full nonlinear autonomous equation and uniformly over finite-smoothness target classes.  Our remedy is also different: rather than annealing one bandwidth, we use one time-independent vector field that averages a continuum of scales.

\citet{brown2026secondorder} propose second-order drifting models to accelerate the recovery of fine-scale structure. They add momentum to the sample dynamics through auxiliary velocities and time-dependent damping, and connect the resulting evolution to Nesterov's accelerated optimization. Their theoretical analysis is local around equilibrium: after linearizing the population dynamics and approximating the target as locally flat, they show that each Fourier mode follows a Nesterov-type second-order ODE. This connects their dynamics to the classical $\mathcal{O}(t^{-2})$ convergence rate for the associated quadratic objective, rather than providing a rate for the full nonlinear evolution. 

Our results strengthen this convergence picture at the level of the full nonlinear population dynamics. We establish quantitative fixed-bandwidth lower bounds for both Gaussian and Laplace kernels, uniformly over classes of targets with finite smoothness. By contrast, our multihead construction accelerates convergence by combining several kernel scales in the drift, including increasingly fine scales, while keeping the dynamics first-order and time-independent. We prove exponential convergence of the resulting nonlinear flow near the uniform law on the torus and the standard Gaussian law on $\R^d$ (Theorems~\ref{thm:local-torus-heat} and~\ref{thm:local-ou}). The gain is therefore a distribution-level convergence guarantee that remains effective at fine scales, beyond the linearized mode-by-mode description.

The finite-particle analysis of~\citet{balasubramanian2026finite} studies a complementary question: what can be guaranteed when drifting is run with a finite number of particles.  For both conservative kernel-score dynamics and the original non-conservative Laplace field, it derives continuous-time, time-averaged bounds that keep track of errors coming from the finite number of particles, the kernel bandwidth, and the approximation of population quantities by empirical ones.  These results quantify how the interacting-particle system behaves as the number of particles grows, but they do not provide a rate at which the full population distribution approaches the target as physical time increases.  Our paper focuses on this latter question. Finally,~\citet{dumont2026monge} prove long-time convergence for a constrained gradient flow on transport maps under convexity assumptions; their evolution is different from the attraction-repulsion population dynamics studied here.

\paragraph{Organization.}
Section~\ref{sec:counterexample} proves the fixed-bandwidth Gaussian and Laplace lower bounds and gives their Fourier interpretation.  Section~\ref{sec:natural-exponential-convergence} introduces the multihead field and proves local exponential convergence for the heat and Ornstein-Uhlenbeck models.  Section~\ref{sec:numerical-experiments} presents a numerical study of a smooth multiscale target, both at the population-PDE and finite-data levels.  Two additional, more complex finite-data experiments are reported in \hyperref[app:additional-numerical-experiments]{Appendix~\ref*{app:additional-numerical-experiments}}. The remaining appendices contain the fixed-bandwidth well-posedness argument, the lower-bound proofs, a variational non-integrability example, and the nonlinear estimates for the multihead flow.

\section{Slow convergence of fixed-bandwidth drifting flows on the torus}
This section shows that a fixed positive bandwidth can severely limit the convergence of drifting flows, even at the population level with exact expectations. On the torus, this obstruction is particularly transparent in Fourier variables: smoothing suppresses high-frequency discrepancies before the drift is computed, so the corresponding damping rates vanish as the frequency increases. We turn this mechanism into nonlinear lower bounds over bounded finite-smoothness classes of target densities.
\label{sec:counterexample}
\subsection{Gaussian smoothing}
\label{sec:Wgradientflow}

In the case of Gaussian smoothing, the drifting field can be written using
the scores of the smoothed densities. Recall that, for a positive density
$\rho$, its score corresponds to the field $\nabla \log \rho$.

Let
\[
    G_\sigma(z)
    \propto
    \exp\left(-\frac{|z|^2}{2\sigma^2}\right),
    \qquad
    \rho_\sigma := G_\sigma * \rho.
\]
Since
\[
    \nabla_x G_\sigma(x-y)
    =
    \frac{y-x}{\sigma^2}G_\sigma(x-y),
\]
we obtain
\[
    \nabla \rho_\sigma(x)
    =
    \frac{1}{\sigma^2}
    \int G_\sigma(x-y)(y-x)\rho(y)\,\dd y.
\]
Dividing by $\rho_\sigma(x)$ gives
\begin{equation}
    \frac{\int G_\sigma(x-y)(y-x)\rho(y)\,\dd y}
         {\int G_\sigma(x-y)\rho(y)\,\dd y}
    =
    \sigma^2 \nabla \log \rho_\sigma(x).
    \label{eq:gaussian-local-attraction-score}
\end{equation}
Thus, by the definition of the drifting field,
\begin{equation}
    V_{\rho^\star,q}^\sigma
    =
    \sigma^2
    \left(
        \nabla\log \rho_\sigma^\star
        -
        \nabla\log q_\sigma
    \right)
    =
    -\sigma^2
    \nabla\log\frac{q_\sigma}{\rho_\sigma^\star}.
    \label{eq:gaussian-score-difference}
\end{equation}
Because convolution with $G_\sigma$ agrees with the heat semigroup
$P_{\sigma^2/2}$, the continuity equation can therefore be written as
\begin{equation}
    \partial_t q_t
    =
    \sigma^2 \diver\left(
        q_t \nabla\log
        \frac{P_{\sigma^2/2}q_t}
             {P_{\sigma^2/2}\rho^\star}
    \right).
    \label{eq:gaussian-score-pde}
\end{equation}
The same computation holds on the torus using the periodic Gaussian kernel.
This is the score formulation used in the next subsection.
\label{sec:gaussian-smoothing}
\subsubsection{A logarithmic lower bound}
We state the obstruction on the flat torus $\T^d=(\R/2\pi\mathbb Z)^d$, equipped with normalized Lebesgue measure.  For a bounded potential $a:\T^d\to\R$, set
\begin{equation}
    \rho_a^\star(x):=\frac{e^{a(x)}}{\int_{\T^d} e^{a(y)}\,\dd y}.
\end{equation}
For a fixed bandwidth $\sigma>0$, let $q_t^a$ solve the Gaussian drifting equation
\begin{equation}
    \partial_t q_t^a
    =
    \sigma^2\diver\left(
        q_t^a\nabla
        \log\frac{P_{\sigma^2/2}q_t^a}{P_{\sigma^2/2}\rho_a^\star}
    \right),
    \qquad
    q_0^a\equiv1,
    \label{eq:gaussian-drifting-pa}
\end{equation}
where $P_s=e^{s\Delta}$ is the heat semigroup on $\T^d$.
\begin{proposition}[Gaussian fixed-bandwidth smoothing is logarithmically slow]
\label{prop:nonlinear-log-lower-bound}
Let $\beta>0,\sigma>0$ and $C>0$. Then there exists a constant $c>0$ such that for all $t\ge 2$,
\begin{equation}
    \sup_{a\in \mathcal C_C^\beta(\T^d)}
    \KL(q_t^a\mid \rho_a^\star)
    \geq
    c(\log t)^{-\beta}.
    \label{eq:gaussian-Hbeta-KL-lower}
\end{equation}
Moreover,
\begin{equation}
    \sup_{a\in \mathcal C_C^\beta(\T^d)}
    W_2(q_t^a,\rho_a^\star)
    \geq
    c(\log t)^{-(\beta+1)/2}.
    \label{eq:gaussian-Hbeta-W2-lower}
\end{equation}
\end{proposition}

The proof of Proposition~\ref{prop:nonlinear-log-lower-bound} can be found in Section~\ref{app:proof-gaussian-smoothing}. Since $a\in\mathcal C_C^\beta(\mathbb T^d)$ is uniformly bounded, the map
\[
    a \longmapsto \rho_a^\star
    :=
    \frac{e^a}{\int_{\mathbb T^d} e^a\,\dd x}
\]
identifies this class, up to changes in the radius, with a bounded
$\mathcal C^\beta$ class of probability densities that are uniformly bounded
above and away from zero. Thus, the lower bound can equivalently be interpreted
as a worst-case statement over a finite-regularity class of target
densities.

\subsubsection{Fourier intuition}
\label{subsubsec:fourier-intuition-gaussian-smoothing}
To give some intuition on the result, we explain the fixed-bandwidth obstruction for the linearized equation in the simplest case, where the
target is the flat density $
    \rho^\star\equiv 1 .$ We use the Fourier convention
\[
    e_n(x)=e^{i n\cdot x},
    \qquad n\in\mathbb Z^d,
\]
so that
\[
    \Delta e_n=-|n|^2 e_n,
    \qquad
    P_s e_n=e^{-s|n|^2}e_n .
\]
For Gaussian drifting with bandwidth $\sigma>0$, set $
    s_\sigma:=\frac{\sigma^2}{2}$.
Since $\rho^\star\equiv1$, the fixed-bandwidth Gaussian drifting equation becomes
\begin{equation}
    \partial_t q_t
    =
    \sigma^2\nabla\cdot\left(
        q_t\nabla\log P_{s_\sigma}q_t
    \right).
    \label{eq:gaussian-fourier-flat-pde}
\end{equation}

We linearize this equation around the equilibrium $q\equiv1$.  Write
\[
    q_t=1+u_t,
    \qquad
    \int_{\T^d}u_t(x)\,\dd x=0,
\]
with $u_t$ small.  Since $P_{s_\sigma}1=1$, one has $
    P_{s_\sigma}q_t
    =
    1+P_{s_\sigma}u_t .$
Therefore
\[
    \log P_{s_\sigma}q_t
    =
    \log(1+P_{s_\sigma}u_t)
    =
    P_{s_\sigma}u_t+O(u_t^2).
\]
Moreover,
\[
    q_t\nabla\log P_{s_\sigma}q_t
    =
    (1+u_t)\nabla(P_{s_\sigma}u_t+O(u_t^2))
    =
    \nabla P_{s_\sigma}u_t+O(u_t^2).
\]
Keeping only the linear terms in~\eqref{eq:gaussian-fourier-flat-pde} gives
\begin{equation}
    \partial_t u_t
    =
    \sigma^2\Delta P_{s_\sigma}u_t .
    \label{eq:gaussian-linearized-flat}
\end{equation}

Now decompose $u_t$ into Fourier modes:
\[
    u_t(x)
    =
    \sum_{n\in\mathbb Z^d\setminus\{0\}}
    \widehat u_t(n)e_n(x).
\]
The zero mode is absent because $u_t$ has mean zero.  For each
$n\neq0$,
\[
    \Delta P_{s_\sigma}e_n
    =
    -|n|^2 e^{-s_\sigma |n|^2}e_n .
\]
Thus~\eqref{eq:gaussian-linearized-flat} becomes, mode by mode,
\begin{equation}
    \frac{\dd}{\dd t}\widehat u_t(n)
    =
    -\lambda_\sigma(n)\widehat u_t(n),
    \label{eq:gaussian-mode-ode}
\end{equation}
where
\begin{equation}
    \lambda_\sigma(n)
    :=
    \sigma^2 |n|^2 e^{-s_\sigma |n|^2}
    =
    \sigma^2 |n|^2
    \exp\left(-\frac{\sigma^2|n|^2}{2}\right).
    \label{eq:gaussian-fixed-scale-rate}
\end{equation}
Consequently,
\[
    \widehat u_t(n)
    =
    e^{-\lambda_\sigma(n)t}\widehat u_0(n).
\]

The key point is the high-frequency behavior of
$\lambda_\sigma(n)$:
\[
    \lambda_\sigma(n)
    =
    \sigma^2 |n|^2
    e^{-\sigma^2|n|^2/2}
    \longrightarrow 0
    \qquad
    \text{as } |n|\to\infty .
\]
Thus the linearized fixed-bandwidth Gaussian operator has no spectral gap:
\[
    \inf_{n\in\mathbb Z^d\setminus\{0\}}\lambda_\sigma(n)=0.
\]
A high-frequency perturbation of $q_t$ is
almost erased by the heat operator $P_{s_\sigma}$ before the drift is
computed.  The transport field therefore sees the mode only through the tiny
factor $e^{-\sigma^2|n|^2/2}$, and the corresponding damping rate is
exponentially small in $|n|^2$.

For example, take a perturbation along the first coordinate,
\[
    u_0(x)=\varepsilon_n\sin(nx_1),
    \qquad
    \varepsilon_n=n^{-\beta}.
\]
Then the linearized solution is
\[
    u_t(x)
    =
    \varepsilon_n e^{-\lambda_\sigma(n\mathbf e_1)t}\sin(nx_1),
\]
where $\mathbf e_1=(1,0,\ldots,0)$.  The time scale needed to damp this mode is
\[
    \lambda_\sigma(n\mathbf e_1)^{-1}
    =
    \frac{1}{\sigma^2n^2}
    \exp\left(\frac{\sigma^2n^2}{2}\right).
\]
At a given time $t$, choose the frequency $n_t$ so that $
    t\lambda_\sigma(n_t\mathbf e_1)\lesssim 1.$
Equivalently,
\[
    n_t\asymp_\sigma \sqrt{\log t},
\]
up to harmless logarithmic corrections from the polynomial factor $n_t^2$.
For this choice of frequency, the mode has not yet been substantially damped
by time $t$.  Its remaining amplitude is still of order
\[
    n_t^{-\beta}
    \asymp
    (\log t)^{-\beta/2}.
\]
Therefore, we have $
\|u_t\|_{L^1(\mathbb{T}^d)}
\asymp n_t^{-\beta}
\asymp (\log t)^{-\beta/2}$, 
which gives the entropy obstruction
\[
    \KL(q_t\mid 1)
    \gtrsim
    (\log t)^{-\beta}.
\]
This linearized calculation captures the high-frequency obstruction underlying Proposition~\ref{prop:nonlinear-log-lower-bound} and predicts the critical frequency scale
\[
n \asymp \sqrt{\log t},
\]
used in the nonlinear proof.

\subsection{Laplace smoothing}
\label{sec:laplace-smoothing-rates}

We now consider the same fixed-bandwidth convergence question as Section~\ref{sec:Wgradientflow}, but for a Laplace-type kernel.
The conclusion is better than in the Gaussian case, but still negative. For Gaussian kernels,
Tweedie's formula rewrites the local attraction field as a smoothed score:
\[
    m_\rho^\sigma(x)
    =
    \sigma^2\nabla\log \rho_\sigma(x).
\]
No analogous identity is available for the Laplace kernel.  We therefore keep
the original local attraction field.
For $\tau>0$, let $L_\tau$ denote
the periodized Laplace smoothing kernel.  Its Fourier coefficients are
\begin{equation}
    \widehat L_\tau(m)
    =
    \left(1+\tau^2 |m|^2\right)^{-(d+1)/2},
    \qquad m\in\mathbb Z^d .
    \label{eq:laplace-fourier-multiplier}
\end{equation}
Thus $L_\tau*\rho$ is the Laplace-smoothed version of a density $\rho$. We also need the vector kernel which gives the local average displacement.  On
the torus, it is convenient to define it through its Fourier coefficients:
\begin{equation}
    \widehat K_\tau(m)
    =
    -i\nabla_\xi
    \left[
        \left(1+\tau^2|\xi|^2\right)^{-(d+1)/2}
    \right]_{\xi=m},
    \qquad m\neq0,
    \qquad
    \widehat K_\tau(0)=0 .
    \label{eq:laplace-vector-kernel}
\end{equation}
For a positive density $\rho$ on $\T^d$, define
\begin{equation}
    M_\tau[\rho](x)
    :=
    \frac{K_\tau*\rho(x)}{L_\tau*\rho(x)} .
    \label{eq:true-laplace-local-attraction-field}
\end{equation}
The denominator is the mass of $\rho$ seen from $x$ at scale $\tau$, while the
numerator is the corresponding average displacement.  Hence $M_\tau[\rho](x)$
is the Laplace analogue of the local attraction field used in the definition of
drifting.

Given a target measure $\rho^\star_a$, the fixed-bandwidth Laplace drifting equation is
\begin{equation}
    \partial_t q_t^{a}
    +
    \diver\left(
        q_t^{a}
        \left[
            M_\tau[\rho_a^\star]-M_\tau[q_t^{a}]
        \right]
    \right)
    =0,
    \qquad
    q_0^{a}\equiv1 .
    \label{eq:true-laplace-drifting-pde}
\end{equation}
\subsubsection{A polynomial lower bound}

For a potential $a:\T^d\to\R$, set
\begin{equation}
    \rho_a^\star(x):=
    \frac{e^{a(x)}}{\int_{\T^d}e^{a(y)}\,\dd y}.
\end{equation}
The periodized Laplace kernel $L_\tau$ is strictly positive and Lipschitz, and
the displacement kernel $K_\tau$ is Lipschitz.  These properties follow, for
example, from their absolutely convergent periodized Euclidean
representations.  Proposition~\ref{prop:fixed-scale-wellposedness} therefore
applies to~\eqref{eq:true-laplace-drifting-pde}; we denote its unique global
characteristic solution by $q_t^a$.

\begin{proposition}[Laplace fixed-bandwidth smoothing is polynomially slow]

\label{prop:true-laplace-polynomial-lower-bound}
Let $\tau>0$, $\beta>0$ and $C>0$. Then, there exists $c>0$ such that for every $t\geq0$,
\begin{equation}
    \sup_{a\in\mathcal C_C^\beta(\T^d)}
    \KL(q_t^a\mid \rho_a^\star)
    \geq
    c
    \left(
        1+\frac{t}{\tau^{d+1}}
    \right)^{-2\beta/(d+1)}
    \label{eq:true-laplace-KL-lower}
\end{equation}
and
\begin{equation}
    \sup_{a\in\mathcal C_C^\beta(\T^d)}
    W_2(q_t^a,\rho_a^\star)
    \geq
    c
    \left(
        1+\frac{t}{\tau^{d+1}}
    \right)^{-(\beta+1)/(d+1)}.
    \label{eq:true-laplace-W2-lower}
\end{equation}
\end{proposition}

The proof of Proposition~\ref{prop:true-laplace-polynomial-lower-bound} can be found in Section~\ref{app:proof-laplace-smoothing}.

\subsubsection{Fourier intuition}
\label{subsubsec:fourier-intuition-laplace-smoothing}

We now explain the fixed-bandwidth Laplace obstruction at the linearized
Fourier level for $\rho^\star\equiv1$.
Since $M_\tau[1]=0$, the Laplace drifting equation becomes
\begin{equation}
    \partial_t q_t
    -
    \diver\left(q_t M_\tau[q_t]\right)
    =0 .
    \label{eq:laplace-fourier-flat-pde}
\end{equation}
Write
\[
    q_t=1+u_t,
    \qquad
    \int_{\T^d}u_t(x)\,\dd x=0,
\]
with $u_t$ small.  Since
\[
    L_\tau*q_t=1+L_\tau*u_t,
    \qquad
    K_\tau*q_t=K_\tau*u_t,
\]
the quotient defining the local attraction field satisfies
\[
    M_\tau[q_t]
    =
    \frac{K_\tau*u_t}{1+L_\tau*u_t}
    =
    K_\tau*u_t+O(u_t^2).
\]
Keeping only the linear terms in~\eqref{eq:laplace-fourier-flat-pde} gives
\begin{equation}
    \partial_t u_t
    =
    \diver(K_\tau*u_t).
    \label{eq:laplace-linearized-flat}
\end{equation}

Recall that
\[
    \widehat L_\tau(m)
    =
    \left(1+\tau^2|m|^2\right)^{-(d+1)/2},
\]
by definition of the vector kernel we have
\[
    \widehat K_\tau(m)
    =
    -i\nabla_\xi\widehat L_\tau(\xi)\big|_{\xi=m}
    =
    i(d+1)\tau^2m
    \left(1+\tau^2|m|^2\right)^{-(d+3)/2},
    \qquad m\neq0 .
\]
Therefore
\[
    \diver(K_\tau*e_m)
    =
    i m\cdot \widehat K_\tau(m)e_m
    =
    -\lambda_\tau(m)e_m,
\]
where
\begin{equation}
    \lambda_\tau(m)
    :=
    (d+1)\tau^2|m|^2
    \left(1+\tau^2|m|^2\right)^{-(d+3)/2}.
    \label{eq:laplace-mode-rate}
\end{equation}
Thus, if
\[
    u_t(x)=
    \sum_{m\in\mathbb Z^d\setminus\{0\}}
    \widehat u_t(m)e_m(x),
\]
then~\eqref{eq:laplace-linearized-flat} becomes, mode by mode,
\begin{equation}
    \frac{\dd}{\dd t}\widehat u_t(m)
    =
    -\lambda_\tau(m)\widehat u_t(m),
    \qquad
    \widehat u_t(m)=e^{-\lambda_\tau(m)t}\widehat u_0(m).
    \label{eq:laplace-mode-ode}
\end{equation}

The high-frequency behavior is the key point.  From
\eqref{eq:laplace-mode-rate},
\begin{equation}
    \lambda_\tau(m)
    \sim
    (d+1)\tau^{-(d+1)}|m|^{-(d+1)}
    \qquad\text{as } |m|\to\infty .
    \label{eq:laplace-high-frequency-rate}
\end{equation}
In particular,
\[
    \lambda_\tau(m)\longrightarrow0
    \qquad\text{as } |m|\to\infty .
\]
The linearized fixed-bandwidth Laplace operator therefore has no spectral gap.
However, the loss is only polynomial in $|m|$, whereas in the Gaussian case it
is exponential in $|m|^2$.

The target-perturbation calculation now follows the Gaussian case, with a
different mode rate.  Take
\[
    \rho^\star(x)=1+\varepsilon_n\sin(nx_1),
    \qquad
    q_0\equiv1,
    \qquad
    \varepsilon_n=n^{-\beta}.
\]
At the linearized level,
\[
    r_t(x)
    \simeq
    -\varepsilon_n e^{-\lambda_\tau(n\mathbf e_1)t}\sin(nx_1).
\]
Since
\[
    \lambda_\tau(n\mathbf e_1)
    \asymp
    \tau^{-(d+1)}n^{-(d+1)},
\]
choosing
\[
    n_t
    \asymp
    \left(1+\frac{t}{\tau^{d+1}}\right)^{1/(d+1)}
\]
leaves a discrepancy of order
\[
    n_t^{-\beta}
    \asymp
    \left(1+\frac{t}{\tau^{d+1}}\right)^{-\beta/(d+1)}.
\]
After squaring the amplitude, this gives
\[
    \KL(q_t\mid\rho^\star)
    \gtrsim
    \left(1+\frac{t}{\tau^{d+1}}\right)^{-2\beta/(d+1)}.
\]
This linear calculation identifies the high-frequency bottleneck underlying Proposition~\ref{prop:true-laplace-polynomial-lower-bound}: for fixed Laplace smoothing, the damping rate decays polynomially, rather than exponentially as in the Gaussian case, and therefore still vanishes at high frequencies.  

\subsection{Heuristic extension to Gaussian initialization on \texorpdfstring{$\mathbb{R}^d$}{Rd}}

\label{subsec:euclidean-fixed-bandwidth-heuristic}

The lower bounds above were stated on the torus in order to keep the argument
simple: Fourier modes are exact eigenfunctions, the smoothed denominators are
uniformly bounded from below, and the nonlinear flow preserves the
high-frequency symmetry of the examples.  The obstruction itself, however, is
not a compactness phenomenon.  It is a fixed-bandwidth phenomenon,
and the same mechanism should also appear on $\mathbb R^d$ when the
initialization is the standard Gaussian.

Let $\gamma$ denote the standard Gaussian density on $\mathbb R^d$ and consider
targets of the form
\[
    \rho_n^\star(x)=Z_n^{-1}\exp\{\varepsilon_n\sin(n x_1)\}\gamma(x),
    \qquad
    \varepsilon_n=n^{-\beta}.
\]
The potentials $\varepsilon_n\sin(nx_1)$ are uniformly bounded in
$\mathcal C^\beta(\mathbb R^d)$.  Moreover,
\[
    \int_{\mathbb R^d}\sin(nx_1)\rho_n^\star(x)\,\dd x
    =
    \frac{\varepsilon_n}{2}+o(\varepsilon_n),
\]
whereas the same moment is initially zero under $q_0=\gamma$.  Thus, as on the
torus, it is enough to understand how quickly the drifting vector field can
create this oscillatory moment.

For Gaussian smoothing, the attenuation of high-frequency oscillations can be computed explicitly.  If
$G_\sigma$ is the Gaussian kernel with variance $\sigma^2$, then
\[
    \frac{
        G_\sigma*(\gamma\sin(n\cdot))(x)
    }{
        G_\sigma*\gamma(x)
    }
    =
    \exp\left(
        -\frac{\sigma^2n^2}{2(1+\sigma^2)}
    \right)
    \sin\left(\frac{n x_1}{1+\sigma^2}\right).
\]
Thus an oscillation of frequency $n$ in the Gaussian background is reduced by
an exponentially small factor before the score is computed.  Differentiation
only contributes polynomial factors in $n$.  Consequently the fixed-bandwidth
Gaussian field sees the perturbation at size
\[
    \mathrm{poly}(n)\exp(-c_\sigma n^2),
\]
exactly as in the torus calculation.  Choosing
\[
    n\asymp_\sigma \sqrt{\log t}
\]
therefore leaves a perturbation of amplitude $n^{-\beta}$, suggesting the candidate
orders
\[
    \KL(q_t^n\mid \rho_n^\star)\gtrsim (\log t)^{-\beta},
    \qquad
    W_2(q_t^n,\rho_n^\star)\gtrsim (\log t)^{-(\beta+1)/2},
\]
up to changes in constants.

For the Laplace kernel the same wave-packet picture gives the polynomial
analogue.  The Fourier transform of $\gamma(x)\sin(nx_1)$ is concentrated in
two packets centered at $\pm n e_1$.  At high frequency, the Laplace multiplier
is essentially constant on each packet and has size
\[
    (1+\tau^2 n^2)^{-(d+1)/2}.
\]
For the drifting vector field, after taking the divergence, this gives the
same effective damping scale as on the torus,
\[
    \lambda_\tau(n)
    \asymp
    \tau^{-(d+1)}n^{-(d+1)}.
\]
Choosing
\[
    n\asymp
    \left(1+\frac{t}{\tau^{d+1}}\right)^{1/(d+1)}
\]
therefore suggests the same candidate rates
\[
    \KL(q_t^n\mid \rho_n^\star)
    \gtrsim
    \left(1+\frac{t}{\tau^{d+1}}\right)^{-2\beta/(d+1)},
    \qquad
    W_2(q_t^n,\rho_n^\star)
    \gtrsim
    \left(1+\frac{t}{\tau^{d+1}}\right)^{-(\beta+1)/(d+1)}.
\]

We do not state this Euclidean version as a theorem.  Proving it would require
technical estimates which are absent from the torus proof and would not change
the underlying message.  On the torus, the examples remain in an exactly
invariant high-frequency subspace, and the positivity of the heat kernel gives
a uniform lower bound on all smoothed denominators.  On $\mathbb R^d$, both
features disappear.  The Gaussian background couples the oscillation to the
slowly varying envelope; the nonlinear transport can leak mass into nearby
frequencies; and the denominators $K*q_t$ decay in the tails, so one needs
weighted lower bounds, tail propagation, and weighted stability estimates for
the quotient defining the drift.  For the Laplace kernel, one would additionally
need a pseudodifferential version of the torus multiplier estimate, uniformly
under the nonlinear evolution.

These difficulties are technical rather than conceptual.  Locally, on spatial
scales of order $1/n$, the Gaussian background is nearly constant and the
fixed-bandwidth kernel acts on the oscillatory part exactly as its Fourier
multiplier predicts.  The torus result isolates this mechanism without the
extra weighted analysis required on $\mathbb R^d$.  We therefore view the
compact setting as a clean model for the fixed-bandwidth obstruction, not as
the source of the obstruction.

\section{Continuum multihead drifting restores exponential convergence}
\label{sec:natural-exponential-convergence}
\subsection{The multihead population field}
The results of Section~\ref{sec:counterexample} show that no single bandwidth gives a uniform rate over the target classes considered there. We therefore average the drifting field over a range of bandwidths. We first give the empirical construction and then define the population equation analyzed in this section.

At the finite-sample level this means replacing one drifting field by a weighted superposition of
drifting fields.  If $0<s_1<\cdots<s_J$ are smoothing scales and $w_j>0$ are weights, the
multihead version has the form
\begin{equation}
    \widehat V_{\ell,n}^{\rm multi}(x)
    =
    \sum_{j=1}^J w_j\,
    \widehat V_{\ell,n}^{s_j}(x),
    \label{eq:empirical-multiscale-drifting-field}
\end{equation}
where $\widehat V_{\ell,n}^{s_j}$ is the usual velocity field at
scale $s_j$.  This multihead algorithm is 
close to practical implementations in which the regression target is built from several bandwidths rather than from a single one as done in the original paper~\citep{deng2026drifting}.

The population object studied in this section is the continuous-scale analogue of
\eqref{eq:empirical-multiscale-drifting-field}.  Let $(\mathsf P_s)_{s\ge0}$ be a smoothing family
with positive kernels.  Given a target law $\rho^\star$ and a current law $\rho$, define the
scale-$s$ smoothed log-ratio by
\begin{equation}
    \psi_s^{\mathsf P}[\rho\mid\rho^\star](x)
    :=
    \log\frac{\mathsf P_s\rho(x)}{\mathsf P_s\rho^\star(x)}.
    \label{eq:abstract-smoothed-log-ratio}
\end{equation}
The associated scale-averaged potential is an average of these smoothed log-ratios,
\begin{equation}
    \Phi_S^{\mathsf P}[\rho\mid\rho^\star](x)
    :=
    \frac{1}{S}\int_0^S \psi_s^{\mathsf P}[\rho\mid\rho^\star](x)\,\dd s.
    \label{eq:abstract-multiscale-potential}
\end{equation}
The population multihead drifting velocity is then
\begin{equation}
    v^{\mathsf P}_{\rho\mid\rho^\star}(x)
    =
    -\nabla \Phi_S^{\mathsf P}[\rho\mid\rho^\star](x),
    \label{eq:abstract-multiscale-velocity}
\end{equation}
and the corresponding evolution is
\begin{equation}
    \partial_t\rho_t
    =
    \nabla\cdot\bigl(\rho_t\nabla\Phi_S^{\mathsf P}[\rho_t\mid\rho^\star]\bigr).
    \label{eq:abstract-multiscale-flow}
\end{equation}
Note that for the heat kernel, this construction is not simply the classical drifting
field evaluated at several bandwidths and averaged with the same weights. Indeed, the classical Gaussian drifting field at bandwidth $\sigma=\sqrt{2s}$ is
$$V_{\rho^\star,\rho}^{\sigma}=-\sigma^2\nabla \log\frac{\mathsf P_s\rho}{\mathsf P_s\rho^\star},$$ whereas
\eqref{eq:abstract-multiscale-velocity} corresponds to
\[
    v^{\rm heat}_{\rho\mid\rho^\star}
    =
    \frac{1}{S}\int_0^S \frac{1}{2s}
    V_{\rho^\star,\rho}^{\sqrt{2s}}\,\dd s,
\]
so removing the prefactor $\sigma^2$ amounts to using a non-uniform scale weighting proportional to
$1/\sigma^2$, and therefore gives relatively more weight to fine scales. In practice, this corresponds to taking weights $w_j=\frac{1}{2s_jJ}$ in~\eqref{eq:empirical-multiscale-drifting-field}.

\subsection{Fourier intuition: scale averaging creates a spectral gap}
\label{subsec:fourier-intuition-multiscale-heat}

We now give a simple Fourier calculation explaining why averaging over heat scales removes the
high-frequency obstruction present at any fixed positive scale. This calculation is linear and should
be understood as the model case behind the perturbative nonlinear argument below.

We work on the flat torus with the target density $\rho^\star\equiv1$. Write a nearby density as
\[
    \rho_t=1+u_t,
    \qquad
    \int_{\T^d}u_t(x)\,\dd x=0.
\]
Consider first a single heat scale $s>0$. At the linearized level, the smoothed log-ratio is
\[
    \log(P_s\rho_t)
    =
    \log(1+P_su_t)
    =
    P_su_t+O(u_t^2).
\]
Thus the linearized potential generated by the scale $s$ is
\[
    \Phi_s^{\rm lin}[u_t]=P_su_t.
\]
The population equation has the form
\[
    \partial_t\rho_t
    =
    \nabla\cdot(\rho_t\nabla\Phi[\rho_t]).
\]
Linearizing around $\rho^\star\equiv1$ gives
\[
    \partial_tu_t
    =
    \Delta P_su_t.
\]
Now write
\[
    u_t(x)=\sum_{n\in\mathbb Z^d\setminus\{0\}}\widehat u_t(n)e_n(x).
\]
The zero mode is absent because $u_t$ has mean zero. For each nonzero Fourier mode,
\[
    \Delta P_s e_n
    =
    -|n|^2e^{-s|n|^2}e_n.
\]
Therefore
\[
    \frac{\dd}{\dd t}\widehat u_t(n)
    =
    -\lambda_s(n)\widehat u_t(n), \qquad \text{with } \lambda_s(n)
    =
    |n|^2e^{-s|n|^2}.
\]
For every fixed $s>0$,
$
    \lambda_s(n)
    \longrightarrow0$
as $n$ tends to infinity.
This is the fixed-scale obstruction: the kernel almost removes very high frequencies before the
drift has a chance to act on them.

The multihead field averages these fixed-scale contributions over $s\in[0,S]$. At the linearized
level, the damping rate of the Fourier mode $n$ is therefore the average of the scale-wise damping
rates:
\[
    \lambda_S^{\rm multi}(n)
    =
    \frac{1}{S}\int_0^S \lambda_s(n)\,\dd s
    =
    \frac{1}{S}\int_0^S |n|^2e^{-s|n|^2}\,\dd s
    =
    \frac{1-e^{-S|n|^2}}{S}.
\]
Thus the scale-averaged linearized equation reads, mode by mode,
\[
    \frac{\dd}{\dd t}\widehat u_t(n)
    =
    -\lambda_S^{\rm multi}(n)\widehat u_t(n),
    \qquad
    \lambda_S^{\rm multi}(n)
    =
    \frac{1-e^{-S|n|^2}}{S}.
\]
The behavior is now completely different from the fixed-scale one. Indeed, since $n\in\mathbb Z^d\setminus\{0\}$, one has $|n|\ge1$; therefore
\[
    \lambda_S^{\rm multi}(n)
    =
    \frac{1-e^{-S|n|^2}}{S}
    \ge
    \frac{1-e^{-S}}{S}
    =:\lambda_S^{\rm lin}>0.
\]
This gives a genuine spectral gap for the linearized scale-averaged operator. More precisely,
\[
    \frac{\dd}{\dd t}|\widehat u_t(n)|^2
    =
    -2\lambda_S^{\rm multi}(n)|\widehat u_t(n)|^2
    \le
    -2\lambda_S^{\rm lin}|\widehat u_t(n)|^2
    \qquad
    \text{for every } n\neq0.
\]
Summing over Fourier modes gives, for any Sobolev index $r\ge0$,
\[
    \|u_t\|_{H^r}
    \le
    e^{-\lambda_S^{\rm lin}t}\|u_0\|_{H^r}.
\]
This is the linear spectral-gap mechanism used in the nonlinear theorem: fixed heat scales damp each
frequency at a rate that vanishes at high frequency, whereas averaging over all heat scales
$0\le s\le S$ gives a uniform positive damping rate.

\subsection{Local exponential convergence for multihead heat drifting on the torus}
\label{subsec:local-torus-heat}

The slow-convergence results above are fixed-bandwidth statements.  We now record the complementary local result for the multihead heat field.  In this subsection $\T^d=(\R/2\pi\mathbb Z)^d$ is equipped with normalized Lebesgue measure; the choice of period only changes constants.  Let
\[
    P_s=e^{s\Delta},\qquad 0\le s\le S,
\]
and take the uniform scale average $\mu_S=\frac{1}{S}\mathds 1_{[0,S]}(s)\,\dd s$ in~\eqref{eq:abstract-multiscale-potential}.  If
\[
    \rho^\star=e^{a}\,\dd x,
    \qquad
    \int_{\T^d}e^{a}\,\dd x=1,
\]
and $h=\log(\rho/\rho^\star)$, then the scale-averaged heat potential is
\begin{equation}
    \Phi^{\rm heat}_S[h]
    =\frac{1}{S}\int_0^S
      \log\frac{P_s(e^{h+a})}{P_s(e^{a})}\,\dd s.
    \label{eq:main-heat-multiscale-potential}
\end{equation}
The corresponding evolution is
\begin{equation}
    \partial_t\rho_t
    =\nabla\cdot\bigl(\rho_t\nabla\Phi^{\rm heat}_S[h_t]\bigr),
    \qquad
    h_t=\log\frac{\rho_t}{\rho^\star},
    \qquad
    \rho_0=\dd x.
    \label{eq:main-heat-multiscale-flow}
\end{equation}
\begin{theorem}[Local exponential convergence for multihead heat drifting]
\label{thm:local-torus-heat}
Fix $S>0$ and $d\ge1$.  There exist constants
\[
    \varepsilon_S>0,
    \qquad
    c_{d,S}>0,
    \qquad
    C_S<\infty,
\]
such that the following holds.  Let $a\in C^{2,1}(\T^d)$ satisfy
\[
    \int_{\T^d}e^{a}\,\dd x=1,
    \qquad
    \|a\|_{C^{2,1}}\le \varepsilon_S.
\]
Then there exists a unique global solution $(\rho_t)_{t\ge0}$ to
\eqref{eq:main-heat-multiscale-flow}, in the sense of
Section~\ref{subsec:loc-endpoint-solutions}.  Furthermore, for every $t\ge0$, it
satisfies
\begin{equation}
    \left\|\nabla\log\frac{\rho_t}{\rho^\star}\right\|_{C^{0,1}}
    \le
    C_S e^{-c_{d,S}t}\|\nabla a\|_{C^{0,1}},
    \label{eq:main-heat-score-decay}
\end{equation}
and
\begin{equation}
    \KL(\rho_t\mid\rho^\star)
    \le
    C_S e^{-2c_{d,S}t}\|\nabla a\|_{C^{0,1}}^2.
    \label{eq:main-heat-entropy-decay}
\end{equation}
\end{theorem}

The proof of Theorem~\ref{thm:local-torus-heat} can be found in Section~\ref{app:local-multiscale-proofs}. This result shows that the weighted scale-averaging construction removes the high-frequency obstruction of fixed-bandwidth drifting. By assigning greater relative weight to fine scales while retaining information from coarser scales, it restores local exponential convergence to the target, both in the score norm and in relative entropy.

\subsection{Local exponential convergence for multihead Ornstein-Uhlenbeck drifting on \texorpdfstring{$\mathbb{R}^d$}{Rd}}
\label{subsec:local-ou}

On $\R^d$ the corresponding local result uses the Ornstein-Uhlenbeck semigroup rather than the translation-invariant heat semigroup.  Let
\[
    \dd\gamma(x)=(2\pi)^{-d/2}e^{-|x|^2/2}\,\dd x,
    \qquad
    L_\gamma=\Delta-x\cdot\nabla,
    \qquad
    Q_s=e^{sL_\gamma}.
\]
Equivalently, Mehler's formula gives, for every bounded measurable function
$f\colon \mathbb{R}^d \to \mathbb{R}$,
\begin{equation}\label{eq:Mehler1}
    Q_s f(x)
    =
    \mathbb{E}_{Z\sim \mathcal N (0,I_d)}\left[
        f\left(e^{-s}x+\sqrt{1-e^{-2s}}\,Z\right)
    \right].
\end{equation}
For a target
\[
    \rho^\star=e^{a}\gamma,
    \qquad
    \int_{\R^d}e^{a}\,\dd\gamma=1,
\]
and $h=\log(\rho/\rho^\star)$, define
\begin{equation}
    \Phi^{\rm OU}_S[h]
    =\frac{1}{S}\int_0^S
      \log\frac{Q_s(e^{h+a})}{Q_s(e^{a})}\,\dd s.
    \label{eq:main-ou-multiscale-potential}
\end{equation}
The population OU multihead drifting equation is
\begin{equation}
    \partial_t\rho_t
    =\nabla\cdot\bigl(\rho_t\nabla\Phi^{\rm OU}_S[h_t]\bigr),
    \qquad
    h_t=\log\frac{\dd\rho_t}{\dd\rho^\star},
    \qquad
    \rho_0=\gamma.
    \label{eq:main-ou-multiscale-flow}
\end{equation}

We use the Ornstein-Uhlenbeck kernel on $\mathbb{R}^d$ because it is naturally adapted to the Gaussian reference measure: it preserves $\gamma$ and the denominators $Q_s(e^a)$ remain uniformly controlled when $a$ is small. This avoids the weighted tail estimates required for the heat kernel and allows the argument to be carried out in unweighted H\"older norms. We nevertheless regard this choice as mainly technical. Indeed, by Mehler's formula~\eqref{eq:Mehler1}, Ornstein-Uhlenbeck smoothing is heat smoothing combined with a deterministic spatial contraction and therefore has the same short-scale regularization that drives the scale-averaging coercivity mechanism. We consequently expect an analogous local exponential convergence result for the heat kernel on $\mathbb{R}^d$, once the proof is supplemented with suitable weighted denominator bounds, tail-propagation estimates, and weighted stability estimates. 

\begin{theorem}[Local exponential convergence for OU multihead drifting]
\label{thm:local-ou}
Fix $S>0$ and $d\ge1$.  There exist constants
\[
    \varepsilon_S>0,
    \qquad
    c_{d,S}>0,
    \qquad
    C_S<\infty,
\]
depending only on $d$ and $S$, such that the following holds.  Let $a\in C^{2,1}(\R^d)$ satisfy
\[
    \int_{\R^d}e^{a}\,\dd\gamma=1,
    \qquad
    \|a\|_{C^{2,1}}\le \varepsilon_S.
\]
Then there exists a unique global solution $(\rho_t)_{t\ge0}$ to
\eqref{eq:main-ou-multiscale-flow}, in the sense of
Section~\ref{subsec:loc-endpoint-solutions}.  Furthermore, for every $t\ge0$, it
satisfies
\begin{equation}
    \left\|\nabla\log\frac{\dd\rho_t}{\dd\rho^\star}\right\|_{C^{0,1}}
    \le
    C_S e^{-c_{d,S}t}\|\nabla a\|_{C^{0,1}},
    \label{eq:main-ou-score-decay}
\end{equation}
and
\begin{equation}
    \KL(\rho_t\mid\rho^\star)
    \le
    C_S e^{-2c_{d,S}t}\|\nabla a\|_{C^{0,1}}^2.
    \label{eq:main-ou-entropy-decay}
\end{equation}
\end{theorem}

The proof of Theorem~\ref{thm:local-ou} can be found in Section~\ref{app:local-multiscale-proofs}. As in the torus setting, the scale average removes the fixed-bandwidth blind spot. The resulting multihead dynamics therefore restores local exponential convergence to the target, both in the score norm and in relative entropy.

\begin{remark}[Finite versus continuum multihead fields]
The uniform spectral gap above uses the full interval $s\in[0,S]$, and therefore includes arbitrarily small scales. Any finite collection $s_1,\ldots,s_J>0$ has a smallest positive scale, so it does not retain a uniform gap over arbitrarily high frequencies. It can nevertheless approximate the continuum mechanism over any fixed range of resolved frequencies. Theorems~\ref{thm:local-torus-heat} and~\ref{thm:local-ou} concern the continuum population flow; the finite-scale experiments below test its finite-resolution analogue.
\end{remark}

\section{Numerical evidence for multihead acceleration}
\label{sec:numerical-experiments}
This section presents Experiment~1, a simple test of the
scale-separation mechanism predicted by the analysis.  We study a smooth
oscillatory target in $\R^2$ first at the population-PDE level and then with
finite data and minibatches.  The two more complex finite-data experiments (a
regularized curve-supported target in $\R^2$ and a target made of parallel
sheets in $\R^{20}$) are reported in
\hyperref[app:additional-numerical-experiments]{Appendix~\ref*{app:additional-numerical-experiments}}.

All reported experiments are deliberately designed to contain components at
distinct spatial or frequency scales.  This setting highlights the regime in
which the multihead method can benefit from resolving several scales
simultaneously.  We also conducted simpler experiments with less pronounced
frequency separation; in those cases, the multihead and single-scale methods
achieved essentially the same errors.  This indicates that the multihead
advantage is associated with genuinely multiscale features rather than with a
systematic improvement in every setting.

\subsection{Experimental design and common protocol}
\label{subsec:new-common-protocol}
In all reported experiments, the initial distribution is the standard Gaussian,
$\rho_0=\gamma$.
\paragraph{Methods and scale selection.}
Each reported experiment compares Gaussian single-scale, Gaussian multihead, Laplace
single-scale, and Laplace multihead drifting.  The Gaussian field uses the
normalized score mismatch
$-\nabla\log(P_s q/P_s\rho^\star)$, without the usual factor $2s$.
The Laplace field is divided by its leading small-bandwidth factor
$(d+1)\tau^2$. These normalizations remove the principal scale-dependent
slowdown and allow a fair comparison of single-scale and multihead algorithms. Note that multiplying any of these vector fields by a positive constant does not change the paths followed by the particles; it only changes the speed at which the particles move along those paths. In this setting, a multihead field is therefore an average of the
corresponding single-scale fields.

For each method, we perform several validation runs with different scale locations and weights, select the best-performing configuration, and then use these choices for conducting the reported test runs. All comparisons are at equal evolution time, not equal computation: a multihead update evaluates several fields. The experiments therefore measure acceleration of the population dynamics rather than computational runtime.

\paragraph{Finite-data comparisons.}
In the finite-data experiments, the target is represented by a fixed training
sample and the evolving law by particles.  All finite-data experiments use
$50{,}000$ target observations and at each iteration draw one target
minibatch and one particle minibatch, each of size $2048$. 

\paragraph{KL evaluation and computational interpretation.}
The KL error is evaluated independently of the grids and minibatches used to
construct the velocity.  Particle histograms are formed on a finer grid,
smoothed with the fixed bandwidth of size 0.05, and compared with an analytic or
high-accuracy population representation of the target. 

\paragraph{Ornstein-Uhlenbeck comparison.}
We do not report the OU multihead method separately because, across our experiments, its errors and convergence curves were consistently close to those of the Gaussian multihead method. We therefore show only the Gaussian multihead results.

\subsection{A simple multiscale target}
\label{subsec:experiment1-new-protocol}
This experiment is the most controlled test of the scale-separation mechanism.  The target
contains one broad and one fine oscillatory component.  The population-PDE version removes all
sampling and estimation effects and therefore tests the intrinsic dynamics.
The finite-data version then asks whether the same ordering survives particle
and minibatch noise.  This setting is the closest experiment to the Gaussian-reference and
Fourier intuition developed in Section~\ref{sec:natural-exponential-convergence}.

\paragraph{Target.}
The target density is
\begin{equation} 
    \rho^\star(x)
    =Z^{-1}\exp\Big(\sin(20x_1+7x_2)-5\cos(x_1+x_2)\Big)\gamma(x).
    \label{eq:experiment1-target-revised}
\end{equation}
The cosine term creates a broad modulation, whereas the sine term creates a
much finer oscillation.

\subsubsection{Population PDE}
\label{subsec:experiment1-new-pde-protocol}

\paragraph{Protocol.}
Here $q_t$ and $\rho^\star$ are represented directly, so no samples,
particles, minibatches, or density estimator enter the dynamics.  We solve the
non-local continuity equation on a $224\times224$ grid over
$[-5.5,5.5]^2$ with a conservative finite-volume scheme and an adaptive time
step bounded by $0.02$.  All methods are run to $T=20$. 
    
\begin{table}[H]
\centering
\small
\begin{tabular}{lll}
\toprule
Kernel & Single-scale & Multihead\\
\midrule
Gaussian & $s=0.005$ & 12-node quadrature on $[0,0.08]$ \\
Laplace & $\tau=0.098$ & 12-node quadrature on $[0,0.135]$  \\
\bottomrule
\end{tabular}
\caption{Scale configurations for Experiment~1 at the population-PDE level.}
\label{tab:experiment1-pde-configurations}
\end{table}

\begin{figure}[H]
    \centering
    \includegraphics[width=0.88\textwidth]{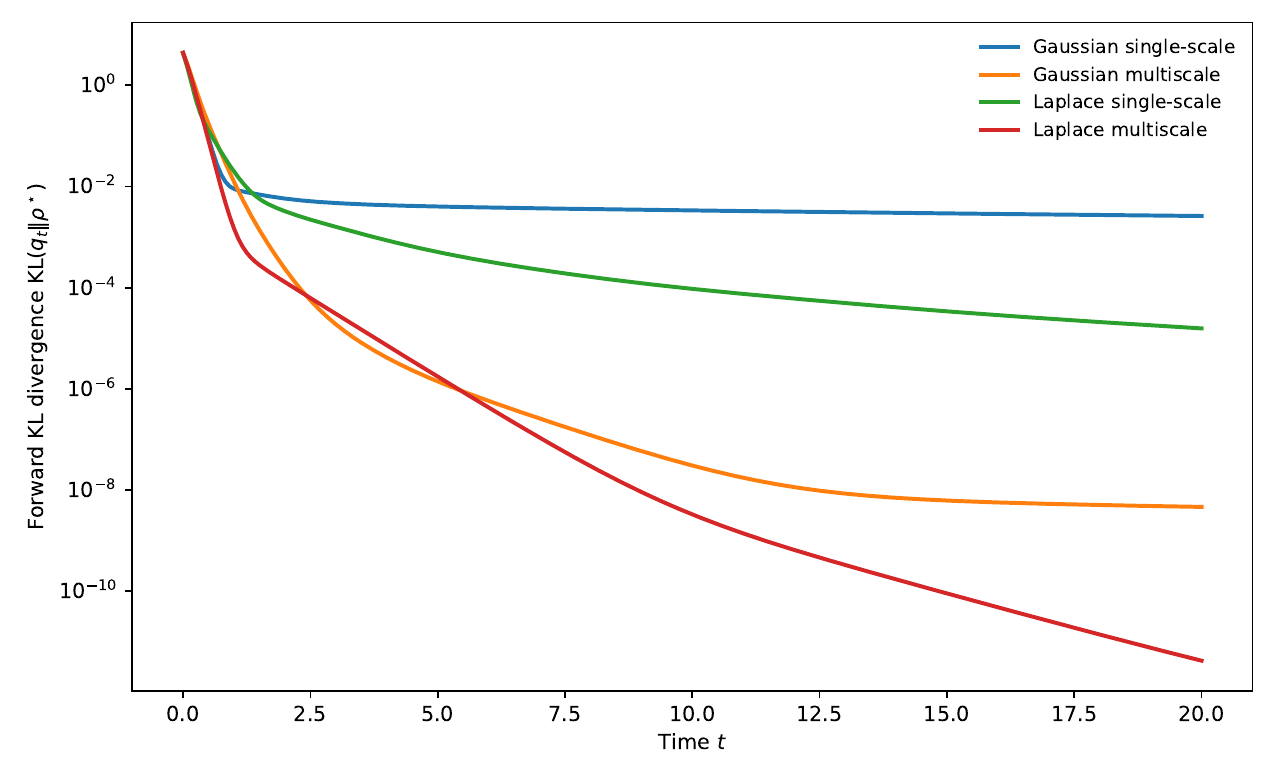}
    \caption{Forward KL for Experiment~1 under the deterministic PDE protocol.
    Both multihead flows continue to decrease after the corresponding
    single-scale flow slows down.}
    \label{fig:experiment1-new-pde-kl}
\end{figure}

\begin{table}[H]
\centering
\small
\begin{tabular}{lccc}
\toprule
Kernel & Single-scale KL & Multihead KL\\
\midrule
Gaussian & $2.6\times10^{-3}$ & $4.6\times10^{-9}$\\
Laplace  & $1.5\times10^{-5}$ & $4.2\times10^{-12}$ \\
\bottomrule
\end{tabular}
\caption{Experiment~1 population-PDE results at the common final time $T=20$.}
\label{tab:experiment1-new-pde-results}
\end{table}

At the same physical time, the multihead error is smaller for both kernels.
The Gaussian multihead flow reaches a KL error of $10^{-3}$ at time $1.58$,
whereas the Gaussian single-scale flow does not reach this threshold by time
$20$.  The corresponding Laplace times are $1.07$ and $3.75$.

This behavior is the numerical counterpart of the spectral mechanism in
Section~\ref{sec:natural-exponential-convergence}: a fixed smoothing scale can
make the fine oscillatory mode nearly invisible, whereas averaging across
scales retains a contribution from resolutions at which that mode is visible.
It is consistent with the scale-averaging spectral-gap intuition, but it is not a
direct numerical verification of Theorem~\ref{thm:local-ou}; rather, it tests a broader setting in which the initial condition $\rho_0$ is not close to the target $\rho^\star$. Within this broader setting, the numerical results further suggest that the Laplace multihead method outperforms the other methods, as it attains both the smallest final KL error and the earliest crossing of the $10^{-3}$ threshold.

\subsubsection{Finite data and minibatches}
\label{subsubsec:experiment1-finite-data}
The population experiment deliberately removes statistical error.  This
finite-data counterpart tests whether the multihead advantage remains after
the velocity is estimated from one fixed training sample and paired
minibatches.

\paragraph{Protocol.}
The selected scale configurations are

\begin{table}[H]
\centering
\small
\begin{tabular}{lll}
\toprule
Kernel & Single-scale & Multihead \\
\midrule
Gaussian & $s=0.05$ & $(0.006,0.04,0.06,0.08)$  \\
Laplace & $\tau=0.18$ & $(0.12,0.21,0.30,0.35)$  \\
\bottomrule 
\end{tabular}    
\caption{Scale configurations for Experiment~1 finite-data runs.}
\label{tab:experiment1-finite-configurations}
\end{table}

\begin{figure}[H]
    \centering
    \includegraphics[width=0.88\textwidth]{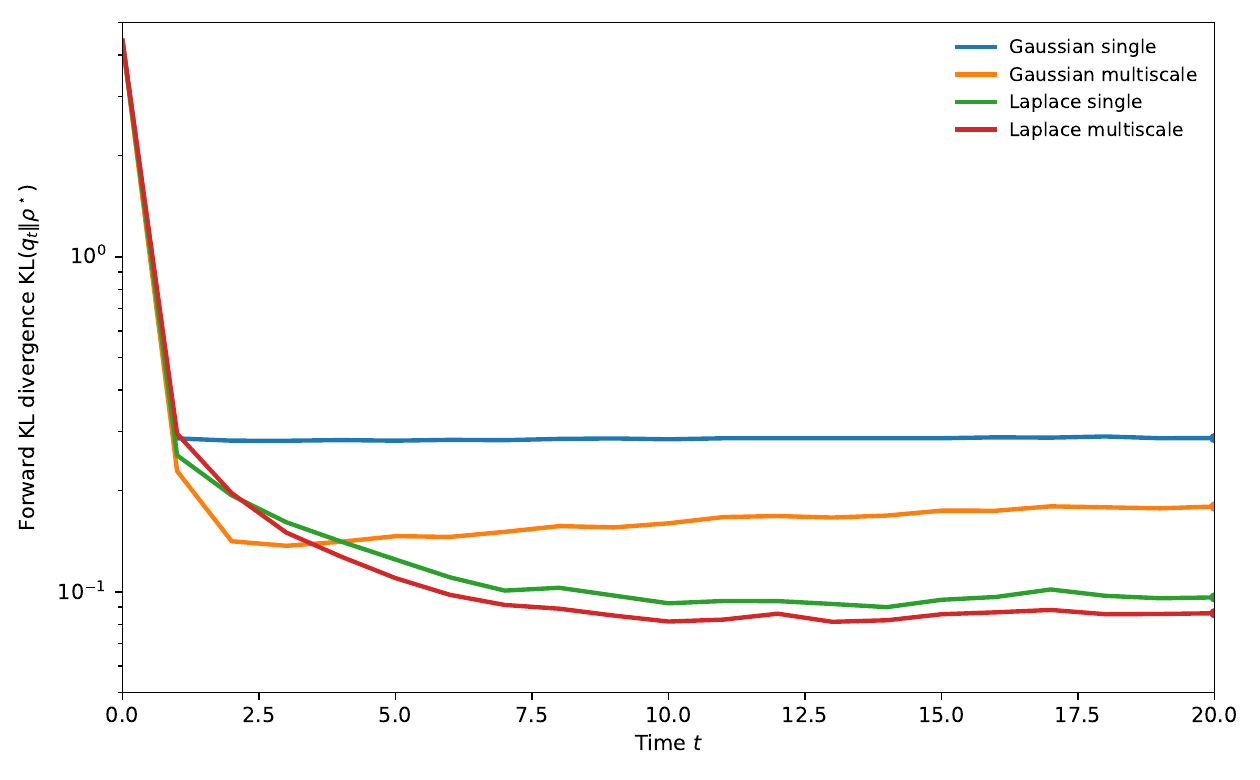}
    \caption{Mean forward KL over three paired finite-data repetitions for
    Experiment~1.}
    \label{fig:experiment1-new-protocol-kl}
\end{figure}

\begin{table}[H]
\centering
\small
\begin{tabular}{lccc}
\toprule
Kernel & Single-scale KL & Multihead KL \\
\midrule
Gaussian & $0.287$ & $0.179$ \\
Laplace  & $0.096$ & $0.086$ \\
\bottomrule
\end{tabular}
\caption{Experiment~1 finite-data results at $T=20$.}
\label{tab:experiment1-new-protocol-results}
\end{table}

Relative to the corresponding single-scale method, multihead reduces the mean
terminal KL by $37.5\%$ for Gaussian and by $10.3\%$ for Laplace. Under this finite-data protocol, the same ordering as in the population experiment persists despite sampling and minibatch noise: the multihead flow achieves a lower mean terminal KL for both kernels. Although statistical noise leads to substantially larger terminal errors and reduces the separation between some methods, the multihead flow achieves a lower mean KL for both kernels.

\paragraph{Additional experiments.}
The main text focuses on the controlled target above. \hyperref[app:additional-numerical-experiments]{Appendix~\ref*{app:additional-numerical-experiments}} reports two more complex finite-data experiments: a target concentrated near a curve in $\R^2$ and a target formed by parallel sheets in $\R^{20}$. Both use the protocol of Subsection~\ref{subsec:new-common-protocol}, so their results are directly comparable with Experiment~1.

\section{Conclusion and future directions}
\label{sec:conclusion-future-work}

This paper gives, to the best of our knowledge, the first quantitative convergence results for drifting at the level of its population PDE.  The analysis isolates the role of the smoothing scale in the dynamics.  For a fixed bandwidth, the discrepancy between the current law and the target is observed only after smoothing, and therefore small-scale errors are only weakly visible to the velocity field.  In the Gaussian case this loss is exponential in the frequency and leads to the logarithmic lower bound of Proposition~\ref{prop:nonlinear-log-lower-bound}; for the true Laplace-type drifting field the loss is polynomial and gives Proposition~\ref{prop:true-laplace-polynomial-lower-bound}.  Thus fixed-scale drifting does not admit a uniform fast convergence theory over natural finite-regularity target classes.

We overcome this obstruction by introducing a multihead drifting flow, in which the velocity field averages the score mismatch across a continuum range of smoothing scales. Theorems~\ref{thm:local-torus-heat} and~\ref{thm:local-ou} show that, for targets sufficiently close to the flat or Gaussian reference measure in the stated $C^{2,1}$ topology, the resulting population flow converges exponentially fast, both in logarithmic score and in relative entropy. This gives a theoretical explanation for the empirical advantage of multihead drifting: it prevents fine-scale discrepancies from being smoothed away before they can influence the drift.

The experiments indicate that the same mechanism remains relevant beyond the perturbative setting. In the smooth, curve-supported, and stratified examples, using several scales improves the distributional match. These observations motivate the two questions considered below: whether the local coercive estimate can be extended globally, and whether the population advantage remains after finite-sample, discretization, and map-learning errors are included.

\subsection{Beyond local convergence}

A natural next step is to replace the perturbative convergence theory by a non-local one.  The proof of Theorems~\ref{thm:local-torus-heat} and~\ref{thm:local-ou} is based on a bootstrap argument in which the logarithmic score remains in a small neighborhood of equilibrium.  In that neighborhood, the linear spectral gap produced by scale averaging dominates the nonlinear terms.  A global theory would require a mechanism that prevents the flow from leaving the region where the multihead field remains coercive.

One possible route is to identify a Lyapunov functional or a dissipation inequality adapted to the scale-averaged logarithmic potential.  Such an estimate would have to control the same objects that appear in the local proof: the smoothed logarithmic ratios, the scale-wise denominators, and the covariance terms generated by the tilted kernels.  In the local analysis these quantities are controlled by smallness in $C^{2,1}$.  Away from this regime, one would need bounds preventing degeneracy of the denominators and growth of the covariance terms along the flow.

A second route is to keep the dynamical viewpoint and look for stability estimates that are localized along the evolving law.  Uniform ambient bounds are unlikely to be the right quantities for rough or nearly singular targets.  For instance, if the target is concentrated near a manifold or a stratified set, normal directions may be strongly contracting while tangential directions are governed by the intrinsic geometry of the support.  The correct stability theory may therefore involve distribution-weighted norms, localized one-sided Lipschitz bounds, or estimates along couplings of the relevant trajectories.  Establishing such localized stability estimates would be a first step toward a convergence theory for non-perturbative and lower-dimensional targets.

\subsection{From the population PDE to the true drifting algorithm}

A second direction is to pass from the population PDE to the finite-sample drifting algorithm of~\cite{deng2026drifting}.  The present paper removes statistical error, time-discretization error, and map-learning error in order to isolate the population mechanism.  A theorem for the true algorithm should keep the population flow as the reference object and prove that the learned discrete trajectory tracks it over the time horizon on which the population equation has already converged.

At a schematic level, one may organize the error as
\[
\begin{aligned}
    W_2(\widehat\rho_N,\rho^\star)
    \leq\;&
    \underbrace{W_2(\rho_T,\rho^\star)}_{\text{population convergence}}
    +
    \underbrace{W_2(\rho_T^\Delta,\rho_T)}_{\text{time discretization}}
    +
    \underbrace{W_2(\rho_N^{n,\Delta},\rho_T^\Delta)}_{\text{empirical field error}}
    +
    \underbrace{W_2(\widehat\rho_N,\rho_N^{n,\Delta})}_{\text{map-learning error}} .
\end{aligned}
\]
Here $\rho_T$ denotes the population multihead flow, $\rho_T^\Delta$ the corresponding time-discretized dynamics with the exact population field, $\rho_N^{n,\Delta}$ the dynamics driven by the empirical multihead field, and $\widehat\rho_N$ the law produced by the learned maps.

The first term is precisely the population problem studied here.  In the local regime it is controlled by Theorems~\ref{thm:local-torus-heat} and~\ref{thm:local-ou}; in a future global theory it would be replaced by the non-local convergence estimate discussed above.

The second term is deterministic.  It should be controlled by a standard discretization argument for the population flow: one propagates perturbations using the one-sided Lipschitz constant of the velocity, while the full spatial and time regularity of the velocity enters through the local truncation error.

The empirical field error is the statistical part of the problem.  In the multihead algorithm the velocity is built from empirical smoothed log-ratios, and therefore contains both numerator and denominator errors at several scales.  A finite-sample proof should control this error in a trajectory-wise norm.  At small scales one expects to need a lower scale cutoff, denominator clipping, and tail localization.  The bias introduced by these operations should be controlled by regularity estimates for the exact score, while the random part should follow from concentration estimates for the smoothed empirical quantities.  The localized score-regularity results of~\cite{stephanovitch2025regularity} provide a natural analytic input for this step.

The last term is the error made when replacing the ideal update map by a learned map.  Conditionally on the empirical field, this should be bounded by a regression inequality for the chosen generator class.  The approximation term in such an inequality requires regularity of good oracle maps.  A natural construction is to choose a smooth transport map $T_i^\star$ pushing the reference Gaussian measure to the population law at step $i$, and then compose it with the exact population update.  Smooth transport-map estimates, such as those in~\cite{stephanovitch2024smooth}, should then give the regularity needed to approximate this oracle map by the chosen architecture.

These four estimates must be combined over the whole training horizon. The main difficulty is to show that the statistical, discretization, and map-learning errors do not grow faster than the population error decays. Such a result would connect the population PDE studied here with the finite-data drifting algorithm.

\section{Acknowledgements}
We warmly thank Maxence Adly, Antonin Chodron de Courcel, Cyril Imbert, Kimia Nadjahi and Gabriel Peyr\'{e} for very insightful discussions.

\printbibliography

\clearpage
\appendix
\appendixtoctrue
\appendixtableofcontents

\section{Additional numerical experiments}
\label{app:additional-numerical-experiments}
This appendix reports the two more complex finite-data experiments deferred
from Section~\ref{sec:numerical-experiments}.  They extend the controlled
smooth-target study to a distribution concentrated near a curved set and to a
high-dimensional target with many parallel strata.  Both experiments use the
common finite-data and KL-evaluation protocol described in
Subsection~\ref{subsec:new-common-protocol}.

\subsection{Experiment 2: a target concentrated near a curve in $\R^2$}
\label{subsec:experiment2-new-protocol}

Experiment~1 in Subsection~\ref{subsec:experiment1-new-protocol} encodes scale
separation through Fourier oscillations in a smooth full-dimensional density.  Experiment~2 asks whether the same phenomenon
persists when scale separation is geometric.  The target has a thin normal
direction, a curved global support, and five distinct frequency bands along the
intrinsic coordinate.

\paragraph{Target.}
Let
\begin{equation}
    \Phi(t)=\left(t,\;0.25\sin(3t)+0.09\sin(17t)\right),
    \qquad t\in[-3.8,3.8].
    \label{eq:experiment2-curve}
\end{equation}
A nonuniform intrinsic density on the curve combines five frequency bands
($0.8$, $2.4$, $7$, $18$, and $42$), and
$\mu=\Phi_{\#}(\rho(t)\,\dd t)$ is the resulting singular measure.  The target
used in both training and evaluation is the heat-regularized law
\begin{equation}
    \rho^\star=P_\varepsilon\mu,
    \qquad \varepsilon=0.0012.
    \label{eq:experiment2-regularized-target}
\end{equation}

\paragraph{Protocol.}
The selected scale configurations are

\begin{table}[H]
\centering
\small
\begin{tabular}{lll}
\toprule
Kernel & Single-scale & Multihead \\
\midrule
Gaussian & $s=0.006$ & $(0.0025,0.005,0.01,0.04)$\\
Laplace & $\tau=0.16$ & $(0.06,0.12,0.20,0.36)$\\
\bottomrule
\end{tabular}
\caption{Scale configurations for Experiment~2.}
\label{tab:experiment2-configurations}
\end{table} 

\begin{figure}[H]
    \centering
    \includegraphics[width=0.88\textwidth]{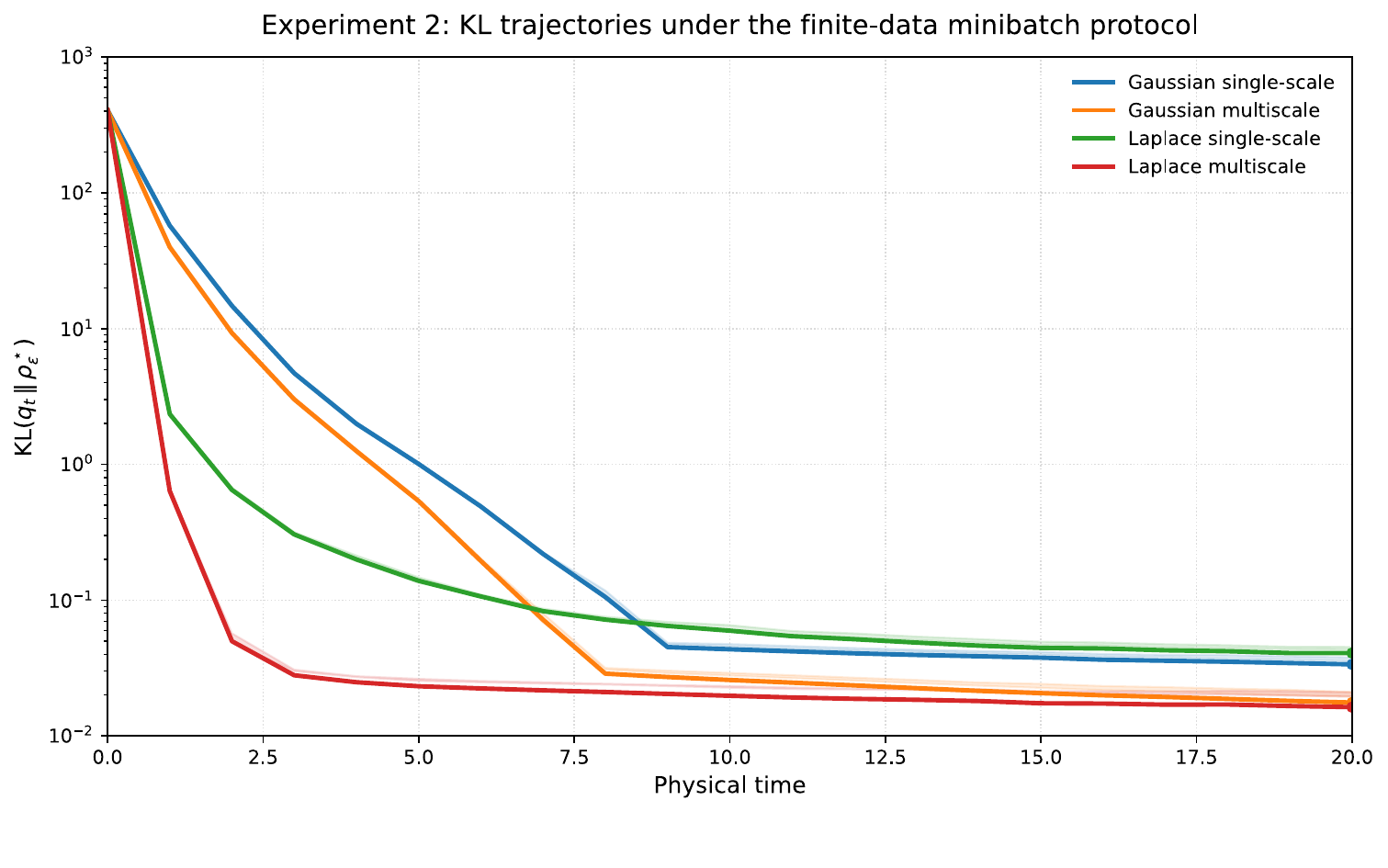}
    \caption{Pooled forward KL for Experiment~2.  Each curve uses the
    $150{,}000$ particles from the three paired repetitions; the shaded range
    is obtained by leaving out one repetition at a time.}
    \label{fig:experiment2-new-protocol-kl}
\end{figure}

\begin{table}[H]
\centering
\small
\begin{tabular}{lccc}
\toprule
Kernel & Single-scale KL & Multihead KL \\
\midrule
Gaussian & $0.0337$ & $0.0176$ \\
Laplace  & $0.0408$ & $0.0163$\\
\bottomrule
\end{tabular}
\caption{Experiment~2 KL estimates at the final time $T=20$.}
\label{tab:experiment2-new-protocol-results}
\end{table}

We see that multihead reduces the terminal KL by $47.7\%$ for Gaussian and by
$60.0\%$ for Laplace. The result is consistent with the fixed-scale obstruction in a geometric form.
A single bandwidth must compromise between moving mass along the full curve,
resolving its intermediate oscillations, and matching the narrow regularized
normal direction.  The multihead field contains scales adapted to each of
these tasks and continues to decrease after the single-scale curves flatten.
Because the target is concentrated near a lower-dimensional set and lies far from the perturbative reference laws, these results suggest that the theoretical mechanism extends beyond the regimes established in this paper.

\subsection{Experiment 3: many parallel sheets in $\R^{20}$}
\label{subsec:experiment3-new-protocol}

The first two experiments create different scales through oscillations or a
curved support.  Here the scales come from the positions and weights
of many parallel sheets. 

\paragraph{Target.}
Write $x=(u,z)$ with $u\in\R^{18}$ and $z\in\R^2$.  For
$i,j\in\{-3,\ldots,3\}$, set
\begin{equation}
    c_{ij}=0.5(i,j),
    \qquad
    S_{ij}=\R^{18}\times\{c_{ij}\}.
    \label{eq:experiment3-sheets}
\end{equation}
Thus there are $49$ parallel $18$-dimensional sheets.  After adding Gaussian
noise, the target density is
\begin{equation}
\begin{aligned}
    \rho^\star(u,z)
    ={}&\mathcal N(0,0.1^2I_{18})(u)
    \sum_{i,j=-3}^{3} w_{ij}\,
    \mathcal N(c_{ij},0.1^2I_2)(z),\\
    w_{ij}
    ={}&\frac1Z
    \exp\!\left(-\frac{\|c_{ij}\|^2}{2}\right)
    \left(1+0.8(-1)^{i+j}\right).
\end{aligned}
    \label{eq:experiment3-target}
\end{equation}
The first factor makes sheets near the center more likely.  The second
factor alternates between $1.8$ and $0.2$, so neighboring sheets have different
weights.  This gives one pattern across the full grid and another between
nearby sheets.
Before the Gaussian noise is added, the target is supported exactly on the
union of the $49$ sheets in~\eqref{eq:experiment3-sheets}.

\paragraph{Protocol.}
The selected scale configurations are
\begin{table}[H]
\centering
\small
\begin{tabular}{llll}
\toprule
Kernel & Single scale & Four scales & Weights \\
\midrule
Gaussian
& $s=0.32$
& $(0.05,0.08,0.30,0.52)$
& $(0.01,0.04,0.35,0.60)$ \\
Laplace
& $\tau=0.23$
& $(0.08,0.13,0.22,0.32)$
& $(0.005,0.020,0.550,0.425)$ \\
\bottomrule
\end{tabular}%
\caption{Scale configurations for Experiment~3.}
\label{tab:experiment3-configurations}
\end{table}
\begin{figure}[H]
    \centering
    \includegraphics[width=0.88\textwidth]{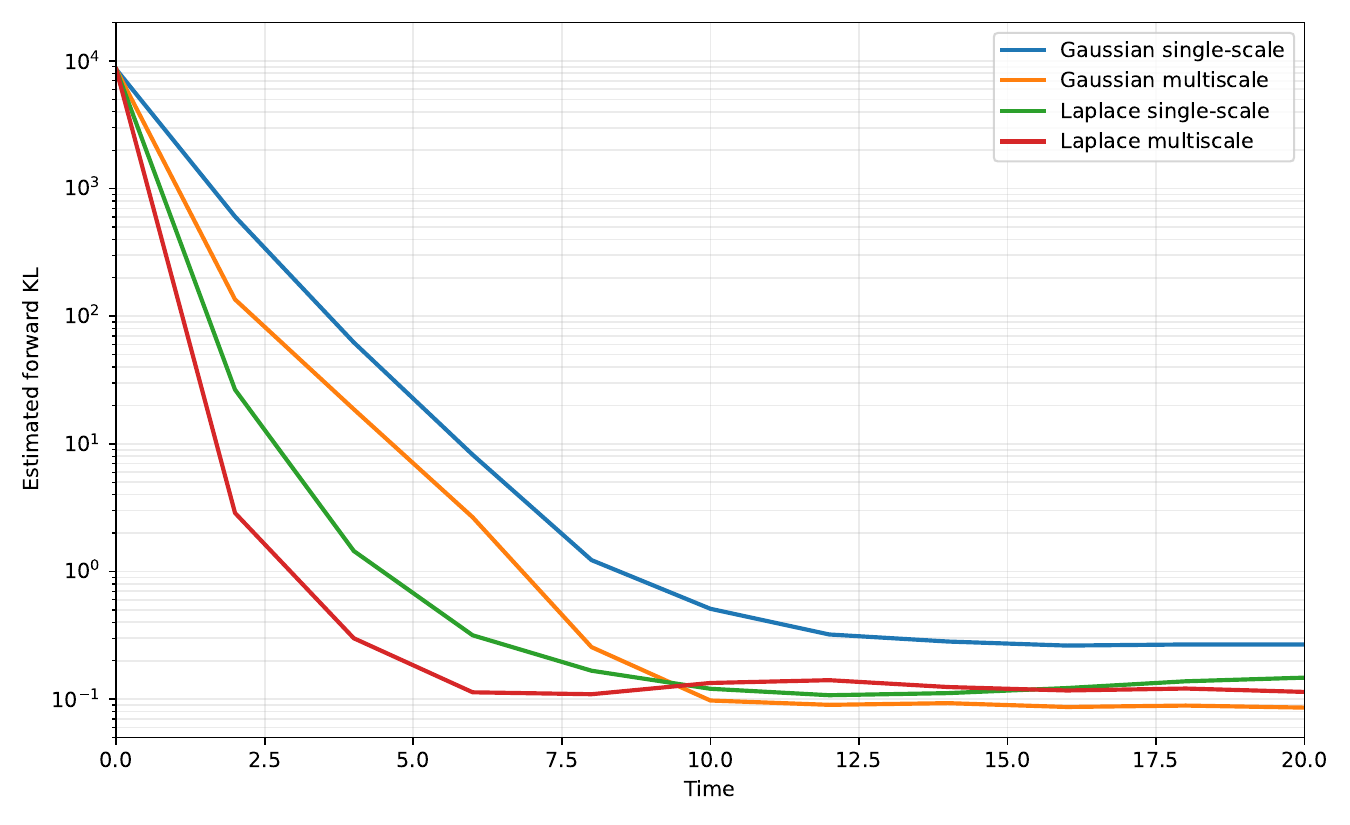}
    \caption{KL error for Experiment~3.}
    \label{fig:experiment3-new-protocol-kl} 
\end{figure}

\begin{table}[H]
\centering
\small
\begin{tabular}{lcc}
\toprule
Kernel & Single-scale KL & Multihead KL \\
\midrule
Gaussian & $1.15419$ & $0.83133$ \\
Laplace  & $0.75062$ & $0.67393$ \\
\bottomrule
\end{tabular}
\caption{Experiment~3 KL estimates at $T=20$.  The estimates use the combined
$6144$ final particles from the three runs.}
\label{tab:experiment3-new-protocol-results}
\end{table}

Using four scales lowers the final KL by $28.0\%$ for Gaussian and by $10.2\%$
for Laplace.  This is expected as a relatively large bandwidth helps move mass across the
full grid, but it smooths over the difference between neighboring heavy and
light sheets.  A smaller bandwidth sees this local difference, but it is less
effective for transport across the grid.  A single bandwidth must compromise
between the two tasks.  The four-scale field contains useful resolutions for
both, so its KL continues to decrease after the single-scale curve begins to
flatten.

Taken together with Experiment~1 in the main text, these results show that the multihead variants achieve lower terminal KL error than the corresponding single-scale methods for both Gaussian and Laplace kernels across all three targets.  The improvement is visible in the deterministic population flow and remains present with finite samples and minibatches, despite the larger statistical error.  These results support using a range of bandwidths for genuinely multiscale targets, while indicating that scale selection and computational cost merit further study.

\section{Well-posedness of the fixed-bandwidth flows}
\label{app:fixed-scale-wellposedness}

The Gaussian and Laplace lower bounds rely on the same characteristic
well-posedness result. Uniform positivity of the denominator yields spatial and
measure Lipschitz bounds for the normalized velocity. Banach's fixed-point
theorem applied to the characteristic flow gives local existence, and iteration
yields a global solution; the same characteristic representation gives
uniqueness and the density bounds stated below. In this section, $W_1$ denotes
the $1$-Wasserstein distance on
$\T^d$ associated with the geodesic distance.

\begin{proposition}[Global well-posedness of normalized fields]
\label{prop:fixed-scale-wellposedness}
Let $J:\T^d\to(0,\infty)$ and $K:\T^d\to\R^d$ be periodic Lipschitz
kernels, and define for $\mu\in\mathcal P(\T^d)$
\begin{equation}
    \mathcal B[\mu](x):=\frac{K*\mu(x)}{J*\mu(x)}.
    \label{eq:normalized-convolution-field}
\end{equation}
For every target $\nu\in\mathcal P(\T^d)$ and every initial law
$\mu_0\in\mathcal P(\T^d)$, the equation
\begin{equation}
    \partial_t\mu_t
    +\diver\bigl(\mu_t(\mathcal B[\nu]-\mathcal B[\mu_t])\bigr)=0,
    \qquad
    \mu_{t=0}=\mu_0,
    \label{eq:normalized-convolution-flow}
\end{equation}
has a unique global narrowly continuous distributional solution.  It is the
pushforward of $\mu_0$ by the unique self-consistent characteristic flow
\begin{equation}
    \dot X_t(x)
    =\mathcal B[\nu](X_t(x))-\mathcal B[\mu_t](X_t(x)),
    \qquad
    X_0(x)=x,
    \qquad
    \mu_t=(X_t)_\#\mu_0.
    \label{eq:fixed-scale-characteristics}
\end{equation}
If $\mu_0=q_0\,\dd x$ with $q_0\in L^\infty(\T^d)$ and
$\operatorname*{ess\,inf}q_0>0$, then $\mu_t=q_t\,\dd x$ for all $t\ge0$ and,
for every $T<\infty$, there is $C_T<\infty$, depending only on $T,d,J,K$,
such that
\begin{equation}
    e^{-C_T}\operatorname*{ess\,inf}q_0
    \le q_t\le e^{C_T}\|q_0\|_{L^\infty},
    \qquad 0\le t\le T.
    \label{eq:fixed-scale-density-bounds}
\end{equation}
Finally, every translation symmetry shared by $\nu$ and $\mu_0$ is preserved
by $\mu_t$.
\end{proposition}

\begin{proof}
Set
\begin{equation}
    m_J:=\min_{\T^d}J>0,
    \qquad
    M_0:=\frac{\|K\|_{L^\infty}}{m_J},
    \qquad
    C_B:=
    \frac{\Lip(K)}{m_J}
    +\frac{\|K\|_{L^\infty}\Lip(J)}{m_J^2}.
    \label{eq:fixed-scale-constants}
\end{equation}
The constants are finite because the torus is compact and $J$ is strictly
positive.

\paragraph{Uniform bounds for the normalized field.}
For every $\mu\in\mathcal P(\T^d)$,
\begin{equation}
    J*\mu\ge m_J,
    \qquad
    \|K*\mu\|_{L^\infty}\le\|K\|_{L^\infty},
\end{equation}
so $\|\mathcal B[\mu]\|_{L^\infty}\le M_0$.  If $x,x'\in\T^d$, then
\[
\begin{aligned}
    |K*\mu(x)-K*\mu(x')|
    &\le \Lip(K)d_{\T^d}(x,x'),\\
    |J*\mu(x)-J*\mu(x')|
    &\le \Lip(J)d_{\T^d}(x,x').
\end{aligned}
\]
Using
\[
    \frac{a}{b}-\frac{a'}{b'}
    =\frac{a-a'}{b}
     +a'\frac{b'-b}{bb'}
\]
and the lower bound $b,b'\ge m_J$ gives
\begin{equation}
    \Lip\bigl(\mathcal B[\mu]\bigr)\le C_B
    \qquad\text{uniformly in }\mu.
    \label{eq:fixed-scale-spatial-lipschitz}
\end{equation}
The same calculation controls the dependence on the measure.  Indeed, for any
coupling $\pi\in\Pi(\mu,\widetilde\mu)$ and every $x\in\T^d$,
\[
\begin{aligned}
    |K*\mu(x)-K*\widetilde\mu(x)|
    &\le \Lip(K)
      \int_{\T^d\times\T^d}d_{\T^d}(y,z)\,\dd\pi(y,z),\\
    |J*\mu(x)-J*\widetilde\mu(x)|
    &\le \Lip(J)
      \int_{\T^d\times\T^d}d_{\T^d}(y,z)\,\dd\pi(y,z).
\end{aligned}
\]
Taking the infimum over couplings yields
\begin{equation}
    \|\mathcal B[\mu]-\mathcal B[\widetilde\mu]\|_{L^\infty}
    \le C_B W_1(\mu,\widetilde\mu).
    \label{eq:fixed-scale-measure-lipschitz}
\end{equation}

\paragraph{Local construction by a contraction.}
Fix $T>0$ for the moment and let
\[
    \mathcal X_T:=C([0,T];\mathcal P(\T^d)),
    \qquad
    d_T(m,\widetilde m)
    :=\sup_{0\le t\le T}W_1(m_t,\widetilde m_t).
\]
Because $\T^d$ is compact, $\mathcal P(\T^d)$ is complete for $W_1$, hence so
is $(\mathcal X_T,d_T)$.  For $m\in\mathcal X_T$, define
\begin{equation}
    v_t^m(x):=\mathcal B[\nu](x)-\mathcal B[m_t](x).
\end{equation}
By~\eqref{eq:fixed-scale-spatial-lipschitz},
\begin{equation}
    \|v_t^m\|_{L^\infty}\le2M_0,
    \qquad
    \Lip(v_t^m)\le2C_B,
    \label{eq:fixed-scale-velocity-bounds}
\end{equation}
uniformly in $m$ and $t$.  Moreover, narrow continuity of $m$ is equivalent
to $W_1$-continuity on the compact torus, so
\eqref{eq:fixed-scale-measure-lipschitz} implies that $t\mapsto v_t^m$ is
continuous in the uniform norm.  The non-autonomous ODE
\begin{equation}
    \dot X_t^m(x)=v_t^m(X_t^m(x)),
    \qquad X_0^m(x)=x,
    \label{eq:fixed-scale-frozen-ode}
\end{equation}
therefore has a unique global flow.  Define
\begin{equation}
    \Phi(m)_t:=(X_t^m)_\#\mu_0.
\end{equation}
The continuity of the flow shows that $\Phi(m)\in\mathcal X_T$.

We next make the contraction estimate explicit.  Work with periodic lifts of
the vector fields and flows to $\R^d$, choosing the same lift of the initial
point for $X^m$ and $X^{\widetilde m}$.  From
\eqref{eq:fixed-scale-spatial-lipschitz} and
\eqref{eq:fixed-scale-measure-lipschitz},
\begin{equation}
\begin{aligned}
    |X_t^m(x)-X_t^{\widetilde m}(x)|
    \le{}&
    2C_B\int_0^t
        |X_s^m(x)-X_s^{\widetilde m}(x)|\,\dd s \\
    &+C_B\int_0^tW_1(m_s,\widetilde m_s)\,\dd s.
\end{aligned}
\end{equation}
Gronwall's lemma gives, for $0\le t\le T$,
\begin{equation}
    \sup_x d_{\T^d}(X_t^m(x),X_t^{\widetilde m}(x))
    \le
    C_B T e^{2C_BT}d_T(m,\widetilde m).
    \label{eq:fixed-scale-flow-stability}
\end{equation}
Coupling the two pushforwards by the common starting point $x\sim\mu_0$ then
yields
\begin{equation}
    d_T(\Phi(m),\Phi(\widetilde m))
    \le C_B T e^{2C_BT}d_T(m,\widetilde m).
    \label{eq:fixed-scale-contraction}
\end{equation}
Choose $T_*>0$ so that $C_BT_*e^{2C_BT_*}<1$.  Banach's fixed-point theorem
gives a unique $\mu\in\mathcal X_{T_*}$ satisfying
$\mu_t=(X_t^\mu)_\#\mu_0$.  Since $T_*$ depends only on $J$ and $K$, the same
construction can be restarted from $\mu_{T_*}$ on each successive interval of
length $T_*$.  This produces a global narrowly continuous characteristic
solution.

The characteristic solution is distributional.  Indeed, if
$\varphi\in C_c^1([0,\infty)\times\T^d)$, differentiation of
$\varphi(t,X_t(x))$ along the flow and integration against $\mu_0$ give the
weak formulation of~\eqref{eq:normalized-convolution-flow}.

\paragraph{Uniqueness in the distributional class.}
Let $\mu\in C([0,T];\mathcal P(\T^d))$ be any narrowly continuous
distributional solution of~\eqref{eq:normalized-convolution-flow}.  The
associated velocity
\begin{equation}
    v_t(x):=\mathcal B[\nu](x)-\mathcal B[\mu_t](x)
\end{equation}
is bounded, uniformly Lipschitz in $x$, and continuous in $t$ with values in
$C(\T^d;\R^d)$ by the estimates above.  We spell out why the linear continuity equation with this prescribed
velocity has only the characteristic solution.  Fix $t\in[0,T]$ and
$\psi\in C^1(\T^d)$.  Choose smooth periodic vector fields $v^\varepsilon$ on
$[0,t]\times\T^d$ such that $v^\varepsilon\to v$ uniformly and
$\sup_\varepsilon\Lip_x(v^\varepsilon)<\infty$, and let
$X_{s,r}^\varepsilon$ be the corresponding two-parameter flows.  The backward
test function
\begin{equation}
    \varphi^\varepsilon(s,x)
    :=\psi\bigl(X_{s,t}^\varepsilon(x)\bigr)
\end{equation}
is smooth and satisfies
\begin{equation}
    \partial_s\varphi^\varepsilon
    +v_s^\varepsilon\cdot\nabla\varphi^\varepsilon=0,
    \qquad
    \varphi^\varepsilon(t,\cdot)=\psi,
\end{equation}
with $\|\nabla\varphi^\varepsilon\|_{L^\infty}$ bounded uniformly in
$\varepsilon$ by the flow Lipschitz estimate.  Using
$\varphi^\varepsilon$ in the distributional formulation on $[0,t]$ (after an
inessential cutoff at the time endpoints) gives
\begin{equation}
\begin{aligned}
    \int_{\T^d}\psi\,\dd\mu_t
    -\int_{\T^d}\varphi^\varepsilon(0,\cdot)\,\dd\mu_0
    =
    \int_0^t\!\int_{\T^d}
    (v_s-v_s^\varepsilon)\cdot
    \nabla\varphi^\varepsilon\,\dd\mu_s\,\dd s.
\end{aligned}
\end{equation}
The right-hand side tends to zero.  Uniform stability of Lipschitz ODE flows
also gives $X_{0,t}^\varepsilon\to X_{0,t}$ uniformly, where $X_{s,r}$ is the
flow of $v$.  Passing to the limit therefore yields
\begin{equation}
    \int_{\T^d}\psi\,\dd\mu_t
    =\int_{\T^d}\psi(X_{0,t}(x))\,\dd\mu_0(x).
\end{equation}
Since this holds for every $\psi$, every distributional solution necessarily
satisfies
\begin{equation}
    \mu_t=(X_{0,t})_\#\mu_0.
    \label{eq:fixed-scale-representation}
\end{equation}

Now let $\mu$ and $\widetilde\mu$ be two solutions with the same initial law,
and let $X$ and $\widetilde X$ be their flows.  Coupling
$\mu_t$ and $\widetilde\mu_t$ by
$(X_t,\widetilde X_t)_\#\mu_0$ gives
\begin{equation}
    W_1(\mu_t,\widetilde\mu_t)
    \le
    \sup_x d_{\T^d}(X_t(x),\widetilde X_t(x)).
\end{equation}
Repeating the flow comparison above and using this coupling yields
\begin{equation}
    \sup_x d_{\T^d}(X_t(x),\widetilde X_t(x))
    \le
    3C_B\int_0^t
    \sup_x d_{\T^d}(X_s(x),\widetilde X_s(x))\,\dd s.
\end{equation}
Gronwall's lemma implies $X_t=\widetilde X_t$ and therefore
$\mu_t=\widetilde\mu_t$ for every $t$.  This proves global uniqueness in the
stated class.

\paragraph{Density bounds.}
Assume $\mu_0=q_0\,\dd x$ with $q_0\in L^\infty$ and
$\operatorname*{ess\,inf}q_0>0$.  By
\eqref{eq:fixed-scale-velocity-bounds}, the characteristic flow and its inverse
are Lipschitz, with
\begin{equation}
    \Lip(X_t),\ \Lip(X_t^{-1})\le e^{2C_Bt}.
    \label{eq:fixed-scale-bilip}
\end{equation}
Thus $X_t$ is bi-Lipschitz and the pushforward of an absolutely continuous
measure is absolutely continuous.  By Rademacher's theorem and the area
formula, for almost every $x$,
\begin{equation}
    q_t(X_t(x))\,|\det DX_t(x)|=q_0(x).
    \label{eq:fixed-scale-density-formula}
\end{equation}
The bi-Lipschitz estimates imply
\begin{equation}
    e^{-2dC_Bt}
    \le |\det DX_t(x)|
    \le e^{2dC_Bt}
    \qquad\text{for a.e. }x.
\end{equation}
Consequently, for $0\le t\le T$,
\begin{equation}
    e^{-2dC_BT}\operatorname*{ess\,inf}q_0
    \le q_t
    \le e^{2dC_BT}\|q_0\|_{L^\infty}
    \qquad\text{a.e. on }\T^d,
\end{equation}
which is~\eqref{eq:fixed-scale-density-bounds} with $C_T=2dC_BT$.

\paragraph{Translation symmetries.}
Let $T_y(x)=x+y$ on $\T^d$, and suppose
$(T_y)_\#\nu=\nu$ and $(T_y)_\#\mu_0=\mu_0$.  Convolution equivariance gives,
for every probability measure $\eta$,
\begin{equation}
    \mathcal B[(T_y)_\#\eta](x)
    =\mathcal B[\eta](x-y).
    \label{eq:fixed-scale-translation-equivariance}
\end{equation}
Hence $\widetilde\mu_t:=(T_y)_\#\mu_t$ solves
\eqref{eq:normalized-convolution-flow} with the same target and the same initial
law as $\mu_t$.  By uniqueness, $\widetilde\mu_t=\mu_t$ for all $t\ge0$.
\end{proof}

\section{Proof of the lower bounds}
\subsection{Proof of the Gaussian smoothing lower bound Proposition~\ref{prop:nonlinear-log-lower-bound}}
\label{app:proof-gaussian-smoothing}

\begin{proof}
Write
\begin{equation}
    P_r=e^{r\Delta},
    \qquad
    r_\sigma=\frac{\sigma^2}{2},
\end{equation}
and set
\begin{equation}
    m:=m_\beta=\max\{j\in\mathbb N_0:j<\beta\},
    \qquad
    \alpha:=\beta-m\in(0,1].
\end{equation}
For $n\ge1$ and a parameter $\eta>0$ to be fixed below, let
\begin{equation}
    h_n(x)=\eta n^{-\beta}\sin(nx_1),
    \qquad
    \rho_n^\star(x)=\frac{\exp(h_n(x))}{\int_{\T^d} \exp(h_n(y))\,\dd y}.
    \label{eq:app-gaussian-target-pn}
\end{equation}
We first verify that this family is uniformly bounded in the H\"older space appearing in the proposition.  If $\kappa$ is a multi-index with $j:=|\kappa|\le m$, then the derivative vanishes unless $\kappa$ is supported in the first coordinate, and in every case
\begin{equation}
    \|\partial^\kappa h_n\|_{L^\infty(\T^d)}
    \le \eta n^{j-\beta}
    \le \eta.
\end{equation}
For $|\kappa|=m$, as
\begin{equation}
    |\psi(u)-\psi(v)|
    \le C_\alpha |u-v|^\alpha,
    \qquad
    \psi\in\{\sin,\cos\},
\end{equation}
this implies, after choosing compatible lifts of the torus coordinates, that
\begin{equation}
\begin{aligned}
    \sup_{x\ne y}
    \frac{|\partial^\kappa h_n(x)-\partial^\kappa h_n(y)|}
         {d_{\T^d}(x,y)^\alpha}
    &\le C_\alpha\eta n^{m-\beta+\alpha} \\
    &=C_\alpha\eta.
\end{aligned}
\end{equation}
Here the same argument covers integer $\beta$, for which $\alpha=1$ and the top-order seminorm is Lipschitz.  Consequently, there exists $C_{d,\beta}<\infty$ such that
\begin{equation}
    \sup_{n\ge1}\|h_n\|_{\mathcal C^\beta(\T^d)}
    \le C_{d,\beta}\eta.
    \label{eq:app-hn-holder-bound}
\end{equation}

Let $q_t^n$ be the solution of~\eqref{eq:gaussian-drifting-pa} with $a=h_n$, that is
\begin{equation}
    \partial_t q_t^n
    =
    \sigma^2\diver
    \left(
        q_t^n
        \nabla
        \log\frac{P_{r_\sigma} q_t^n}{P_{r_\sigma} \rho_n^\star}
    \right),
    \qquad
    q_0^n\equiv1.
    \label{eq:app-nonlinear-torus-drifting-pn}
\end{equation}
We show that for a suitable frequency $n=n_t\simeq\sqrt{\log t}$, the solution has not created the $n$-th sine moment of the target by time $t$.

\medskip

\noindent\textbf{Step 1: the smoothed velocity is exponentially small for coordinate-periodic perturbations.}
The target $\rho_n^\star$ is invariant under translations in $x_2,\ldots,x_d$ and under the shift $x_1\mapsto x_1+2\pi/n$.  The initial density has the same symmetries.  Since equation~\eqref{eq:app-nonlinear-torus-drifting-pn} is equivariant under torus translations, Proposition~\ref{prop:fixed-scale-wellposedness} implies that $q_t^n$ retains these symmetries for every $t\ge0$.  Consequently, $q_t^n-1$ and $\rho_n^\star-1$ only have Fourier modes in
\begin{equation}
    \{\ell n\mathbf e_1:\ell\in\mathbb Z\setminus\{0\}\},
    \qquad
    \mathbf e_1=(1,0,\ldots,0).
\end{equation}
Let $\rho$ be a probability density on $\T^d$ with these symmetries.  Since
\begin{equation}
    P_{r_\sigma} e_k=e^{-r_\sigma|k|^2}e_k,
    \qquad k\in\mathbb Z^d,
\end{equation}
and $|\widehat{\rho-1}(k)|\le\|\rho-1\|_{L^1}\le2$, we get
\begin{equation}
\begin{aligned}
    \|\nabla P_{r_\sigma}\rho\|_{L^\infty}
    &=
    \|\nabla P_{r_\sigma}(\rho-1)\|_{L^\infty} \\
    &\le
    C\sum_{\ell\ne0}|\ell n|e^{-r_\sigma \ell^2n^2}
    \le
    C e^{-\gamma n^2},
\end{aligned}
\label{eq:app-periodic-smoothing-estimate}
\end{equation}
for constants $C,\gamma>0$ depending only on $\sigma$.

The heat kernel on $\T^d$ is continuous and strictly positive.  Therefore there exists $c_{d,\sigma}^{\rm heat}>0$ such that, for every probability density $\rho$,
\begin{equation}
    P_{r_\sigma}\rho(x)\ge c_{d,\sigma}^{\rm heat},
    \qquad x\in\T^d.
    \label{eq:app-heat-lower-bound}
\end{equation}
Applying~\eqref{eq:app-periodic-smoothing-estimate} and~\eqref{eq:app-heat-lower-bound} to $\rho=q_t^n$ and $\rho=\rho_n^\star$ gives
\begin{equation}
    \left\|
    \nabla
    \log\frac{P_{r_\sigma} q_t^n}{P_{r_\sigma} \rho_n^\star}
    \right\|_{L^\infty}
    \le
    C e^{-\gamma n^2},
    \label{eq:app-velocity-small}
\end{equation}
where now $C$ may also depend on $d$.

\medskip

\noindent\textbf{Step 2: the model sine moment remains small.}
Define
\begin{equation}
    A_n(t):=\int_{\T^d} \sin(nx_1)q_t^n(x)\,\dd x.
\end{equation}
Since $q_0^n\equiv1$, $A_n(0)=0$.  Differentiating along~\eqref{eq:app-nonlinear-torus-drifting-pn} and integrating by parts gives
\begin{equation}
\begin{aligned}
    A_n'(t)
    &=
    \sigma^2\int_{\T^d} \sin(nx_1)
    \diver
    \left(
        q_t^n
        \nabla
        \log\frac{P_{r_\sigma} q_t^n}{P_{r_\sigma} \rho_n^\star}
    \right)\dd x \\
    &=
    -\sigma^2 n\int_{\T^d} \cos(nx_1)q_t^n
    \partial_{x_1}
    \log\frac{P_{r_\sigma} q_t^n}{P_{r_\sigma} \rho_n^\star}\,\dd x .
\end{aligned}
\end{equation}
Using that $q_t^n$ is a probability density and using~\eqref{eq:app-velocity-small},
\begin{equation}
    |A_n'(t)|\le Cn e^{-\gamma n^2}.
\end{equation}
Therefore
\begin{equation}
    |A_n(t)|\le Ctn e^{-\gamma n^2}.
    \label{eq:app-model-sine-moment-small}
\end{equation}

\medskip

\noindent\textbf{Step 3: the target sine moment has size $\eta n^{-\beta}$.}
Put $\varepsilon_n:=\eta n^{-\beta}$.  Since the target depends only on $x_1$ and the reference measure is normalized,
\begin{equation}
    \rho_n^\star(x)=\frac{e^{\varepsilon_n\sin(nx_1)}}{\int_{\T^d} e^{\varepsilon_n\sin(ny_1)}\,\dd y}.
\end{equation}
A Taylor expansion at $\varepsilon_n=0$, together with $\int_{\T^d}\sin^2(nx_1)\,\dd x=1/2$, gives
\begin{equation}
    \int_{\T^d} \sin(nx_1)\rho_n^\star(x)\,\dd x
    =
    \frac{\varepsilon_n}{2}+O(\varepsilon_n^3),
\end{equation}
with an absolute implicit constant, uniformly in $n$.  Hence there is a universal $\eta_0>0$ such that, whenever $0<\eta\le\eta_0$,
\begin{equation}
    \int_{\T^d} \sin(nx_1)\rho_n^\star(x)\,\dd x
    \ge
    c\eta n^{-\beta}
    \qquad\text{for every }n\ge1.
    \label{eq:app-target-sine-moment}
\end{equation}
We now fix
\begin{equation}
    0<\eta\le\min\left\{\eta_0,\frac{C}{C_{d,\beta}}\right\}.
    \label{eq:app-gaussian-eta-choice}
\end{equation}
By~\eqref{eq:app-hn-holder-bound}, this ensures that $h_n\in\mathcal C_C^\beta(\T^d)$ for every $n\ge1$.

\medskip

\noindent\textbf{Step 4: choose the hard frequency.}
For $t\ge2$, let
\begin{equation}
    n_t:=\left\lceil L\sqrt{\log t}\right\rceil,
\end{equation}
where $L\ge1$ will be fixed momentarily.  Since $n_t\ge L\sqrt{\log t}$ and $n_t\lesssim_L\sqrt{\log t}$ for $t\ge2$, one has
\begin{equation}
    \sup_{t\ge2}t n_t^{\beta+1}e^{-\gamma n_t^2}\longrightarrow0
    \qquad\text{as }L\to\infty.
\end{equation}
We may therefore choose $L$, depending only on $d,\beta,\sigma$ and $\eta$, so large that
\begin{equation}
    Ct n_t^{\beta+1}e^{-\gamma n_t^2}\le \frac{c\eta}{2}
    \qquad\text{for every }t\ge2,
    \label{eq:app-hard-frequency-choice}
\end{equation}
where $C$ and $c$ are the constants in~\eqref{eq:app-model-sine-moment-small} and~\eqref{eq:app-target-sine-moment}.  Combining those two estimates yields
\begin{equation}
\begin{aligned}
    \left|
    \int_{\T^d} \sin(n_t x_1)(q_t^{n_t}-\rho_{n_t}^\star)\,\dd x
    \right|
    &\ge
    c\eta n_t^{-\beta}-Ct n_t e^{-\gamma n_t^2} \\
    &\ge
    c\eta n_t^{-\beta}.
\end{aligned}
\label{eq:app-sine-moment-gap}
\end{equation}
Since $n_t\asymp_L\sqrt{\log t}$, this gives
\begin{equation}
    \left|
    \int_{\T^d} \sin(n_t x_1)(q_t^{n_t}-\rho_{n_t}^\star)\,\dd x
    \right|
    \ge
    c\eta(\log t)^{-\beta/2}.
    \label{eq:app-sine-moment-gap-log}
\end{equation}

\medskip

\noindent\textbf{Step 5: convert the moment gap into KL and Wasserstein lower bounds.}
By Pinsker's inequality,
\begin{equation}
    \KL(q_t^{n_t}\mid \rho_{n_t}^\star)
    \ge
    2\|q_t^{n_t}-\rho_{n_t}^\star\|_{\operatorname{TV}}^2.
\end{equation}
Since $|\sin(n_t x_1)|\le1$,
\begin{equation}
    \|q_t^{n_t}-\rho_{n_t}^\star\|_{\operatorname{TV}}
    \ge
    \frac12
    \left|
    \int_{\T^d} \sin(n_t x_1)(q_t^{n_t}-\rho_{n_t}^\star)\,\dd x
    \right|.
\end{equation}
Together with~\eqref{eq:app-sine-moment-gap-log}, this yields
\begin{equation}
    \KL(q_t^{n_t}\mid \rho_{n_t}^\star)
    \ge
    c\eta^2(\log t)^{-\beta}.
\end{equation}
Taking the supremum over $a\in\mathcal C_C^\beta(\T^d)$ and absorbing the fixed factor $\eta^2$ into the constant proves~\eqref{eq:gaussian-Hbeta-KL-lower}.

For the Wasserstein bound, use Kantorovich duality with
\begin{equation}
    \varphi_{n_t}(x):=\frac{\sin(n_t x_1)}{n_t}.
\end{equation}
Because $\|\nabla\varphi_{n_t}\|_{L^\infty}\le1$, this function is $1$-Lipschitz on $\T^d$.  Hence
\begin{equation}
\begin{aligned}
    W_1(q_t^{n_t},\rho_{n_t}^\star)
    &\ge
    \left|
    \int_{\T^d} \frac{\sin(n_t x_1)}{n_t}(q_t^{n_t}-\rho_{n_t}^\star)\,\dd x
    \right| \\
    &\ge
    c\eta n_t^{-(\beta+1)} \\
    &\ge
    c\eta(\log t)^{-(\beta+1)/2}.
\end{aligned}
\end{equation}
Since $W_2\ge W_1$, taking the supremum and absorbing $\eta$ into the constant proves~\eqref{eq:gaussian-Hbeta-W2-lower}.
\end{proof}

\subsection{Proof of the Laplace smoothing lower bound Proposition~\ref{prop:true-laplace-polynomial-lower-bound}}
\label{app:proof-laplace-smoothing}

\begin{proof}
Fix $0<\eta\leq1/4$, to be chosen small enough below.  For $n\geq1$, set
\begin{equation}
    \varepsilon_n:=\eta n^{-\beta},
    \qquad
    a^{(n)}(x):=\log\bigl(1+\varepsilon_n\sin(nx_1)\bigr),
    \qquad
    \rho_n^\star:=\rho_{a^{(n)}}^\star.
    \label{eq:true-laplace-bad-target}
\end{equation}
Since $\int_{\T^d}(1+\varepsilon_n\sin(nx_1))\,\dd x=1$, we have
\begin{equation}
    \rho_n^\star(x)=1+\varepsilon_n\sin(nx_1).
\end{equation}
Thus $\rho_n^\star$ is a positive probability density on $\T^d$, and $a^{(n)}$ is smooth.  Let $q_t^n:=q_t^{a^{(n)}}$ be the unique global characteristic solution of~\eqref{eq:true-laplace-drifting-pde} with target $\rho_n^\star$, whose existence and positivity follow from Proposition~\ref{prop:fixed-scale-wellposedness}.

Before analyzing the flow, we verify that these targets lie in the required
smoothness class. Since $\eta\leq1/4$, the denominator in
$a^{(n)}=\log(1+\eta n^{-\beta}\sin(nx_1))$ is at least $3/4$. Repeated
differentiation therefore gives
\[
    \|\partial_1^k a^{(n)}\|_\infty
       \le C_k\eta n^{k-\beta},\qquad 0\le k\le m_\beta+1.
\]
All derivatives involving another coordinate vanish. Combining the bounds at
orders $m_\beta$ and $m_\beta+1$ gives
\[
\begin{aligned}
    |\partial_1^{m_\beta}a^{(n)}(x)
       -\partial_1^{m_\beta}a^{(n)}(y)|
       &\le C_\beta\eta n^{m_\beta-\beta}
                       \min\{1,n d_{\T}(x,y)\}\\
       &\le C_\beta\eta d_{\T}(x,y)^{\beta-m_\beta}.
\end{aligned}
\]
Here $\min\{1,u\}\le u^{\beta-m_\beta}$; when $\beta$ is an integer, the
last seminorm is Lipschitz. Hence
\[
    \sup_{n\ge1}\|a^{(n)}\|_{\mathcal C^\beta(\T^d)}\le C_\beta\eta.
\]
Thus, imposing $\eta\le C/C_\beta$ in addition to $\eta\le1/4$ gives
$a^{(n)}\in\mathcal C_C^\beta(\T^d)$ for every $n$. The parameter $\eta$
may be decreased further below to close the bootstrap.

\medskip
\noindent\textbf{Step 1: high-frequency attenuation.}
For $n\geq1$, define
\begin{equation}
    \omega_\tau(n)
    :=
    \sum_{\ell\neq0}
    \mu_\tau(\ell n e_1),
    \qquad
    \mu_\tau(m)
    :=
    -m\cdot\nabla_m\widehat L_\tau(m).
\end{equation}
By~\eqref{eq:laplace-mode-rate},
\begin{equation}
    \mu_\tau(m)
    =
    (d+1)\tau^2|m|^2
    \left(1+\tau^2|m|^2\right)^{-(d+3)/2},
\end{equation}
and therefore
\begin{equation}
    \omega_\tau(n)
    \leq
    C_d\tau^{-(d+1)}n^{-(d+1)}.
    \label{eq:omega-laplace-bound}
\end{equation}

Let $\rho=1+r$ be a density depending only on $x_1$, with period $2\pi/n$ in the $x_1$ variable, and assume $\|r\|_{L^\infty}\leq1/2$.  Since $K_\tau*1=0$, the numerator of $M_\tau[\rho]$ is $K_\tau*r$.  The only nonzero Fourier modes of $r$ are $\ell n e_1$, $\ell\neq0$.  Hence
\begin{equation}
    \|K_\tau*r\|_{L^\infty}
    \leq
    C n^{-1}\omega_\tau(n)\|r\|_{L^\infty},
    \label{eq:laplace-numerator-small}
\end{equation}
and
\begin{equation}
    \|\diver(K_\tau*r)\|_{L^\infty}
    \leq
    C\omega_\tau(n)\|r\|_{L^\infty}.
    \label{eq:laplace-div-numerator-small}
\end{equation}
Moreover,
\begin{equation}
    L_\tau*\rho=1+L_\tau*r,
    \qquad
    \|L_\tau*r\|_{L^\infty}\leq\|r\|_{L^\infty}\leq1/2,
\end{equation}
so the normalization in $M_\tau[\rho]$ is uniformly nondegenerate.  Differentiating the quotient and using that $\nabla L_\tau$ is integrable gives
\begin{equation}
    \|M_\tau[\rho]\|_{L^\infty}
    \leq
    C n^{-1}\omega_\tau(n)\|r\|_{L^\infty},
    \qquad
    \|\diver M_\tau[\rho]\|_{L^\infty}
    \leq
    C\omega_\tau(n)\|r\|_{L^\infty}.
    \label{eq:true-laplace-invisibility}
\end{equation}
The constant depends only on $d$ and $\tau$.

\medskip
\noindent\textbf{Step 2: the true nonlinear flow stays close to the flat density on the invisible time scale.}
The target $\rho_n^\star$ and the initialization $q_0^n\equiv1$ depend only on $x_1$ and are $2\pi/n$-periodic.  By uniqueness and symmetry, the solution $q_t^n$ has the same property.  Set
\begin{equation}
    R_n(t):=\|q_t^n-1\|_{L^\infty}.
\end{equation}
As long as $R_n(t)\leq1/2$, estimate~\eqref{eq:true-laplace-invisibility} gives
\begin{equation}
    \left\|
        \diver\left(M_\tau[\rho_n^\star]-M_\tau[q_t^n]\right)
    \right\|_{L^\infty}
    \leq
    C\omega_\tau(n)\left(\varepsilon_n+R_n(t)\right).
    \label{eq:velocity-div-bound}
\end{equation}
Along characteristics of the transport equation,
\begin{equation}
    \frac{\dd}{\dd t}\log q_t^n(X_t)
    =
    -
    \diver\left(M_\tau[\rho_n^\star]-M_\tau[q_t^n]\right)(X_t).
\end{equation}
Therefore, by~\eqref{eq:velocity-div-bound} and Gronwall's inequality, there is a numerical $\delta>0$ such that
\begin{equation}
    \omega_\tau(n)t\leq\delta
    \qquad\Longrightarrow\qquad
    R_n(t)\leq C \varepsilon_n.
    \label{eq:q-close-to-flat}
\end{equation}
After decreasing $\eta$, this closes the bootstrap $R_n(t)\leq1/2$.

\medskip
\noindent\textbf{Step 3: the $n$-th sine moment cannot be created quickly.}
Define the model sine moment
\begin{equation}
    A_n(t)
    :=
    \int_{\T^d}\sin(nx_1)q_t^n(x)\,\dd x.
\end{equation}
Since $q_0^n\equiv1$, $A_n(0)=0$.  Differentiating along~\eqref{eq:true-laplace-drifting-pde} gives
\begin{equation}
\begin{aligned}
    A_n'(t)
    &=
    \int_{\T^d}
    \nabla\sin(nx_1)\cdot
    q_t^n
    \left(
        M_\tau[\rho_n^\star]-M_\tau[q_t^n]
    \right)\dd x .
\end{aligned}
\end{equation}
Using~\eqref{eq:true-laplace-invisibility},~\eqref{eq:q-close-to-flat}, and $\int q_t^n=1$, we get, whenever $\omega_\tau(n)t\leq\delta$,
\begin{equation}
    |A_n'(t)|
    \leq
    C\omega_\tau(n)\varepsilon_n.
\end{equation}
Thus, after decreasing $\delta$ if necessary,
\begin{equation}
    |A_n(t)|
    \leq
    \frac{\varepsilon_n}{8}
    \qquad
    \text{whenever }\omega_\tau(n)t\leq\delta .
    \label{eq:model-sine-moment-small-laplace}
\end{equation}

On the other hand, the target moment is exact:
\begin{equation}
    \int_{\T^d}\sin(nx_1)\rho_n^\star(x)\,\dd x
    =
    \varepsilon_n\int_{\T^d}\sin^2(nx_1)\,\dd x
    =
    \frac{\varepsilon_n}{2}.
    \label{eq:target-sine-moment-laplace}
\end{equation}
Combining~\eqref{eq:model-sine-moment-small-laplace} and~\eqref{eq:target-sine-moment-laplace}, we obtain
\begin{equation}
    \left|
    \int_{\T^d}\sin(nx_1)
    \left(q_t^n-\rho_n^\star\right)(x)\,\dd x
    \right|
    \geq
    \frac{3\varepsilon_n}{8}
    \label{eq:true-laplace-moment-gap}
\end{equation}
whenever $\omega_\tau(n)t\leq\delta$.

\medskip
\noindent\textbf{Step 4: choose the invisible frequency.}
Choose $B\geq1$ large enough, depending only on $d,\tau$ and the constants above, and set
\begin{equation}
    n_t
    :=
    \left\lceil
    B
    \left(
        1+\frac{t}{\tau^{d+1}}
    \right)^{1/(d+1)}
    \right\rceil .
\end{equation}
By~\eqref{eq:omega-laplace-bound}, this choice ensures
\begin{equation}
    \omega_\tau(n_t)t\leq\delta.
\end{equation}
Therefore~\eqref{eq:true-laplace-moment-gap} gives
\begin{equation}
    \left|
    \int_{\T^d}\sin(n_tx_1)
    \left(q_t^{n_t}-\rho_{n_t}^\star\right)(x)\,\dd x
    \right|
    \geq
    c\eta
    \left(
        1+\frac{t}{\tau^{d+1}}
    \right)^{-\beta/(d+1)}.
    \label{eq:true-laplace-final-moment-gap}
\end{equation}

\medskip
\noindent\textbf{Step 5: convert the moment gap into KL and Wasserstein lower bounds.}
Since $|\sin(n_tx_1)|\leq1$,~\eqref{eq:true-laplace-final-moment-gap} implies
\begin{equation}
    \|q_t^{n_t}-\rho_{n_t}^\star\|_{L^1}
    \geq
    c\eta
    \left(
        1+\frac{t}{\tau^{d+1}}
    \right)^{-\beta/(d+1)}.
\end{equation}
Pinsker's inequality gives
\begin{equation}
    \KL(q_t^{n_t}\mid \rho_{n_t}^\star)
    \geq
    c\eta^2
    \left(
        1+\frac{t}{\tau^{d+1}}
    \right)^{-2\beta/(d+1)}.
\end{equation}

For the Wasserstein bound, use the $1$-Lipschitz test function
\begin{equation}
    \varphi_t(x):=\frac{\sin(n_tx_1)}{n_t}.
\end{equation}
Kantorovich duality gives
\begin{equation}
\begin{aligned}
    W_2(q_t^{n_t},\rho_{n_t}^\star)
    &\geq
    W_1(q_t^{n_t},\rho_{n_t}^\star) \\
    &\geq
    \left|
    \int_{\T^d}
    \frac{\sin(n_tx_1)}{n_t}
    \left(q_t^{n_t}-\rho_{n_t}^\star\right)(x)\,\dd x
    \right| \\
    &\geq
    c\eta n_t^{-(\beta+1)} \\
    &\geq
    c\eta
    \left(
        1+\frac{t}{\tau^{d+1}}
    \right)^{-(\beta+1)/(d+1)}.
\end{aligned}
\end{equation}

Taking the supremum over $a\in\mathcal C_C^\beta(\T^d)$ and absorbing the fixed powers of $\eta$ into a constant $c_C>0$ proves~\eqref{eq:true-laplace-KL-lower} and~\eqref{eq:true-laplace-W2-lower}.
\end{proof}

\subsection{Non-integrability of the Gaussian plug-in drift as a Wasserstein gradient flow}
\label{app:no-hidden-energy}

This appendix establishes an obstruction to representing the Gaussian plug-in drift as the Wasserstein gradient flow of a scalar energy functional. Although the drift is a spatial gradient field, its potential is not a variational derivative with respect to the density.

Work on the one-dimensional torus $\T=\R/(2\pi\mathbb Z)$ with normalized Lebesgue measure, and let $P_s=e^{s\partial_{xx}}$ with $s>0$.  Fix a smooth positive target density $\rho^\star$.  Suppose, for contradiction, that there is a $C^2$ functional $\mathcal E$ on smooth positive probability densities such that the Gaussian drifting equation can be written as the Wasserstein gradient flow of $\mathcal E$, namely
\begin{equation}
    \nabla\frac{\delta \mathcal E}{\delta q}(q)
    =
    \sigma^2\nabla\log\frac{P_s q}{P_s\rho^\star}
\end{equation}
for every smooth positive density $q$.  Since first variations are defined only up to additive constants, this means that, when paired against zero-mean perturbations,
\begin{equation}
    \frac{\delta \mathcal E}{\delta q}(q)
    =
    \sigma^2\log\frac{P_s q}{P_s\rho^\star}.
\end{equation}
Therefore, for every pair of smooth zero-mean perturbations $h,k$, the second variation would have to satisfy
\begin{equation}
    D^2\mathcal E(q)[h,k]
    =
    \sigma^2\int_\T k(x)\frac{P_s h(x)}{P_s q(x)}\,\dd x.
    \label{eq:appendix-hessian-form}
\end{equation}
The Hessian of a $C^2$ functional must be symmetric in $h$ and $k$.  We now show that the bilinear form in~\eqref{eq:appendix-hessian-form} is not symmetric in general.

Let
\begin{equation}
    q_\varepsilon(x)=1+\varepsilon\cos x,
    \qquad |\varepsilon|\ll 1,
\end{equation}
so that
\begin{equation}
    P_s q_\varepsilon(x)=1+a\cos x,
    \qquad a:=\varepsilon e^{-s}.
\end{equation}
Choose the zero-mean perturbations
\begin{equation}
    h(x)=\cos x,
    \qquad
    k(x)=\cos 2x.
\end{equation}
Then
\begin{equation}
    P_s h=e^{-s}\cos x,
    \qquad
    P_s k=e^{-4s}\cos 2x.
\end{equation}
Using
\begin{equation}
    \frac{1}{1+a\cos x}=1-a\cos x+O(a^2),
\end{equation}
we obtain
\begin{align}
    \int_\T k\frac{P_s h}{P_s q_\varepsilon}\,\dd x
    &=
    e^{-s}\int_\T \frac{\cos 2x\cos x}{1+a\cos x}\,\dd x \\
    &=
    -\frac{a e^{-s}}{4}+O(a^2),
\end{align}
and similarly
\begin{align}
    \int_\T h\frac{P_s k}{P_s q_\varepsilon}\,\dd x
    &=
    e^{-4s}\int_\T \frac{\cos x\cos 2x}{1+a\cos x}\,\dd x \\
    &=
    -\frac{a e^{-4s}}{4}+O(a^2).
\end{align}
Since $s>0$, one has $e^{-s}\neq e^{-4s}$.  Thus, for $\varepsilon$ small and nonzero,
\begin{equation}
    \int_\T k\frac{P_s h}{P_s q_\varepsilon}\,\dd x
    \neq
    \int_\T h\frac{P_s k}{P_s q_\varepsilon}\,\dd x.
\end{equation}
This contradicts the symmetry of the Hessian of $\mathcal E$.  Consequently, at every positive scale, the Gaussian plug-in potential
\[
    \log\frac{P_s q}{P_s\rho^\star}
\]
is not the first variation of any $C^2$ energy functional on densities.  The corresponding drifting equation is therefore not an exact Wasserstein gradient flow in this sense.

\section{Proofs of the local multihead exponential convergence results}
\label{app:local-multiscale-proofs}

This appendix proves Theorems~\ref{thm:local-torus-heat} and~\ref{thm:local-ou}.  The two proofs are written in a common notation.  The model-specific input is isolated in the semigroup and covariance estimates; after that point the argument is identical.

\subsection{Proof strategy}
\label{subsec:loc-proof-strategy}

The proof proceeds in two stages: an a priori convergence estimate for smooth
solutions, followed by the well-posedness and approximation argument at
$C^{2,1}$ regularity. The main components are as follows.

\begin{enumerate}[label=\textup{(\arabic*)}]
\item \textbf{Score formulation.} Set
\[
    h_t=\log(\dd\rho_t/\dd\rho^\star),
    \qquad \xi_t=\nabla h_t,
    \qquad y(t)=\|\xi_t\|_{C^{0,1}}.
\]
The $C^{0,1}$ norm controls both the magnitude and the spatial variation of the
score error. A global $L^\infty$ bound on $h_t$ is not imposed; in the OU model,
a bounded score is compatible with linear growth of $h_t$ at infinity.

\item \textbf{Linear damping from scale averaging.} At smoothing scale $s$, let
\[
    z_s=\nabla\log\frac{\mathsf P_s(e^{a+h})}{\mathsf P_s(e^a)},
    \qquad
    W_S[\nabla h]=\frac1S\int_0^S z_s\,\dd s.
\]
The linearization of $W_S$ at the reference law is
$S^{-1}\int_0^S\mathsf K_s\nabla h\,\dd s$. If $\mathcal G$ denotes the
generator acting on gradients, then
\[
    \mathcal G\left(\frac1S\int_0^S\mathsf K_s\xi\,\dd s\right)
       =\frac{\mathsf K_S-I}{S}\xi.
\]
The strict contraction of $\mathsf K_S$ on the relevant score fields yields a
rate $\lambda_S>0$. The nonlinear argument uses the estimates
\[
    \|W_S[\xi]\|_{C^{1,1}}\le C_S\|\xi\|_{C^{0,1}}
\]
and
\[
    \left\|\mathcal G W_S[\xi]
          -\frac{\mathsf K_S-I}{S}\xi\right\|_{C^{0,1}}
    \le C_S\bigl(\|a\|_{C^{2,1}}+\|\xi\|_{C^{0,1}}\bigr)
               \|\xi\|_{C^{0,1}}.
\]
The latter estimate follows from a cancellation in the scale variable and
controls the nonlinear remainder without requiring two additional spatial
derivatives.

\item \textbf{Closed score estimate.} Combining the preceding bounds with the
equation for $\xi_t$ gives
\[
    D^+y(t)+\lambda_S y(t)
       \le C_S\bigl(\|a\|_{C^{2,1}}+y(t)\bigr)y(t).
\]
If $\|a\|_{C^{2,1}}$ and $y(0)$ are sufficiently small, the nonlinear term is
absorbed by the linear damping. A first-exit argument then yields
$y(t)\le e^{-\lambda_St/2}y(0)$ on the smooth existence interval. A logarithmic
Sobolev inequality gives the corresponding relative-entropy estimate.

\item \textbf{Global well-posedness and approximation.} The smooth local theory,
a continuation criterion based on the $C^{0,1}$ score bound, and an
$L^\infty$ stability estimate give global existence and uniqueness. The
$C^{2,1}$ result is then obtained by mollifying the target and passing to the
limit. Higher derivatives enter only the smooth construction; the smallness
thresholds and decay rate depend only on $d$ and $S$.
\end{enumerate}

Section~\ref{subsec:loc-convergence-mechanism} establishes the a priori
convergence estimate. Section~\ref{subsec:loc-wellposedness} develops the
well-posedness and stability estimates, and Section~\ref{subsec:loc-finish}
passes to $C^{2,1}$ targets.

\subsection{Framework and notion of solution}
\label{subsec:loc-common-notation}

Throughout this appendix, $t$ is the evolution time and $s\in[0,S]$ is the
smoothing scale. The dimension $d$ and maximal scale $S>0$ are fixed.
Constants $C_S$ may depend on $d,S$ and change from line to line. Whenever a
constant also depends on a bound for a target or a score, that dependence is
stated. We use the $C^{k,1}$ norms from Section~\ref{subsec:notation}; vector
and tensor norms are taken componentwise.

\paragraph{Scalar and vector semigroups.}
The scalar semigroup $\mathsf P_s$ acts on densities, while the vector
semigroup $\mathsf K_s$ describes the corresponding action on gradients. The state space $\mathsf X$,
reference probability measure $\mathsf m$, and scalar generator $L$ are
\begin{equation}
\begin{array}{c|c|c|c|c|c}
\text{model}&\mathsf X&\mathsf m&\mathsf P_s&L&\mathsf K_s\\
\hline
\text{torus}&\T^d&\dd x&P_s=e^{s\Delta}&\Delta&P_s\\
\text{OU}&\R^d&\gamma&Q_s=e^{sL_\gamma}&L_\gamma=\Delta-x\cdot\nabla&e^{-s}Q_s
\end{array}
\label{eq:loc-model-table}
\end{equation}
Here $\dd x$ is normalized Lebesgue measure on the torus, and $\gamma$ is
standard Gaussian measure. We use a different letter for the vector
generator:
\begin{equation}
    \mathcal G=\Delta\quad\text{on }\T^d,
    \qquad
    \mathcal G=L_\gamma-I\quad\text{on }\R^d.
    \label{eq:loc-vector-generator}
\end{equation}
The identities needed below are
\begin{equation}
    \nabla\mathsf P_s f=\mathsf K_s\nabla f,
    \qquad \nabla Lf=\mathcal G\nabla f,
    \qquad \partial_s\mathsf K_s=\mathcal G\mathsf K_s.
    \label{eq:loc-score-commutation}
\end{equation}
The factor $e^{-s}$ in the OU vector semigroup is the contraction factor in
Mehler's formula and is part of the gradient commutation identity
\eqref{eq:loc-score-commutation}. All scale integrals and generator identities below are understood
through the kernels, or in distributions.

\paragraph{Target, score, and velocity.}
The target potential is normalized by
\begin{equation}
    \int e^a\,\dd\mathsf m=1,
    \qquad \rho^\star=e^a\mathsf m.
    \label{eq:loc-target-potential}
\end{equation}
For a current law, write
\[
    h=\log\frac{\dd\rho}{\dd\rho^\star},
    \qquad \dd\rho=e^{h+a}\,\dd\mathsf m,
    \qquad \xi=\nabla h.
\]
More generally, throughout this appendix we consider score fields of the form
$\xi=\nabla g$, where $g$ is globally defined and is periodic in the torus case.
For any such primitive $g$, define
\begin{equation}
\begin{aligned}
    \psi_s[g]&:=\log\frac{\mathsf P_s(e^{g+a})}{\mathsf P_s(e^a)},\\
    W_S[\xi]&:=\frac1S\int_0^S\nabla\psi_s[g]\,\dd s,
    \qquad \xi=\nabla g.
\end{aligned}
\label{eq:loc-psi-phi-w}
\end{equation}
Adding a constant to $g$ adds the same constant to $\psi_s[g]$ and leaves its
gradient unchanged. Hence $W_S$ is well defined as a function of the score.
With this convention, the transport velocity is $-W_S[\xi]$.
When $g$ is Lipschitz on $\R^d$, it has at most linear growth; the Gaussian
kernel then makes the numerator finite and positive for every $s>0$.

\paragraph{Evolution equations for the density and log-ratio.}
The evolution equation is
\begin{equation}
    \partial_t\rho_t=\nabla\cdot\bigl(\rho_tW_S[\nabla h_t]\bigr),
    \qquad h_t=\log\frac{\dd\rho_t}{\dd\rho^\star}.
    \label{eq:loc-rho-flow}
\end{equation}
For smooth positive densities, division by the density gives
\begin{equation}
    \partial_t h_t
       =L\left(\frac1S\int_0^S\psi_s[h_t]\,\dd s\right)
          +(\nabla a+\nabla h_t)\cdot W_S[\nabla h_t].
    \label{eq:loc-h-flow}
\end{equation}
Indeed, in the OU case the logarithm of the Lebesgue density differs from
$h_t+a$ by $-|x|^2/2$ and a constant, producing the term
$-x\cdot W_S[\nabla h_t]$ inside $L=L_\gamma$.
On the torus the same computation uses $L=\Delta$.

\paragraph{Solution concept at $C^{2,1}$ regularity.}
\label{subsec:loc-endpoint-solutions}

The notion of solution used in
Theorems~\ref{thm:local-torus-heat} and~\ref{thm:local-ou} is as follows.
Let $a\in C^{2,1}$ satisfy~\eqref{eq:loc-target-potential}, and let
$h_{\rm in}$ satisfy
\[
    \nabla h_{\rm in}\in C^{0,1},
    \qquad
    \int e^{h_{\rm in}}\,\dd\rho^\star=1.
\]
On a finite interval $[0,T]$, a solution of~\eqref{eq:loc-rho-flow} with
initial datum $h_{\rm in}$ is a narrowly continuous curve of probability
measures $(\rho_t)_{0\le t\le T}$ (that is, continuous when tested against
bounded continuous functions) of the form
\[
    \dd\rho_t=e^{h_t}\,\dd\rho^\star,
\]
such that
\[
    h,\nabla h\in C_{\rm loc}([0,T]\times\mathsf X),
    \qquad
    \sup_{0\le t\le T}\|\nabla h_t\|_{C^{0,1}}<\infty,
\]
where $\mathsf X=\T^d$ in the heat case and $\mathsf X=\R^d$ in the OU case,
and such that $h_0=h_{\rm in}$.  The equation is understood in the distributional sense (equivalently, after
testing against smooth functions):
\begin{equation}
    \int\varphi\,\dd\rho_t
    -\int\varphi e^{h_{\rm in}}\,\dd\rho^\star
    =-\int_0^t\int\nabla\varphi\cdot W_S[\nabla h_\tau]
      \,\dd\rho_\tau\,\dd\tau
    \label{eq:loc-endpoint-weak-form}
\end{equation}
for every smooth compactly supported test function $\varphi$; on the torus,
compact support is unnecessary for $\varphi$. A solution is global if these
properties hold on every finite interval.

In the OU case, the probability normalization
$\int e^{h_t}\,\dd\rho^\star=1$ and the bounded score imply that $h_t$ has at
most linear growth, so all terms in~\eqref{eq:loc-endpoint-weak-form} are well
defined.

For both convergence theorems, $h_{\rm in}=-a$.  In that case
$e^{h_{\rm in}}\rho^\star=\mathsf m$, giving the reference initial laws
$\dd x$ on the torus and $\gamma$ in the OU model.

The proof first treats smooth solutions and later passes to this $C^{2,1}$
solution concept. In particular, the uniqueness assertion in the two
theorems is understood in this solution class.

\subsection{A priori convergence estimate for smooth solutions}
\label{subsec:loc-convergence-mechanism}

This section establishes an a priori estimate for smooth solutions. Spatial
derivatives of the score are assumed bounded on each finite time interval in
order to justify the differentiations below. The constants in the $C^{0,1}$
estimates are independent of these higher-derivative bounds. Existence and
continuation are established in Section~\ref{subsec:loc-wellposedness}. The
smallness thresholds $\varepsilon_0,r_0>0$ are chosen below and depend only on
$d$ and $S$.

\subsubsection{Linear smoothing and strict contraction}\label{subsec:loc-semigroups}

The following lemma provides the integrable one-derivative smoothing estimate
used for the averaged velocity and the strict fixed-scale contraction used in
the damping estimate.

\begin{lemma}[Semigroup smoothing and contraction estimates]
\label{lem:loc-semigroup-estimates}
Fix $S>0$.  The following estimates hold.
\begin{enumerate}[label=\textup{(\roman*)}]
\item For every integer $q\ge0$ there is $C_{q,S}<\infty$ such that, for $0<s\le S$,
\begin{equation}
    \|\mathsf K_s\xi\|_{C^{q+1,1}}
    \le C_{q,S}(1\wedge s)^{-1/2}\|\xi\|_{C^{q,1}},
    \label{eq:loc-semigroup-smoothing}
\end{equation}
and therefore, for $0\le s_0<s_1\le S$,
\begin{equation}
    \int_{s_0}^{s_1}\|\mathsf K_{s-s_0}\xi\|_{C^{q+1,1}}\,\dd s
    \le C_{q,S}\sqrt{s_1-s_0}\,\|\xi\|_{C^{q,1}}.
    \label{eq:loc-integrated-smoothing}
\end{equation}
\item For every integer $m\ge1$ there is $C_{m,S}<\infty$ such that
\begin{equation}
    \sup_{0\le s\le S}\|\mathsf K_s\xi\|_{C^{m-1,1}}
    \le C_{m,S}\|\xi\|_{C^{m-1,1}}.
    \label{eq:loc-semigroup-high-bound}
\end{equation}
In the torus case one may take $C_{m,S}=1$.
\item The semigroup is non-expansive on $C^{0,1}$:
\begin{equation}
    \|\mathsf K_s\xi\|_{C^{0,1}}\le \|\xi\|_{C^{0,1}},
    \qquad 0\le s\le S.
    \label{eq:loc-nonstrict-contraction}
\end{equation}
\item There is $\theta_S\in(0,1)$ such that
\begin{equation}
    \|\mathsf K_S\xi\|_{C^{0,1}}
    \le \theta_S\|\xi\|_{C^{0,1}}
    \label{eq:loc-strict-contraction}
\end{equation}
for every score field $\xi$ in the torus case and for every vector field $\xi$ in the OU case.  More explicitly, in the OU case one may take $\theta_S=e^{-S}$.  In the torus case one may take $\theta_S=1-m_S$, where $m_S:=\inf_{x\in\T^d}p_S(x)>0$ and $p_S$ is the heat kernel at time $S$.
\end{enumerate}
\end{lemma}

\begin{proof}
In the torus case $\mathsf K_s=P_s$.  The Markov property and the commutation $\partial^\kappa P_s=P_s\partial^\kappa$ give boundedness on every $C^{m-1,1}$, and in particular non-expansiveness on $C^{0,1}$.  For smoothing, periodize the Euclidean heat kernel. Its first derivative
has $L^1$ norm at most $C_ds^{-1/2}$ for $0<s\le1$. To gain one derivative,
assign one derivative to the kernel and the remaining derivatives to the input:
\[
    \|D^{q+2}P_s\xi\|_\infty
       \le C_d\|\nabla p_s\|_{L^1}\|D^{q+1}\xi\|_\infty.
\]
The highest input derivative exists as a bounded weak derivative when
$\xi\in C^{q,1}$; the Lipschitz seminorm of the output derivative of order
$q+1$ is controlled by its next derivative. Lower derivatives are controlled
by the Markov property. For $s\ge1$ use the semigroup property and the bound
at scale $1$. Finally,
$\int_0^\delta(1\wedge u)^{-1/2}\,\dd u\le C_S\sqrt\delta$ for
$0<\delta\le S$, proving~\eqref{eq:loc-integrated-smoothing}.

It remains in the torus case to prove the strict contraction on score fields.  If $f$ has zero Lebesgue mean, then
\[
    P_Sf(x)=\int_{\T^d}(p_S(x-y)-m_S)f(y)\,\dd y.
\]
Since $p_S-m_S\ge0$ and $\int(p_S-m_S)=1-m_S$, one has
\[
    \|P_Sf\|_{L^\infty}\le (1-m_S)\|f\|_{L^\infty}.
\]
Apply this componentwise to a periodic score field $\xi=\nabla g$: each component has zero mean.  To get the Lipschitz seminorm without imposing more regularity, fix a nonzero torus displacement $y$ and set
\[
    \delta_y\xi_i(x):=\frac{\xi_i(x+y)-\xi_i(x)}{d(y,0)}.
\]
By periodicity $\delta_y\xi_i$ has zero mean, and translation invariance gives
\[
    \delta_y P_S\xi_i=P_S\delta_y\xi_i .
\]
Hence
\[
    \|\delta_yP_S\xi_i\|_{L^\infty}
    \le (1-m_S)\|\delta_y\xi_i\|_{L^\infty}
    \le (1-m_S)\Lip(\xi_i).
\]
Taking the supremum over $y\ne0$ gives the same contraction for the Lipschitz seminorm.  Together with the $L^\infty$ estimate this gives~\eqref{eq:loc-strict-contraction} in $C^{0,1}$.

In the OU case $\mathsf K_s=e^{-s}Q_s$.  Mehler's formula gives
\[
    Q_s f(x)=\mathbb E\bigl[f(e^{-s}x+\sqrt{1-e^{-2s}}Z)\bigr],
    \qquad Z\sim\gamma.
\]
Thus $\partial^\kappa Q_s f=e^{-s|\kappa|}Q_s(\partial^\kappa f)$.  Since $Q_s$ is Markov and
\[
    \Lip(Q_sf)\le e^{-s}\Lip(f),
\]
this gives~\eqref{eq:loc-semigroup-high-bound} for every $m\ge1$.  In particular,
\[
    \|e^{-s}Q_s\xi\|_{L^\infty}\le e^{-s}\|\xi\|_{L^\infty},
    \qquad
    \Lip(e^{-s}Q_s\xi)\le e^{-2s}\Lip(\xi),
\]
which implies both non-expansiveness and the strict bound with $\theta_S=e^{-S}$.  For~\eqref{eq:loc-semigroup-smoothing}, the first derivatives are controlled by the same commutation.  For the extra derivative, set $\sigma_s=\sqrt{1-e^{-2s}}$.  Gaussian integration by parts in Mehler's formula gives, for bounded $F$,
\[
    |\partial_k Q_sF(x)|
    \le e^{-s}\sigma_s^{-1}\mathbb E|Z_k|\,\|F\|_{L^\infty}
    \le C_d(1\wedge s)^{-1/2}\|F\|_{L^\infty}.
\]
Apply this to each weak derivative of $\xi$ of order $q+1$, using
$D^{q+1}Q_s\xi=e^{-(q+1)s}Q_sD^{q+1}\xi$. The resulting bound on derivatives
of order $q+2$ controls the top Lipschitz seminorm of the $C^{q+1,1}$ norm;
all lower derivatives follow from the same commutation.  This proves the OU smoothing estimate and its integrated form.
\end{proof}

\subsubsection{Tilted expectations and the smoothed target}
\label{subsec:loc-tilted}

Differentiating the logarithm of a smoothed density yields an expectation
with respect to a tilted smoothing kernel. Define
\begin{equation}
    \E_{s,H}[F](x):=\frac{\mathsf P_s(e^HF)(x)}{\mathsf P_s(e^H)(x)}.
    \label{eq:loc-M-def}
\end{equation}
For scalar $F,G$, set
$\Cov_{s,H,x}(F,G)=\E_{s,H}[FG](x)-\E_{s,H}[F](x)\E_{s,H}[G](x)$;
for vectors use tensor products, as in Section~\ref{subsec:notation}.
The dependence on $x$ is suppressed when unambiguous. The following estimate
bounds the covariance by the Lipschitz seminorms of its arguments and is
therefore insensitive to their additive levels.

\begin{lemma}[Tilted covariance estimate]
\label{lem:loc-covariance}
Let $R<\infty$ and $S>0$. If $H$ is Lipschitz with
$\Lip(H)\le R$, then, uniformly in $x$ and $0\le s\le S$,
\begin{equation}
    |\Cov_{s,H,x}(F,G)|
       \le C_{R,S}(1\wedge s)\Lip(F)\Lip(G)
    \label{eq:loc-covariance-bound}
\end{equation}
for Lipschitz $F,G$. In the OU case these functions may have linear growth.
The statement applies componentwise to vector and tensor covariances.
\end{lemma}

\begin{proof}
At $s=0$ the tilted kernel is a point mass and the covariance is zero.
For $s>0$, let $Y,Y'$ be independent with the tilted kernel law. Then
\[
    \Cov(F(Y),G(Y))
       =\frac12\mathbb E[(F(Y)-F(Y'))(G(Y)-G(Y'))].
\]
For any fixed center $x_0$, the triangle inequality gives
\begin{equation}
    |\Cov(F(Y),G(Y))|
       \le 2\Lip(F)\Lip(G)\,\mathbb E d(Y,x_0)^2.
    \label{eq:loc-covariance-second-moment}
\end{equation}
It therefore suffices to bound the second moment of the tilted displacement.

On the torus, take $x_0=x$. The tilted law has density
$p_s(x-y)e^{H(y)}/P_s(e^H)(x)$, and
$\osc H\le R\operatorname{diam}(\T^d)$. Consequently,
\[
\begin{aligned}
    \mathbb E d_{\T}(Y,x)^2
    &\le e^{\osc H}P_s(d_{\T}(x,\cdot)^2)(x)\\
    &\le e^{\osc H}\mathbb E
        \min\{2|B_s|^2,\operatorname{diam}(\T^d)^2\}
     \le C_{R,d}(1\wedge s),
\end{aligned}
\]
where $B_s$ is standard Brownian motion in $\R^d$ and the torus heat process
is $x+\sqrt2B_s$ modulo $2\pi\mathbb Z^d$.

For OU, set $x_0=e^{-s}x$ and $\sigma_s=\sqrt{1-e^{-2s}}$.
In Mehler's representation $Y=x_0+\sigma_s Z$, with the Gaussian law of $Z$
weighted by $e^{H(x_0+\sigma_s Z)}$.
Subtract $H(x_0)$ from the exponent. Since
$|H(x_0+\sigma_s z)-H(x_0)|\le R\sigma_s|z|$ and $\sigma_s\le1$,
\[
\begin{aligned}
    \mathbb E|Y-x_0|^2
    &=\sigma_s^2
      \frac{\mathbb E_\gamma[|Z|^2e^{H(x_0+\sigma_s Z)-H(x_0)}]}
           {\mathbb E_\gamma[e^{H(x_0+\sigma_s Z)-H(x_0)}]}\\
    &\le \sigma_s^2
      \frac{\mathbb E_\gamma[|Z|^2e^{R|Z|}]}
           {\mathbb E_\gamma[e^{-R|Z|}]}
     \le C_{R,d}(1\wedge s).
\end{aligned}
\]
The numerator is finite, the denominator is strictly positive, and both are
independent of $x$. Combining either displacement bound with
\eqref{eq:loc-covariance-second-moment} proves the lemma. The same Gaussian
exponential moment justifies the expectations and products for Lipschitz
functions of linear growth, without a boundedness assumption on their
values.
\end{proof}

The next identities are used repeatedly.  In the torus case set $\alpha_s=1$, while in the OU case set $\alpha_s=e^{-s}$.  The quotient rule and scalar-semigroup commutation give, for smooth inputs,
\[
    \partial_j\E_{s,H}[F]
       =\alpha_s\bigl(\E_{s,H}[\partial_jF]
                      +\Cov_{s,H}(F,\partial_jH)\bigr).
\]
Apply this first with $F=\nabla(g+a)$ and then subtract the target term.
This gives
\begin{equation}
    \nabla\psi_s[g]
    =\alpha_s\Bigl(\E_{s,g+a}[\nabla g+\nabla a]-\E_{s,a}[\nabla a]\Bigr),
    \label{eq:loc-score-formula}
\end{equation}
and
\begin{align}
    \nabla^2\psi_s[g]
    =\alpha_s^2\Big(&\E_{s,g+a}[\nabla^2g+\nabla^2a]-\E_{s,a}[\nabla^2a]
    \nonumber\\
    &+\Cov_{s,g+a}(\nabla g+\nabla a,\nabla g+\nabla a)
     -\Cov_{s,a}(\nabla a,\nabla a)\Big).
    \label{eq:loc-hessian-formula}
\end{align}
These identities are first used for smooth inputs. At the stated
regularity, mollification preserves the score and target norm bounds, and the positive-scale
kernel formulas pass to the limit locally uniformly.

\begin{lemma}[Bounds for the smoothed target score]
\label{lem:loc-coefficients}
There are $\varepsilon_0>0$ and $C_S<\infty$ such that, if
$\|a\|_{C^{2,1}}\le\varepsilon_0$, then the smoothed target score
\[
    c_s:=\nabla\log\mathsf P_s(e^a)
\]
satisfies, uniformly for $0\le s\le S$,
\begin{equation}
    \frac12\le\mathsf P_s(e^a)\le2,
    \label{eq:loc-r-lower-upper}
\end{equation}
and
\begin{equation}
    \|\log\mathsf P_s(e^a)\|_{L^\infty}
       +\|c_s\|_{C^{1,1}}+\|\nabla c_s\|_{C^{0,1}}
       \le C_S\|a\|_{C^{2,1}}.
    \label{eq:loc-coeff-bound}
\end{equation}
\end{lemma}

\begin{proof}
The Markov property bounds $\mathsf P_s(e^a)$ between
$e^{-\|a\|_\infty}$ and $e^{\|a\|_\infty}$, giving
\eqref{eq:loc-r-lower-upper} for small $\varepsilon_0$.
The product rule and
$\Lip(uv)\le\|u\|_\infty\Lip(v)+\|v\|_\infty\Lip(u)$ give
\[
    \|e^a-1\|_{C^{2,1}}\le C\|a\|_{C^{2,1}}
\]
on this small ball. Commutation of derivatives with the heat or OU semigroup
then gives the same bound for $\mathsf P_s(e^a)-1$.
To check the logarithm without hiding a derivative, for a positive scalar
function $r$ use
\[
    \nabla\log r=\frac{\nabla r}{r},
    \qquad
    \nabla^2\log r=\frac{\nabla^2r}{r}
                   -\frac{\nabla r\otimes\nabla r}{r^2}.
\]
Apply these identities with $r=\mathsf P_s(e^a)$.
The lower bound $r\ge1/2$ controls the reciprocal factors, and the displayed
Lipschitz product inequality controls the Lipschitz seminorm of each Hessian
term. This proves~\eqref{eq:loc-coeff-bound} using precisely two derivatives
of $a$ and the Lipschitz seminorm of its Hessian.
\end{proof}

\subsubsection{Scale estimates: controlling the averaged velocity}\label{subsec:loc-scale-estimates}

The first estimate controls the $C^{0,1}$ norm of the smoothed score uniformly
in the scale variable. The second uses the scale equation to obtain the
additional derivative required in the evolution equation. Throughout this
subsection, $g$ is smooth, $\xi:=\nabla g$, and
$y:=\|\xi\|_{C^{0,1}}$. All estimates are invariant under addition of a
constant to $g$.

\begin{proposition}[Uniform Lipschitz bound for the smoothed score]
\label{prop:loc-score-lifting}
There are $\varepsilon_0>0$, $r_0>0$, and $C_S<\infty$ such that, if
\[
    \|a\|_{C^{2,1}}\le\varepsilon_0,
    \qquad
    \|\nabla g\|_{C^{0,1}}\le r_0,
\]
then
\begin{equation}
    \sup_{0\le s\le S}\|\nabla\psi_s[g]\|_{C^{0,1}}
    \le C_S\|\nabla g\|_{C^{0,1}}.
    \label{eq:loc-score-lifting}
\end{equation}
\end{proposition}

\begin{proof}
Set $H_\theta=a+\theta g$, $0\le\theta\le1$.  If $\varepsilon_0+r_0$ is bounded, Lemma~\ref{lem:loc-covariance} applies uniformly to all $H_\theta$.  After decreasing $\varepsilon_0$ and $r_0$, all constants depend only on $d$ and $S$.

\paragraph{Step 1: $L^\infty$ estimate.} By~\eqref{eq:loc-score-formula} and $\alpha_s\le1$,
\[
\begin{aligned}
    |\nabla\psi_s[g]|
    &\le |\E_{s,g+a}[\nabla g]|
      +|\E_{s,g+a}[\nabla a]-\E_{s,a}[\nabla a]|  \\
    &\le C_dy+|\E_{s,g+a}[\nabla a]-\E_{s,a}[\nabla a]|.
\end{aligned}
\]
Differentiating $\E_{s,H_\theta}[F]$ in $\theta$ gives
\begin{equation}
    \frac{\dd}{\dd\theta}\E_{s,H_\theta}[F]
    =\Cov_{s,H_\theta}(F,g).
    \label{eq:loc-theta-M}
\end{equation}
Thus Lemma~\ref{lem:loc-covariance}, applied with $F=\nabla a$, gives
\[
    |\E_{s,g+a}[\nabla a]-\E_{s,a}[\nabla a]|
    \le C_S\|\nabla^2a\|_\infty\|\nabla g\|_\infty
    \le C_S\|a\|_{C^{2,1}}y.
\]
This proves $\|\nabla\psi_s[g]\|_{L^\infty}\le C_Sy$.

\paragraph{Step 2: Lipschitz estimate.} For smooth $g$, it suffices to bound $\|\nabla^2\psi_s[g]\|_{L^\infty}$.  In~\eqref{eq:loc-hessian-formula}, the expectation part is
\[
    \E_{s,g+a}[\nabla^2g]
    +\E_{s,g+a}[\nabla^2a]-\E_{s,a}[\nabla^2a].
\]
The first term is bounded by $C_dy$. The second difference is controlled using~\eqref{eq:loc-theta-M} and Lemma~\ref{lem:loc-covariance} with $F=\nabla^2a$:
\[
    |\E_{s,g+a}[\nabla^2a]-\E_{s,a}[\nabla^2a]|
    \le C_S\Lip(\nabla^2a)\|\nabla g\|_\infty
    \le C_S\|a\|_{C^{2,1}}y.
\]
For the covariance part, expand
\[
\begin{aligned}
&\Cov_{s,g+a}(\nabla g+\nabla a,\nabla g+\nabla a)-\Cov_{s,a}(\nabla a,\nabla a) \\
&\quad=\Cov_{s,g+a}(\nabla g,\nabla g)
   +\Cov_{s,g+a}(\nabla g,\nabla a)
   +\Cov_{s,g+a}(\nabla a,\nabla g) \\
&\qquad+\Cov_{s,g+a}(\nabla a,\nabla a)-\Cov_{s,a}(\nabla a,\nabla a).
\end{aligned}
\]
The first three terms are bounded together by
$C_S(y^2+\|a\|_{C^{2,1}}y)$. The two mixed tensor covariances are transposes of one another and are
retained separately in the identity.  For the last term, differentiate in $\theta$; componentwise,
\[
\begin{aligned}
\frac{\dd}{\dd\theta}\Cov_{s,H_\theta}(F,F)
&=\Cov_{s,H_\theta}(F\otimes F,g)
  -\E_{s,H_\theta}[F]\otimes\Cov_{s,H_\theta}(F,g) \\
&\quad-\Cov_{s,H_\theta}(F,g)\otimes \E_{s,H_\theta}[F],
\end{aligned}
\]
with $F=\nabla a$.  Lemma~\ref{lem:loc-covariance} gives a bound $C_S\|a\|_{C^{2,1}}^2y$.  Taking $\varepsilon_0,r_0\le1$, all terms are bounded by $C_Sy$.  Together with the $L^\infty$ estimate, this proves~\eqref{eq:loc-score-lifting}.
\end{proof}

The scale equation for $z_s:=\nabla\psi_s[g]$ is the common parabolic core of both proofs.

\begin{lemma}[Scale equation]
\label{lem:loc-scale-equation}
For smooth $g$,
\begin{equation}
    \partial_s z_s=\mathcal G z_s+F_s(z_s),
    \qquad
    z_0=\nabla g,
    \label{eq:loc-scale-equation}
\end{equation}
where
\begin{equation}
    F_s(z)=2(c_s+z)\cdot\nabla z+2(\nabla c_s)^Tz,
    \qquad
    c_s=\nabla\log\mathsf P_s(e^{a}),
    \quad
    \nabla c_s=\nabla^2\log\mathsf P_s(e^{a}).
    \label{eq:loc-Fs-def}
\end{equation}
\end{lemma}

\begin{proof}
Both smoothed densities solve $\partial_su=Lu$. The chain rule gives
\[
    \frac{L u}{u}=L\log u+|\nabla\log u|^2
\]
for every smooth positive $u$, both for $L=\Delta$ and for $L=L_\gamma$.
Subtracting this identity for the numerator and denominator in
\eqref{eq:loc-psi-phi-w} yields
\[
    \partial_s\psi_s
    =L\psi_s+|\nabla\psi_s|^2+2c_s\cdot\nabla\psi_s.
\]
Taking a gradient gives the result.  In the torus case $\nabla L=\Delta\nabla$, while in the OU case $\nabla L_\gamma=(L_\gamma-I)\nabla$.  The identity $\nabla(|z|^2)=2(z\cdot\nabla)z$ is used with $z=\nabla\psi_s$, a gradient field.
\end{proof}

\begin{proposition}[One-derivative gain from scale averaging]
\label{prop:loc-averaged-gain}
There are $\varepsilon_0>0$, $r_0>0$, and $C_S<\infty$ such that, if
\[
    \|a\|_{C^{2,1}}\le\varepsilon_0,
    \qquad
    \|\nabla g\|_{C^{0,1}}\le r_0,
\]
then
\begin{equation}
    \int_0^S\|\nabla\psi_s[g]\|_{C^{1,1}}\,\dd s
    \le C_S\|\nabla g\|_{C^{0,1}}.
    \label{eq:loc-averaged-gain}
\end{equation}
Consequently,
\begin{equation}
    \|W_S[\nabla g]\|_{C^{1,1}}
    \le C_S\|\nabla g\|_{C^{0,1}}.
    \label{eq:loc-W-H2}
\end{equation}
\end{proposition}

\begin{proof}
Let $z_s=\nabla\psi_s[g]$.  Duhamel's formula for~\eqref{eq:loc-scale-equation} is
\begin{equation}
    z_s=\mathsf K_s\nabla g+
        \int_0^s\mathsf K_{s-u}F_u(z_u)\,\dd u.
    \label{eq:loc-duhamel-z}
\end{equation}
Lemma~\ref{lem:loc-coefficients} gives
\begin{equation}
    \sup_{0\le u\le S}(\|c_u\|_{C^{1,1}}+\|\nabla c_u\|_{C^{0,1}})
    \le C_S\|a\|_{C^{2,1}}.
    \label{eq:loc-c-b-small}
\end{equation}
The estimates needed at this stage follow directly from
$\Lip(AB)\le\|A\|_\infty\Lip(B)+\|B\|_\infty\Lip(A)$, componentwise:

\begin{align}
    \|AB\|_{C^{0,1}}
    &\le C\|A\|_{C^{0,1}}\|B\|_{C^{0,1}},
    \nonumber\\
    \|A\cdot\nabla z\|_{C^{0,1}}
    &\le C\|A\|_{C^{0,1}}\|z\|_{C^{1,1}}.
    \label{eq:loc-product-H1}
\end{align}
Using~\eqref{eq:loc-c-b-small}, Proposition~\ref{prop:loc-score-lifting}, and~\eqref{eq:loc-product-H1},
\begin{equation}
    \|F_u(z_u)\|_{C^{0,1}}
    \le C_S(\|a\|_{C^{2,1}}+\|\nabla g\|_{C^{0,1}})\|z_u\|_{C^{1,1}}
       +C_S\|a\|_{C^{2,1}}\|\nabla g\|_{C^{0,1}}.
    \label{eq:loc-F-H1-bound}
\end{equation}
Fubini in Duhamel's formula
is legitimate for smooth inputs; for each source scale $u$, the integrated
kernel cost is at most $C_S\sqrt{S-u}\le C_S$. Thus
Lemma~\ref{lem:loc-semigroup-estimates} gives
\[
\begin{aligned}
    \int_0^S\|z_s\|_{C^{1,1}}\,\dd s
    &\le C_S\|\nabla g\|_{C^{0,1}}
       +C_S\int_0^S\|F_u(z_u)\|_{C^{0,1}}\,\dd u \\
    &\le C_S\|\nabla g\|_{C^{0,1}}
       +C_S(\|a\|_{C^{2,1}}+\|\nabla g\|_{C^{0,1}})
         \int_0^S\|z_u\|_{C^{1,1}}\,\dd u.
\end{aligned}
\]
Here the term $C_S\|a\|_{C^{2,1}}\|\nabla g\|_{C^{0,1}}$ has been absorbed into the first term after taking $\|a\|_{C^{2,1}}\le1$.  Decrease $\varepsilon_0,r_0$ so that the coefficient of the last integral
is at most $1/2$. Moving that term to the left leaves a bound with twice the
first constant, independent of higher derivatives of $g$.  This proves~\eqref{eq:loc-averaged-gain}.  The bound~\eqref{eq:loc-W-H2} follows from $W_S[\nabla g]=\frac{1}{S}\int_0^S z_s\,\dd s$.
\end{proof}

\subsubsection{The nonlinear remainder after cancellation}
\label{subsec:loc-cancellation}

Proposition~\ref{prop:loc-averaged-gain} gives
$W_S[\xi]\in C^{1,1}$, whereas the score equation contains
$\mathcal G W_S[\xi]$. A direct application of $\mathcal G$ would therefore
require two additional spatial derivatives. The nonlinear contribution
vanishes at $s=0$, and Duhamel's formula in the scale variable replaces the
action of $\mathcal G$ by the bounded difference $\mathsf K_{S-s}-I$. This
yields the required estimate without derivative loss.

For a score $\xi=\nabla g$, define
\begin{equation}
    N_S[\xi]:=\mathcal G\left(
       W_S[\xi]-\frac1S\int_0^S\mathsf K_s\xi\,\dd s\right).
    \label{eq:loc-N-def}
\end{equation}
The definition is initially distributional. The proposition below shows that
$N_S[\xi]\in C^{0,1}$ with a perturbative bound, which is the regularity
required in the score equation.

\begin{proposition}[Nonlinear remainder estimate]
\label{prop:loc-cancelled}
After decreasing the positive thresholds $\varepsilon_0,r_0$ if necessary,
if $\|a\|_{C^{2,1}}\le\varepsilon_0$ and
$\|\xi\|_{C^{0,1}}\le r_0$, with $\xi=\nabla g$ smooth, then
\begin{equation}
    N_S[\xi]=\frac1S\int_0^S(\mathsf K_{S-u}-I)
                            F_u(\nabla\psi_u[g])\,\dd u
    \label{eq:loc-cancel-identity}
\end{equation}
in distributions, and
\begin{equation}
    \|N_S[\xi]\|_{C^{0,1}}
       \le C_S\bigl(\|a\|_{C^{2,1}}+\|\xi\|_{C^{0,1}}\bigr)
                    \|\xi\|_{C^{0,1}}.
    \label{eq:loc-cancel-estimate}
\end{equation}
Here $F_u$ is the nonlinear source term in the scale equation~\eqref{eq:loc-Fs-def}.
\end{proposition}

\begin{proof}
Write $z_s=\nabla\psi_s[g]$. Both $z_s$ and $\mathsf K_s\xi$ start at
$\xi$ when $s=0$. Subtracting the linear part in Duhamel's formula gives
\[
    z_s-\mathsf K_s\xi
       =\int_0^s\mathsf K_{s-u}F_u(z_u)\,\dd u.
\]
Average and exchange the scale integrals:
\[
    W_S[\xi]-\frac1S\int_0^S\mathsf K_s\xi\,\dd s
       =\frac1S\int_0^S\int_0^{S-u}\mathsf K_\tau F_u(z_u)
                                     \,\dd\tau\,\dd u.
\]
For fixed $u$, integration of
$\partial_\tau\mathsf K_\tau=\mathcal G\mathsf K_\tau$ yields
\[
    \mathcal G\int_0^{S-u}\mathsf K_\tau F_u(z_u)\,\dd\tau
        =(\mathsf K_{S-u}-I)F_u(z_u).
\]
Testing against a compactly supported smooth field justifies this identity.
The kernel converges to the identity in distributions at $\tau=0$, and the
integrable bound~\eqref{eq:loc-F-H1-bound} permits Fubini. The argument
therefore requires neither global strong continuity nor a pointwise bound on
$\mathcal G\mathsf K_\tau$ as $\tau\downarrow0$. This proves
\eqref{eq:loc-cancel-identity}.

By non-expansiveness of $\mathsf K_s$, the operator
$\mathsf K_{S-u}-I$ has $C^{0,1}$ norm at most $2$. Hence
\[
    \|N_S[\xi]\|_{C^{0,1}}
       \le\frac2S\int_0^S\|F_u(z_u)\|_{C^{0,1}}\,\dd u.
\]
The bound on the nonlinear source term,~\eqref{eq:loc-F-H1-bound}, and the integrated gain
\eqref{eq:loc-averaged-gain} bound the last integral by
$C_S(\|a\|_{C^{2,1}}+\|\xi\|_{C^{0,1}})\|\xi\|_{C^{0,1}}$.
Thus the cancellation identity replaces the action of $\mathcal G$ by a
bounded semigroup difference and closes the estimate in $C^{0,1}$.
\end{proof}

\subsubsection{Closing the score inequality in time}
\label{subsec:loc-physical}

Combining the scale estimates with the score equation yields the differential
inequality used in the convergence argument. At each time $t$, set
$\xi_t=\nabla h_t$. The linear part of
$\mathcal G W_S[\xi]$ is explicit:
\begin{equation}
    \mathcal G\left(\frac1S\int_0^S\mathsf K_s\xi\,\dd s\right)
       =\frac{\mathsf K_S-I}{S}\xi,
    \qquad
    \lambda_S:=\frac{1-\theta_S}{S}>0,
    \label{eq:loc-AS-lambda}
\end{equation}
where $\theta_S$ is the contraction constant from
Lemma~\ref{lem:loc-semigroup-estimates}.

\begin{lemma}[Linear damping after scale averaging]
\label{lem:loc-linear-damping}
For the equation $\dot\xi=(\mathsf K_S-I)\xi/S$,
\begin{equation}
    D^+\|\xi\|_{C^{0,1}}\le-\lambda_S\|\xi\|_{C^{0,1}}
    \label{eq:loc-linear-damping}
\end{equation}
on score fields of the form $\xi=\nabla g$ in the torus case and on all vector fields in the OU
case. More precisely,
\begin{equation}
    \left\|e^{t(\mathsf K_S-I)/S}\xi\right\|_{C^{0,1}}
       \le e^{-\lambda_St}\|\xi\|_{C^{0,1}}.
    \label{eq:loc-linear-semigroup-damping}
\end{equation}
Here $D^+f(t):=\limsup_{h\downarrow0}(f(t+h)-f(t))/h$.
\end{lemma}

\begin{proof}
The operator $\mathsf K_S$ is bounded, so its exponential is defined by a
norm-convergent power series:
\[
    e^{t(\mathsf K_S-I)/S}
       =e^{-t/S}\sum_{n=0}^\infty\frac{(t/S)^n}{n!}\mathsf K_S^n.
\]
In the OU case, $\|\mathsf K_S^n\xi\|_{C^{0,1}}
\le\theta_S^n\|\xi\|_{C^{0,1}}$ for all fields. The same estimate holds
on the torus for fields of the form $\xi=\nabla g$ because $\mathsf K_S$ preserves that class.
Taking norms in the series proves~\eqref{eq:loc-linear-semigroup-damping};
applying it over a time increment proves~\eqref{eq:loc-linear-damping}.
\end{proof}

\begin{lemma}[Equation for the score]
\label{lem:loc-score-equation}
For every smooth solution of~\eqref{eq:loc-h-flow}, put
$\xi_t=\nabla h_t$ and $W_t=W_S[\xi_t]$. Then
\begin{equation}
    \partial_t\xi_t=\frac{\mathsf K_S-I}{S}\xi_t
                       +(W_t\cdot\nabla)\xi_t+R_t,
    \label{eq:loc-physical-score-eq}
\end{equation}
where the remainder is
\begin{equation}
    R_t=N_S[\xi_t]+(\nabla^2a)W_t
                    +(\nabla W_t)^T(\nabla a+\xi_t).
    \label{eq:loc-R-def}
\end{equation}
\end{lemma}

\begin{proof}
Take a gradient in~\eqref{eq:loc-h-flow}. The scalar-generator term becomes
$\mathcal G W_t$ by~\eqref{eq:loc-score-commutation}. For the products,
\[
\begin{aligned}
    \nabla(\nabla a\cdot W_t)
       &=(\nabla^2a)W_t+(\nabla W_t)^T\nabla a,\\
    \nabla(\xi_t\cdot W_t)
       &=(W_t\cdot\nabla)\xi_t+(\nabla W_t)^T\xi_t.
\end{aligned}
\]
The second equality uses the symmetry of $\nabla\xi_t=\nabla^2h_t$.
Finally,~\eqref{eq:loc-N-def} and~\eqref{eq:loc-AS-lambda} give
$\mathcal G W_t=(\mathsf K_S-I)\xi_t/S+N_S[\xi_t]$.
This proves the equation and the stated remainder, including the transport
sign.
\end{proof}

\begin{lemma}[Damped transport estimate]
\label{lem:loc-damped-transport}
Let a smooth score field satisfy~\eqref{eq:loc-physical-score-eq}, with a
smooth bounded velocity and the indicated remainder. Then
\begin{equation}
    D^+\|\xi_t\|_{C^{0,1}}+\lambda_S\|\xi_t\|_{C^{0,1}}
       \le C_d\|W_t\|_{C^{1,1}}\|\xi_t\|_{C^{0,1}}
             +\|R_t\|_{C^{0,1}}.
    \label{eq:loc-damped-transport}
\end{equation}
\end{lemma}

\begin{proof}
Let $T_{t,h}$ be the propagator from $t$ to $t+h$ for
$\partial_\tau u=W_\tau\cdot\nabla u$. Its characteristics have velocity
$-W_\tau$. The supremum norm of each transported component is preserved,
and the Lipschitz constant of the inverse characteristic map is bounded by
$\exp(\int_t^{t+h}\Lip(W_\tau)\,\dd\tau)$. Consequently,
\begin{equation}
    \|T_{t,h}f\|_{C^{0,1}}
       \le\exp\left(C_d\int_t^{t+h}\|W_\tau\|_{C^{1,1}}\,\dd\tau\right)
             \|f\|_{C^{0,1}}.
    \label{eq:loc-transport-growth}
\end{equation}
For the smooth solutions considered here, bounded higher derivatives justify
first-order Taylor expansion in $C^{0,1}$. Apply the bounded linear
evolution first and transport second:
\[
    \xi_{t+h}
       =T_{t,h}e^{h(\mathsf K_S-I)/S}\xi_t+hR_t+o(h)
       \quad\text{in }C^{0,1}.
\]
Both sides have first-order term
$\xi_t+h((\mathsf K_S-I)\xi_t/S+W_t\cdot\nabla\xi_t+R_t)$;
the difference of the two split evolutions is of higher order. The order of
composition matters in this estimate: the strict contraction is applied to
$\xi_t$, which is a score, before componentwise transport, which need not
preserve gradient fields.
Taking norms, using~\eqref{eq:loc-linear-semigroup-damping} and
\eqref{eq:loc-transport-growth}, subtracting $\|\xi_t\|_{C^{0,1}}$, and
dividing by $h$ proves the claim as $h\downarrow0$.
\end{proof}

\begin{proposition}[Differential inequality for the score]
\label{prop:loc-closed-ineq}
There are $\varepsilon_0>0$, $r_0>0$, and $C_S<\infty$ such that the following holds.  Let $h_t$ be a smooth solution on $[0,T]$ and assume
\[
    \|a\|_{C^{2,1}}\le\varepsilon_0,
    \qquad
    \sup_{0\le t\le T}\|\nabla h_t\|_{C^{0,1}}\le r_0.
\]
Then, with $y(t):=\|\nabla h_t\|_{C^{0,1}}$,
\begin{equation}
    D^+y(t)+\lambda_Sy(t)
    \le C_S(\|a\|_{C^{2,1}}+y(t))y(t),
    \qquad 0\le t\le T.
    \label{eq:loc-closed-ineq}
\end{equation}
\end{proposition}

\begin{proof}
Write $\xi_t=\nabla h_t$ and $W_t=W_S[\xi_t]$.
By Lemmas~\ref{lem:loc-damped-transport} and~\ref{lem:loc-score-equation},
\begin{equation}
    D^+y(t)+\lambda_Sy(t)
    \le C\|W_t\|_{C^{1,1}}y(t)+\|R_t\|_{C^{0,1}}.
    \label{eq:loc-preclosed}
\end{equation}
Proposition~\ref{prop:loc-averaged-gain} gives
\[
    \|W_t\|_{C^{1,1}}\le C_Sy(t).
\]
Proposition~\ref{prop:loc-cancelled} gives
\[
    \|N_S[\nabla h_t]\|_{C^{0,1}}
    \le C_S(\|a\|_{C^{2,1}}+y(t))y(t).
\]
The product estimates in $C^{0,1}$ and the bound on $W_t$ give
\[
\begin{aligned}
    \|(\nabla^2a)W_t\|_{C^{0,1}}
    +\|(\nabla W_t)^T\nabla a\|_{C^{0,1}}
    &\le C_S\|a\|_{C^{2,1}}y(t),\\
    \|(\nabla W_t)^T\xi_t\|_{C^{0,1}}
    &\le C_Sy(t)^2.
\end{aligned}
\]
Together with~\eqref{eq:loc-R-def}, this gives
\[
    \|R_t\|_{C^{0,1}}
    \le C_S(\|a\|_{C^{2,1}}+y(t))y(t).
\]
Substitution into~\eqref{eq:loc-preclosed} proves~\eqref{eq:loc-closed-ineq}.
\end{proof}

\subsubsection{A priori exponential convergence for smooth solutions}
\label{subsec:loc-conditional}

The preceding estimates imply exponential decay on any interval on which a
smooth solution exists. The entropy estimate used in this argument is recorded
first.

\begin{lemma}[Entropy from score control]
\label{lem:loc-entropy-score}
Let $a$ be bounded and normalized by~\eqref{eq:loc-target-potential}.
There is a constant depending only on $d$ and $\osc a$ such that, for every
normalized $h$ with bounded gradient,
\begin{equation}
\begin{aligned}
    \KL(e^h\rho^\star\mid\rho^\star)
       &\le C_de^{\osc a}\int|\nabla h|^2e^h\,\dd\rho^\star\\
       &\le C_de^{\osc a}\|\nabla h\|_{L^\infty}^2.
\end{aligned}
\label{eq:loc-entropy-score}
\end{equation}
The statement holds for both reference measures in
\eqref{eq:loc-model-table}.
\end{lemma}

\begin{proof}
Both normalized Lebesgue measure on the torus and standard Gaussian measure
satisfy a logarithmic Sobolev inequality
\[
    \Ent_{\mathsf m}(f^2)\le C_d\int|\nabla f|^2\,\dd\mathsf m;
\]
see~\citet{gross1975logarithmic,bakry2014analysis}.
The bounded-perturbation step can be verified directly. For a nonnegative
function $u$ and a probability measure $\nu$,
\[
    \Ent_\nu(u)
       =\inf_{c>0}\int\bigl(u\log(u/c)-u+c\bigr)\,\dd\nu.
\]
The integrand is nonnegative, so $\dd\rho^\star=e^a\,\dd\mathsf m$ implies
\[
\begin{aligned}
    \Ent_{\rho^\star}(f^2)
       &\le e^{\sup a}\Ent_{\mathsf m}(f^2)\\
       &\le C_de^{\sup a}\int|\nabla f|^2\,\dd\mathsf m
        \le C_de^{\osc a}\int|\nabla f|^2\,\dd\rho^\star.
\end{aligned}
\]
This is the Holley-Stroock comparison~\citep{holley1987logarithmic}.
Take $f=e^{h/2}$. Then $\int f^2\,\dd\rho^\star=1$ and
$|\nabla f|^2=\tfrac14e^h|\nabla h|^2$, giving
\eqref{eq:loc-entropy-score} after absorbing $1/4$ into the constant.
For an unbounded $h$ in the OU case, bounded score and probability
normalization give $|h(x)|\le C+C|x|$.
Apply the inequality to truncated, spatially cut off versions of
$e^{h/2}$ and pass to the limit. Gaussian tails dominate
$e^h(1+|h|+|\nabla h|^2)$, and the cutoff-gradient contribution tends to
zero. Thus the substitution is valid also at $C^{2,1}$ regularity.
\end{proof}

\begin{proposition}[Conditional exponential convergence]
\label{prop:loc-conditional-decay}
There are $\varepsilon_*>0$ and $r_*>0$, depending only on $d,S$, such that
the following holds. Let $h_t$ be a smooth normalized solution of
\eqref{eq:loc-h-flow} on $[0,T)$, and suppose
\[
    \|a\|_{C^{2,1}}\le\varepsilon_*,
    \qquad y(0)=\|\nabla h_0\|_{C^{0,1}}<r_*/2.
\]
Then, for every $0\le t<T$,
\begin{equation}
    \|\nabla h_t\|_{C^{0,1}}
       \le e^{-\lambda_St/2}\|\nabla h_0\|_{C^{0,1}},
    \qquad 0\le t<T,
    \label{eq:loc-conditional-score}
\end{equation}
and
\begin{equation}
    \KL(\rho_t\mid\rho^\star)
       \le C_Se^{-\lambda_St}\|\nabla h_0\|_{C^{0,1}}^2.
    \label{eq:loc-conditional-entropy}
\end{equation}
The conclusion is conditional on smooth existence on $[0,T)$; global
existence is established in Proposition~\ref{prop:loc-smooth-global}.
\end{proposition}

\begin{proof}
First choose $r_*>0$ smaller than $r_0$ and so that
$C_Sr_*\le\lambda_S/4$, where $C_S$ is the constant in
\eqref{eq:loc-closed-ineq}. Then choose
$0<\varepsilon_*\le\varepsilon_0$ with
$C_S\varepsilon_*\le\lambda_S/4$. Thus
\begin{equation}
    C_S(\varepsilon_*+r_*)\le\lambda_S/2.
    \label{eq:loc-eps-small-absorb}
\end{equation}
As long as $y(t)<r_*$, Proposition~\ref{prop:loc-closed-ineq} gives
$D^+y(t)+\lambda_Sy(t)/2\le0$. Since $y$ is continuous for a smooth
solution, the Dini form of Gronwall's inequality yields
$y(t)\le e^{-\lambda_St/2}y(0)$ up to any first exit.
At an exit time below $T$, continuity would give $y=r_*$, whereas this bound
gives $y\le y(0)<r_*/2$. No such exit can occur, proving
\eqref{eq:loc-conditional-score}. Applying
Lemma~\ref{lem:loc-entropy-score} gives
\eqref{eq:loc-conditional-entropy}.
\end{proof}

Proposition~\ref{prop:loc-conditional-decay} establishes the a priori decay
estimate for smooth solutions. It remains to establish global well-posedness
and to pass from smooth targets to targets in $C^{2,1}$.

\subsection{Existence, continuation, and passage to $C^{2,1}$ regularity}
\label{subsec:loc-wellposedness}

The smooth construction requires higher-derivative estimates that may depend
on higher derivatives of the fixed smooth target. The bounds used in the passage to
$C^{2,1}$ regularity in Section~\ref{subsec:loc-finish} depend only on $d,S$,
the $C^{2,1}$ target norm, and the $C^{0,1}$ score norm.

\subsubsection{Calculus and linear scale estimates}
\label{subsec:loc-calculus}

The following product, commutator, and transport estimates propagate higher
derivatives in the smooth construction. They do not enter the $C^{0,1}$
convergence constants.

\begin{lemma}[Product, commutator, and transport estimates]
\label{lem:loc-calculus}
For smooth bounded fields with bounded derivatives, and any integer $m\ge1$,
\begin{equation}
    \|AB\|_{C^{m-1,1}}
       \le C_m\bigl(\|A\|_{C^{0,1}}\|B\|_{C^{m-1,1}}
                    +\|A\|_{C^{m-1,1}}\|B\|_{C^{0,1}}\bigr),
    \label{eq:loc-tame-product}
\end{equation}
and, for $|\alpha|\le m$,
\begin{equation}
\begin{aligned}
    \|D^\alpha(A\cdot\nabla u)-A\cdot\nabla D^\alpha u\|_{L^\infty}
       \le C_m\bigl(&\|A\|_{C^{m-1,1}}\|u\|_{C^{0,1}}\\
                    &+\|A\|_{C^{0,1}}\|u\|_{C^{m-1,1}}\bigr).
\end{aligned}
\label{eq:loc-tame-commutator}
\end{equation}
Here and below $C_m$ may also depend on $d$.
If a smooth field solves
$\partial_tu=W\cdot\nabla u+(\nabla W)^Tu+f$, then
\begin{equation}
\begin{aligned}
    D^+\|u\|_{C^{m-1,1}}
       \le {}&C_m\|W\|_{C^{1,1}}\|u\|_{C^{m-1,1}}\\
             &+C_m\|W\|_{C^{m,1}}\|u\|_{C^{0,1}}
               +\|f\|_{C^{m-1,1}}.
\end{aligned}
\label{eq:loc-tame-transport}
\end{equation}
\end{lemma}

\begin{proof}
For smooth $u$, the norm $C^{m-1,1}$ is equivalent, with dimensional
constants, to $\sum_{|\alpha|\le m}\|D^\alpha u\|_\infty$.
Indeed, the Lipschitz seminorm of each derivative of order $m-1$ equals
the supremum of the Euclidean norm of its gradient.

The product rule reduces~\eqref{eq:loc-tame-product} to products of
intermediate derivatives. The interpolation inequalities between the lowest
and highest derivatives, followed by Young's inequality, bound each such
product by the right-hand side. For the drift commutator, the terms are
\[
    \sum_{0<\beta\le\alpha}\binom{\alpha}{\beta}
           (D^\beta A)\cdot\nabla D^{\alpha-\beta}u.
\]
When $|\alpha|=m>1$, put $j=|\beta|$. Interpolation between the first and
$m$th derivative gives the bounds
\[
\begin{aligned}
    \|D^j A\|_\infty
       &\le C_m\|A\|_{C^{0,1}}^{(m-j)/(m-1)}
                   \|A\|_{C^{m-1,1}}^{(j-1)/(m-1)},\\
    \|D^{m+1-j}u\|_\infty
       &\le C_m\|u\|_{C^{0,1}}^{(j-1)/(m-1)}
                   \|u\|_{C^{m-1,1}}^{(m-j)/(m-1)}.
\end{aligned}
\]
Their product is bounded by the sum in
\eqref{eq:loc-tame-commutator}. Lower orders follow in the same way, and
$m=1$ follows directly from the product rule. The interpolation bounds here
follow from Taylor's formula with finite differences, choosing the increment
to balance the low derivative bound and the highest-order remainder; the
same argument applies to periodic extensions on the torus.

To prove~\eqref{eq:loc-tame-transport}, differentiate the equation through
order $m-1$ and follow the characteristics of $-W$.
The undifferentiated transport term preserves each supremum norm.
Commutators are estimated by~\eqref{eq:loc-tame-commutator}; multiplication
by $\nabla W$ is estimated by~\eqref{eq:loc-tame-product}.
For the top Lipschitz seminorm, take the difference of the differentiated
equation along two characteristics. Their separation grows at most as
$\exp(\int\Lip(W)\,\dd t)$. Divide by the initial separation, take the
supremum, and then take the upper right derivative in time. This contributes
$C_m\|W\|_{C^{1,1}}\|u\|_{C^{m-1,1}}$; the differentiated products
contribute $C_m\|W\|_{C^{m,1}}\|u\|_{C^{0,1}}$.
The source term contributes its defining $C^{m-1,1}$ norm. This proves the
claimed inequality.
\end{proof}

\paragraph{Approximation at $C^{2,1}$ regularity.}
A positive convolution mollifier does not increase a supremum norm or a
Lipschitz seminorm of a derivative. It approximates a $C^{k,1}$ function
locally uniformly through order $k$, while keeping its $C^{k,1}$ norm
bounded; convergence in that full norm is neither asserted nor needed.
Kernel identities are proved first for smooth inputs, then passed to the
limit at positive scales by dominated convergence. Bounds for the last
Lipschitz seminorm pass by lower semicontinuity. The weak formulation and
operator limits are checked explicitly below.

\begin{lemma}[Linear scale estimate without derivative loss]
\label{lem:loc-linear-scale-tame}
Let $m\ge1$ and $0\le s_0<s_1\le S$.  Let $u_s$ be a smooth vector field solving
\begin{equation}
    \partial_su_s=\mathcal G u_s+A_s\cdot\nabla u_s+B_su_s+f_s,
    \qquad s_0\le s\le s_1,
    \label{eq:loc-linear-scale-eq}
\end{equation}
where $A_s$ is a smooth vector field and $B_s$ is a smooth matrix field.  Put
\[
    K_A:=\sup_{s_0\le s\le s_1}\|A_s\|_{C^{0,1}},
    \qquad
    K_B:=\int_{s_0}^{s_1}\|B_s\|_{C^{0,1}}\,\dd s .
\]
There is a constant $C_{m,S,K_A,K_B}<\infty$ such that
\begin{align}
    &\sup_{s_0\le s\le s_1}\|u_s\|_{C^{m-1,1}}
      +\int_{s_0}^{s_1}\|u_s\|_{C^{m,1}}\dd s
      \nonumber\\
    &\qquad\le
      C_{m,S,K_A,K_B}\Biggl(\|u_{s_0}\|_{C^{m-1,1}}
      +\int_{s_0}^{s_1}\Bigl[\|f_r\|_{C^{m-1,1}}
      \nonumber\\
    &\hspace{8em}+(\|A_r\|_{C^{m-1,1}}+\|B_r\|_{C^{m-1,1}})
      \|u_r\|_{C^{0,1}}\Bigr]\dd r\Biggr).
    \label{eq:loc-linear-scale-tame}
\end{align}
The estimate is valid for the torus and for the OU model.  In the OU case it is understood for bounded smooth fields with bounded derivatives.
\end{lemma}

\begin{proof}
We first prove the principal estimate used below.  Let
\[
    \mathcal L_s^0v:=\mathcal G v+A_s\cdot\nabla v .
\]
If $v$ is a smooth scalar or vector field satisfying
\[
    \partial_sv_s=\mathcal L_s^0v_s+g_s,
    \qquad s_0\le s\le s_1,
\]
then
\begin{equation}
    \sup_{s_0\le s\le s_1}\|v_s\|_{L^\infty}
    +\int_{s_0}^{s_1}\|\nabla v_s\|_{L^\infty}\,\dd s
    \le C_{S,K_A}
       \left(\|v_{s_0}\|_{L^\infty}
       +\int_{s_0}^{s_1}\|g_r\|_{L^\infty}\,\dd r\right).
    \label{eq:loc-principal-gradient-est}
\end{equation}
Estimate~\eqref{eq:loc-principal-gradient-est} follows from the score
semigroup $\mathsf K_t$ defined in~\eqref{eq:loc-model-table}.
The heat-kernel and Mehler formulas give, for $0<t\le S$,
\begin{equation}
    \|\mathsf K_t f\|_{L^\infty}\le\|f\|_{L^\infty},
    \qquad
    \|\nabla\mathsf K_t f\|_{L^\infty}
       \le C_S t^{-1/2}\|f\|_{L^\infty}.
    \label{eq:loc-Linf-gradient-semigroup}
\end{equation}
The maximum principle gives
\begin{equation}
    \sup_{s_0\le s\le s_1}\|v_s\|_{L^\infty}
    \le \|v_{s_0}\|_{L^\infty}
       +\int_{s_0}^{s_1}\|g_r\|_{L^\infty}\,\dd r.
    \label{eq:loc-principal-max}
\end{equation}
For OU on $\R^d$, this application can be localized by adding
$\epsilon e^{C(s-s_0)}(1+|x|^2)$ to the scalar upper barrier, and applying
the argument to each component and its negative. Since $A_s$ is bounded,
$C$ can be chosen so this is a supersolution for the drift $-x+A_s$.
The modified difference attains its positive maximum on a bounded set;
letting $\epsilon\downarrow0$ proves the global estimate.
Duhamel's formula relative to $\mathsf K_t$ and
\eqref{eq:loc-Linf-gradient-semigroup} give
\[
\begin{aligned}
    \|\nabla v_s\|_{L^\infty}
    &\le C_S(s-s_0)^{-1/2}\|v_{s_0}\|_{L^\infty} \\
    &\quad+C_S\int_{s_0}^s(s-r)^{-1/2}
       \bigl(K_A\|\nabla v_r\|_{L^\infty}
             +\|g_r\|_{L^\infty}\bigr)\,\dd r.
\end{aligned}
\]
On an interval $I=[\sigma_0,\sigma_1]\subset[s_0,s_1]$ of length at most $\delta$,
application of the same formula starting at $\sigma_0$, followed by Fubini, yields
\[
\begin{aligned}
    \int_I\|\nabla v_s\|_{L^\infty}\,\dd s
    &\le 2C_S\sqrt\delta
       \left(\|v_{\sigma_0}\|_{L^\infty}
             +\int_I\|g_r\|_{L^\infty}\,\dd r\right) \\
    &\quad+2C_SK_A\sqrt\delta
       \int_I\|\nabla v_r\|_{L^\infty}\,\dd r.
\end{aligned}
\]
Choose $\delta>0$, depending only on $S$ and $K_A$, so that the last
coefficient is at most $1/2$.  Absorb that term, partition $[s_0,s_1]$ into
finitely many such intervals, and use~\eqref{eq:loc-principal-max} at their
left endpoints.  Summing the resulting estimates proves
\eqref{eq:loc-principal-gradient-est}.

For smooth inputs, use the norm equivalence and commutator bounds established
in Lemma~\ref{lem:loc-calculus}.  Apply derivatives $D^\alpha$, $|\alpha|\le m$, to~\eqref{eq:loc-linear-scale-eq}, keeping $A_s\cdot\nabla$ in the principal operator.  In the torus case $D^\alpha$ commutes with $\mathcal G$.  In the OU case the commutator is explicit:
$[D^\alpha,\mathcal G]u=-|\alpha|D^\alpha u$. It has order at most
$m$, not $m+1$; in particular
\[
    \|[D^\alpha,\mathcal G]u_s\|_{L^\infty}
    \le C_m\|u_s\|_{C^{m-1,1}}.
\]
The drift commutator satisfies
\[
    \|[D^\alpha,A_s\cdot\nabla]u_s\|_{L^\infty}
    \le C_m\bigl(\|A_s\|_{C^{m-1,1}}\|u_s\|_{C^{0,1}}
      +\|A_s\|_{C^{0,1}}\|u_s\|_{C^{m-1,1}}\bigr),
    \qquad |\alpha|\le m,
\]
and the zeroth-order term satisfies the product bound without derivative loss
\[
    \|B_su_s\|_{C^{m-1,1}}
    \le C_m\bigl(\|B_s\|_{C^{0,1}}\|u_s\|_{C^{m-1,1}}
      +\|B_s\|_{C^{m-1,1}}\|u_s\|_{C^{0,1}}\bigr).
\]
The bounded derivative of order $m$ represents the top Lipschitz seminorm,
and its gradient supplies the integrand $C^{m,1}$ norm. Hence the principal
estimate provides one derivative in the integrated norm without increasing
the required regularity of the coefficients.

Using~\eqref{eq:loc-principal-gradient-est} for each differentiated equation gives, for $s_0\le t\le s_1$,
\[
\begin{aligned}
    X_m(t)
    &:= \sup_{s_0\le s\le t}\|u_s\|_{C^{m-1,1}}
      +\int_{s_0}^t\|u_s\|_{C^{m,1}}\,\dd s \\
    &\le C_{m,S,K_A}
      \Biggl(
      \|u_{s_0}\|_{C^{m-1,1}}
      +\int_{s_0}^t\Bigl[\|f_r\|_{C^{m-1,1}}
      +(1+\|A_r\|_{C^{0,1}}+\|B_r\|_{C^{0,1}})\|u_r\|_{C^{m-1,1}} \\
    &\hspace{38mm}
      +(\|A_r\|_{C^{m-1,1}}+\|B_r\|_{C^{m-1,1}})\|u_r\|_{C^{0,1}}\Bigr] \dd r
      \Biggr).
\end{aligned}
\]
The integral of $\|A_r\|_{C^{0,1}}$ over any subinterval is bounded by $S K_A$, and the integral of $\|B_r\|_{C^{0,1}}$ over any subinterval is bounded by $K_B$.  Gronwall's lemma therefore absorbs the term containing $\|u_r\|_{C^{m-1,1}}$ and gives~\eqref{eq:loc-linear-scale-tame} at $t=s_1$.
\end{proof}

\subsubsection{Smooth construction and continuation}
\label{subsec:loc-smooth-construction}

Consider smooth bounded score fields of the form $\xi=\nabla g$, with
bounded spatial derivatives of every order. On the torus, $g$ is periodic. On $\R^d$, a smooth
curl-free field has a global primitive.
For the smooth targets in this subsection, every positive-order derivative
of $a$ is bounded; smallness is imposed only on $\|a\|_{C^{2,1}}$.

To reconstruct the density we choose one primitive, denoted by $g_\xi$.
On the torus it is determined by
\begin{equation}
    \nabla g_\xi=\xi,
    \qquad \int_{\T^d}g_\xi\,\dd x=0.
    \label{eq:loc-torus-primitive}
\end{equation}
In the OU case take the primitive anchored at the origin:
\begin{equation}
    g_\xi(x):=\int_0^1\xi(\theta x)\cdot x\,\dd\theta,
    \qquad g_\xi(0)=0.
    \label{eq:loc-ou-primitive}
\end{equation}
Indeed, closedness gives
\[
    \partial_jg_\xi(x)
       =\int_0^1\left(\xi_j(\theta x)
                    +\theta\sum_i x_i\partial_i\xi_j(\theta x)\right)\dd\theta
       =\int_0^1\frac\dd{\dd\theta}\bigl(\theta\xi_j(\theta x)\bigr)\dd\theta
       =\xi_j(x).
\]
The primitive has at most linear growth when $\xi$ is bounded. Its additive
constant affects neither $W_S[\xi]$ nor any score estimate.

For the characteristic iteration, define the remaining linear and
target-dependent terms by
\begin{equation}
    E_S[\xi]:=\frac{\mathsf K_S-I}{S}\xi+N_S[\xi]
                   +(\nabla^2a)W_S[\xi]+(\nabla W_S[\xi])^T\nabla a.
    \label{eq:loc-E-map}
\end{equation}
The score equation is then
\begin{equation}
    \partial_t\xi=(W_S[\xi]\cdot\nabla)\xi
                   +(\nabla W_S[\xi])^T\xi+E_S[\xi].
    \label{eq:loc-abstract-score}
\end{equation}
This is the same differentiated equation as
\eqref{eq:loc-physical-score-eq}; its transport terms are now grouped so that
freezing the coefficients preserves gradients.

\begin{lemma}[Smooth dependence estimates for the local theory]
\label{lem:loc-tame}
Let $a\in C^\infty$ satisfy~\eqref{eq:loc-target-potential}; in the OU model assume that every positive-order derivative of $a$ is bounded.  There are $\varepsilon_{\rm loc}>0$ and $r_{\rm loc}>0$, no larger than the smallness thresholds in Propositions~\ref{prop:loc-score-lifting}-\ref{prop:loc-cancelled}, with the following property.  If $\|a\|_{C^{2,1}}\le\varepsilon_{\rm loc}$, then for every integer $q\ge1$, every $R<r_{\rm loc}$, and every $M<\infty$, there is $C_{q,R,M,a,S}<\infty$ such that, whenever two smooth bounded score fields $\xi=\nabla g$ and
$\tilde\xi=\nabla\tilde g$, with all
spatial derivatives bounded, satisfy
\[
    \|\xi\|_{C^{0,1}},\|\tilde\xi\|_{C^{0,1}}\le R,
    \qquad
    \|\xi\|_{C^{q,1}},\|\tilde\xi\|_{C^{q,1}}\le M,
\]
one has
\begin{equation}
    \|W_S[\xi]-W_S[\tilde\xi]\|_{C^{q+1,1}}
    +\|E_S[\xi]-E_S[\tilde\xi]\|_{C^{q,1}}
    \le C_{q,R,M,a,S}\|\xi-\tilde\xi\|_{C^{q,1}}.
    \label{eq:loc-tame-Lip}
\end{equation}
Moreover, whenever $\|\xi\|_{C^{0,1}}\le R$,
\begin{equation}
    \|W_S[\xi]\|_{C^{q+1,1}}
    +\|E_S[\xi]\|_{C^{q,1}}
    \le C_{q,R,a,S}(1+\|\xi\|_{C^{q,1}}).
    \label{eq:loc-tame-one}
\end{equation}
\end{lemma}

\begin{proof}
The product and commutator estimates in Lemma~\ref{lem:loc-calculus} ensure
that each higher-order derivative is multiplied only by a controlled
$C^{0,1}$ factor. Consequently, the iteration does not lose derivatives.

Let $z_s^\xi:=\nabla\psi_s[g_{\xi}]$.  By Lemma~\ref{lem:loc-scale-equation},
\begin{equation}
    \partial_s z_s^\xi=\mathcal Gz_s^\xi+F_s(z_s^\xi),
    \qquad
    z_0^\xi=\xi.
    \label{eq:loc-z-xi-scale}
\end{equation}
Since $a$ is smooth, the coefficient field $c_s$ and its derivative are bounded in every $C^{m-1,1}$, uniformly for $0\le s\le S$.  By Lemma~\ref{lem:loc-coefficients}, their $C^{0,1}$ norms are $O(\|a\|_{C^{2,1}})$.  By Proposition~\ref{prop:loc-score-lifting}, after decreasing $\varepsilon_{\rm loc}$ and $r_{\rm loc}$ if necessary,
\begin{equation}
    \sup_{0\le s\le S}\|z_s^\xi\|_{C^{0,1}}\le C_SR
    \label{eq:loc-z-low-bound}
\end{equation}
whenever $\|\xi\|_{C^{0,1}}\le R<r_{\rm loc}$.

We first establish the estimate for a fixed input $\xi$. Set $m=q+1$ and apply Lemma~\ref{lem:loc-linear-scale-tame} to~\eqref{eq:loc-z-xi-scale} with
\[
    u_s=z_s^\xi,
    \qquad
    A_s=2(c_s+z_s^\xi),
    \qquad
    B_s=2(\nabla c_s),
    \qquad
    f_s=0.
\]
The $C^{0,1}$ norms of the coefficients are uniformly bounded by
\eqref{eq:loc-z-low-bound} and Lemma~\ref{lem:loc-coefficients}.  Moreover,
\[
    \|A_s\|_{C^{m-1,1}}\le C_{m,a,S}\bigl(1+\|z_s^\xi\|_{C^{m-1,1}}\bigr),
    \qquad
    \|B_s\|_{C^{m-1,1}}\le C_{m,a,S}.
\]
Applying~\eqref{eq:loc-linear-scale-tame} on $[0,t]$, $0\le t\le S$, and using $\sup_s\|z_s^\xi\|_{C^{0,1}}\le C_SR$, gives
\[
\begin{aligned}
    X_m(t)
    &:=\sup_{0\le s\le t}\|z_s^\xi\|_{C^{m-1,1}}
      +\int_0^t\|z_s^\xi\|_{C^{m,1}}\,\dd s  \\
    &\le C_{m,R,a,S}\left(\|\xi\|_{C^{m-1,1}}
      +1+\int_0^t X_m(r)\,\dd r\right).
\end{aligned}
\]
Gronwall's lemma gives
\begin{equation}
    \sup_{0\le s\le S}\|z_s^\xi\|_{C^{q,1}}
    +\int_0^S\|z_s^\xi\|_{C^{q+1,1}}\,\dd s
    \le C_{q,R,a,S}(1+\|\xi\|_{C^{q,1}}).
    \label{eq:loc-z-one-est}
\end{equation}
In particular, if the input is bounded by $M$ in $C^{q,1}$, the right-hand side is bounded by $C_{q,R,M,a,S}$.

For two inputs, put $v_s=z_s^\xi-z_s^{\tilde\xi}$.  The exact difference is
\begin{equation}
    \partial_sv_s=\mathcal Gv_s
    +2(c_s+z_s^\xi)\cdot\nabla v_s
    +2(\nabla c_s)^Tv_s
    +2v_s\cdot\nabla z_s^{\tilde\xi},
    \qquad
    v_0=\xi-\tilde\xi,
    \label{eq:loc-v-scale}
\end{equation}
and equivalently
\begin{equation}
    F_s(z_s^\xi)-F_s(z_s^{\tilde\xi})
    =2(c_s+z_s^\xi)\cdot\nabla v_s
     +2v_s\cdot\nabla z_s^{\tilde\xi}
     +2(\nabla c_s)^Tv_s.
    \label{eq:loc-F-difference}
\end{equation}
First apply Lemma~\ref{lem:loc-linear-scale-tame} to~\eqref{eq:loc-v-scale} with $m=1$.  The product estimate gives
\[
    \|v_s\cdot\nabla z_s^{\tilde\xi}\|_{C^{0,1}}
    \le C\|z_s^{\tilde\xi}\|_{C^{1,1}}\|v_s\|_{C^{0,1}},
\]
and $\int_0^S\|z_s^{\tilde\xi}\|_{C^{1,1}}\,\dd s$ is bounded by~\eqref{eq:loc-z-one-est}.  Hence Gronwall gives the $C^{0,1}$ difference estimate
\begin{equation}
    \sup_{0\le s\le S}\|v_s\|_{C^{0,1}}
    +\int_0^S\|v_s\|_{C^{1,1}}\,\dd s
    \le C_{R,M,a,S}\|\xi-\tilde\xi\|_{C^{0,1}}.
    \label{eq:loc-v-low-diff}
\end{equation}
Applying Lemma~\ref{lem:loc-linear-scale-tame} to~\eqref{eq:loc-v-scale} with
$m=q+1$, the only source term is $2v_s\cdot\nabla z_s^{\tilde\xi}$. The
product estimate gives
\begin{equation}
    \|v_s\cdot\nabla z_s^{\tilde\xi}\|_{C^{q,1}}
    \le C_q\|z_s^{\tilde\xi}\|_{C^{1,1}}\|v_s\|_{C^{q,1}}
      +C_q\|z_s^{\tilde\xi}\|_{C^{q+1,1}}\|v_s\|_{C^{0,1}}.
    \label{eq:loc-v-forcing-est}
\end{equation}
The coefficient $\|z_s^{\tilde\xi}\|_{C^{1,1}}$ is integrable in $s$ by~\eqref{eq:loc-z-one-est}; the second term on the right of~\eqref{eq:loc-v-forcing-est} is integrable and bounded by~\eqref{eq:loc-z-one-est} and~\eqref{eq:loc-v-low-diff}.  The lower-order coefficient contribution in~\eqref{eq:loc-linear-scale-tame}, namely
\[
    (\|c_s+z_s^\xi\|_{C^{q,1}}+
      \|\nabla c_s\|_{C^{q,1}})\|v_s\|_{C^{0,1}},
\]
is also integrable and bounded by~\eqref{eq:loc-z-one-est} and~\eqref{eq:loc-v-low-diff}.  Gronwall therefore yields
\begin{equation}
    \sup_{0\le s\le S}\|z_s^\xi-z_s^{\tilde\xi}\|_{C^{q,1}}
    +\int_0^S\|z_s^\xi-z_s^{\tilde\xi}\|_{C^{q+1,1}}\,\dd s
    \le C_{q,R,M,a,S}\|\xi-\tilde\xi\|_{C^{q,1}}.
    \label{eq:loc-z-diff-est}
\end{equation}
Combining~\eqref{eq:loc-F-difference},~\eqref{eq:loc-v-low-diff},~\eqref{eq:loc-v-forcing-est}, and~\eqref{eq:loc-z-diff-est} gives
\begin{equation}
    \int_0^S\|F_s(z_s^\xi)-F_s(z_s^{\tilde\xi})\|_{C^{q,1}}\,\dd s
    \le C_{q,R,M,a,S}\|\xi-\tilde\xi\|_{C^{q,1}}.
    \label{eq:loc-F-diff-integrated}
\end{equation}
Similarly,~\eqref{eq:loc-z-one-est} and the product estimate give
\begin{equation}
    \int_0^S\|F_s(z_s^\xi)\|_{C^{q,1}}\,\dd s
    \le C_{q,R,a,S}(1+\|\xi\|_{C^{q,1}}).
    \label{eq:loc-F-one-integrated}
\end{equation}

Since $W_S[\xi]=\frac{1}{S}\int_0^S z_s^\xi\,\dd s$, estimates~\eqref{eq:loc-z-one-est} and~\eqref{eq:loc-z-diff-est} give the asserted bounds for $W_S[\xi]$ and $W_S[\xi]-W_S[\tilde\xi]$. For the correction term, use the cancellation formula~\eqref{eq:loc-cancel-identity}. For two inputs,
\[
    N_S[\xi]-N_S[\tilde\xi]
    =\frac{1}{S}\int_0^S(\mathsf K_{S-u}-I)
      \bigl(F_u(z_u^\xi)-F_u(z_u^{\tilde\xi})\bigr)\,\dd u.
\]
The operator $\mathsf K_{S-u}-I$ is bounded on $C^{q,1}$ by~\eqref{eq:loc-semigroup-high-bound}, and~\eqref{eq:loc-F-diff-integrated} gives the difference bound for $N_S$. The bound for $N_S[\xi]$ follows in the same way from~\eqref{eq:loc-F-one-integrated}. Finally, $(\mathsf K_S-I)/S$ is bounded on every $C^{q,1}$ by~\eqref{eq:loc-semigroup-high-bound}, and multiplication by the fixed smooth coefficients $\nabla a$ and $\nabla^2a$ is controlled by the same product estimate. This proves~\eqref{eq:loc-tame-Lip} and~\eqref{eq:loc-tame-one}.
\end{proof}

\begin{proposition}[Smooth local construction and continuation]
\label{prop:loc-local}
Let $\varepsilon_{\rm loc},r_{\rm loc}$ be as in
Lemma~\ref{lem:loc-tame}, and let $a$ be a smooth normalized target with
$\|a\|_{C^{2,1}}\le\varepsilon_{\rm loc}$. In the OU model assume that all
positive-order derivatives of $a$ are bounded. Fix $0<r<r_{\rm loc}$ and a
smooth bounded initial score $\xi_{\rm in}=\nabla g_{\rm in}$, with bounded derivatives
of every order, such that $\|\xi_{\rm in}\|_{C^{0,1}}<r$.

Then the score equation~\eqref{eq:loc-abstract-score} has a unique smooth
solution $\xi_t$ in the same class, on its maximal interval $[0,T_{\max})$
within the ball $\|\xi_t\|_{C^{0,1}}<r$.  If $g_t=g_{\xi_t}$ and
\begin{equation}
    h_t:=g_t-
    \log\int e^{g_t}\,\dd\rho^\star,
    \qquad
    \dd\rho_t=e^{h_t}\,\dd\rho^\star,
    \label{eq:loc-reconstruct-h}
\end{equation}
then $h_t$ solves~\eqref{eq:loc-h-flow} and $\rho_t$ solves~\eqref{eq:loc-rho-flow}.  Conversely, every smooth normalized solution of~\eqref{eq:loc-h-flow} gives, after taking its score, a solution of~\eqref{eq:loc-abstract-score}.  Finally, if $T_{\max}<\infty$ and
\begin{equation}
    \limsup_{t\uparrow T_{\max}}\|\xi_t\|_{C^{0,1}}<r,
    \label{eq:loc-continuation-cond}
\end{equation}
then the solution extends beyond $T_{\max}$.
\end{proposition}

\begin{proof}
The score equation is constructed by Picard iteration. Fix an integer
$q\ge2$ and choose
\[
    \|\xi_{\rm in}\|_{C^{0,1}}<R<r
\]
and then choose $M>2\|\xi_{\rm in}\|_{C^{q,1}}$.  Set $\xi^{(0)}_t=\xi_{\rm in}$.  Given $\xi^{(n)}$, define $W^{(n)}=W_S[\xi^{(n)}]$ and $E^{(n)}=E_S[\xi^{(n)}]$, and solve the linear transport system
\begin{equation}
    \partial_t\xi^{(n+1)}
    =(W^{(n)}\cdot\nabla)\xi^{(n+1)}
     +(\nabla W^{(n)})^T\xi^{(n+1)}+E^{(n)},
    \qquad
    \xi^{(n+1)}_{t=0}=\xi_{\rm in}.
    \label{eq:loc-picard}
\end{equation}
For fixed smooth $W^{(n)},E^{(n)}$, this equation is solved by characteristics.

The iteration preserves this gradient structure. Indeed, if $\xi=\nabla g$,
\eqref{eq:loc-N-def} and~\eqref{eq:loc-score-commutation} give
\[
    E_S[\xi]
       =\nabla\left[
         L\left(\frac1S\int_0^S\psi_s[g_\xi]\,\dd s\right)
          +\nabla a\cdot W_S[\xi]\right].
\]
Thus the right-hand side is a gradient. The scalar linear transport equation
with velocity $-W^{(n)}$ and source equal to the expression in brackets may be
solved and then differentiated. The product rule gives
\eqref{eq:loc-picard}, so $\xi^{(n+1)}$ is again a gradient. The scalar
solution remains periodic on the torus, so no nonzero periods are introduced.
On $\R^d$ its source has bounded gradient and at most linear growth, which is
compatible with bounded-velocity characteristics and the chosen primitive.

The transport estimate~\eqref{eq:loc-tame-transport}, with $m=q+1$,
gives the coarser local bound
\begin{equation}
    \frac{\dd^+}{\dd t}\|\xi^{(n+1)}_t\|_{C^{q,1}}
    \le C_q\|W^{(n)}_t\|_{C^{q+1,1}}\|\xi^{(n+1)}_t\|_{C^{q,1}}
       +\|E^{(n)}_t\|_{C^{q,1}}.
    \label{eq:loc-picard-high}
\end{equation}
The time interval is chosen inductively.  Suppose that on $[0,T]$
\[
    \sup_{0\le t\le T}\|\xi^{(n)}_t\|_{C^{q,1}}\le M,
    \qquad
    \sup_{0\le t\le T}\|\xi^{(n)}_t\|_{C^{0,1}}\le R .
\]
Lemma~\ref{lem:loc-tame} gives
\[
    \|W^{(n)}_t\|_{C^{q+1,1}}
    +\|E^{(n)}_t\|_{C^{q,1}}
    \le C_{q,R,M,a,S},
    \qquad 0\le t\le T.
\]
Thus~\eqref{eq:loc-picard-high} implies
\[
    \|\xi^{(n+1)}_t\|_{C^{q,1}}
    \le e^{C_{q,R,M,a,S}t}\bigl(\|\xi_{\rm in}\|_{C^{q,1}}+1\bigr)-1 .
\]
The same transport estimate in $C^{0,1}$, using the $C^{0,1}$ bounds of
Propositions~\ref{prop:loc-averaged-gain} and~\ref{prop:loc-cancelled},
and the explicit term~\eqref{eq:loc-E-map}, gives
\[
    \|\xi^{(n+1)}_t\|_{C^{0,1}}
    \le e^{C_{R,M,a,S}t}\bigl(\|\xi_{\rm in}\|_{C^{0,1}}+1\bigr)-1 .
\]
Since $\|\xi_{\rm in}\|_{C^{q,1}}<M/2$ and
$\|\xi_{\rm in}\|_{C^{0,1}}<R$, choose $T>0$ so that the two right-hand
sides are bounded by $M$ and $R$, respectively. Induction in $n$ then gives the
invariant bounds
\begin{equation}
    \sup_{0\le t\le T}\|\xi_t^{(n)}\|_{C^{q,1}}\le M,
    \qquad
    \sup_{0\le t\le T}\|\xi_t^{(n)}\|_{C^{0,1}}\le R<r.
    \label{eq:loc-picard-invariant-ball}
\end{equation}

Set $\Delta^{(n+1)}:=\xi^{(n+1)}-\xi^{(n)}$.  Subtracting the equations for $\xi^{(n+1)}$ and $\xi^{(n)}$ gives
\[
\begin{aligned}
    \partial_t\Delta^{(n+1)}
    &=(W^{(n)}\cdot\nabla)\Delta^{(n+1)}
      +(\nabla W^{(n)})^T\Delta^{(n+1)}+G^{(n)},\\
    G^{(n)}
    &:=(W^{(n)}-W^{(n-1)})\cdot\nabla\xi^{(n)}
      +(\nabla W^{(n)}-\nabla W^{(n-1)})^T\xi^{(n)}
      +E^{(n)}-E^{(n-1)} .
\end{aligned}
\]
The product estimate and~\eqref{eq:loc-tame-Lip}, applied at order $q-1$, give
\[
    \|G^{(n)}_t\|_{C^{q-1,1}}
    \le C_{q,R,M,a,S}\|\xi^{(n)}_t-\xi^{(n-1)}_t\|_{C^{q-1,1}}.
\]
Using the transport estimate in $C^{q-1,1}$ and~\eqref{eq:loc-picard-invariant-ball},
\[
    \sup_{0\le t\le T}\|\Delta_t^{(n+1)}\|_{C^{q-1,1}}
    \le C_{q,R,M,a,S}T e^{C_{q,R,M,a,S}T}
    \sup_{0\le t\le T}\|\Delta_t^{(n)}\|_{C^{q-1,1}}.
\]
After reducing $T$ so that the prefactor is $<1$, this is a contraction.  Thus the iterates converge in $C([0,T];C^{q-1,1})$ to a solution bounded in $L^\infty([0,T];C^{q,1})$.  The same difference estimate gives uniqueness in this class.  Taking $q=2$ first gives a solution on a fixed interval. The lifespan can be
chosen independently of the higher regularity index. For any $m>2$, apply
\eqref{eq:loc-tame-transport} to the iterates, using their uniform $C^{0,1}$
bounds and~\eqref{eq:loc-tame-one}:
\[
    D^+\|\xi^{(n+1)}\|_{C^{m-1,1}}
       \le C_{m,R,a,S}\bigl(1+\|\xi^{(n+1)}\|_{C^{m-1,1}}
                              +\|\xi^{(n)}\|_{C^{m-1,1}}\bigr).
\]
A scalar comparison function $B$ solving
$B'=C_{m,R,a,S}(1+2B)$, with $B(0)$ larger than the initial $C^{m-1,1}$ norm,
bounds all iterates by induction on $n$ throughout the same interval.
Passing to the limit and iterating this argument in $m$ gives bounded spatial
derivatives of every order. The equation supplies the corresponding time
regularity, and the locally constructed solution is therefore smooth.

\paragraph{Reconstruction and probability normalization.}
Let $g_t=g_{\xi_t}$. Since the score equation is a spatial gradient,
there is a scalar function $\zeta(t)$ such that
\begin{equation}
    \partial_tg_t
       =L\left(\frac1S\int_0^S\psi_s[g_t]\,\dd s\right)
           +(\nabla a+\xi_t)\cdot W_S[\xi_t]-\zeta(t).
    \label{eq:loc-g-with-zeta}
\end{equation}
The scalar function enforces the chosen anchoring of $g_t$ and has no
effect on $\xi_t$. Set
$c(t)=\log\int e^{g_t}\,\dd\rho^\star$ and $h_t=g_t-c(t)$.
With $\dd\mathsf m$ proportional to $e^{-V}\,\dd x$, where $V=0$ on the
torus and $V=|x|^2/2$ for OU, integration by parts gives
\[
\begin{aligned}
   &\int e^{h_t}\left[
       L\left(\frac1S\int_0^S\psi_s[h_t]\,\dd s\right)
       +(\nabla a+\nabla h_t)\cdot W_S[\xi_t]\right]\dd\rho^\star\\
   &\quad=\int\left[\nabla\cdot W_S[\xi_t]
              +(\nabla h_t+\nabla a-\nabla V)\cdot W_S[\xi_t]\right]
                 e^{h_t}\,\dd\rho^\star=0.
\end{aligned}
\]
The integrand is the divergence of the flux written relative to Lebesgue
measure. For OU, use cutoffs equal to one on a ball of radius $R$, with
gradient bounded by $C/R$. The boundary term is bounded by
$C\|W_S[\xi_t]\|_\infty/R$ times a probability tail and tends to zero.
The needed integrability also follows from the linear growth of $g_t$ and
Gaussian tails.
This also justifies differentiating the normalization integral: under the
origin anchoring, bounded $\partial_t\xi_t$ makes $\partial_tg_t$ grow at
most linearly, uniformly on finite time intervals. Gaussian domination then
applies; the torus case is immediate by compactness. Differentiating the
normalization integral in~\eqref{eq:loc-g-with-zeta} now yields $c'(t)=-\zeta(t)$.
Therefore $h_t=g_t-c(t)$ solves exactly~\eqref{eq:loc-h-flow}, without an
uncontrolled time-dependent constant. Multiplying by the density gives
\eqref{eq:loc-rho-flow}. Conversely, differentiating any smooth normalized
solution gives the score equation.

\paragraph{Continuation using only the $C^{0,1}$ score bound.}
It remains to prove continuation.  Suppose~\eqref{eq:loc-continuation-cond} holds.  Choose $R<r$ and $t_0<T_{\max}$ such that $\|\xi_t\|_{C^{0,1}}\le R$ for $t_0\le t<T_{\max}$.  The transport estimate~\eqref{eq:loc-tame-transport} for~\eqref{eq:loc-abstract-score} gives, for $q\ge1$,
\[
\begin{aligned}
    \frac{\dd^+}{\dd t}\|\xi_t\|_{C^{q,1}}
    &\le C_q\|W_S[\xi_t]\|_{C^{1,1}}\|\xi_t\|_{C^{q,1}}
       +C_q\|W_S[\xi_t]\|_{C^{q+1,1}}\|\xi_t\|_{C^{0,1}} \\
    &\quad+\|E_S[\xi_t]\|_{C^{q,1}}.
\end{aligned}
\]
Using Proposition~\ref{prop:loc-averaged-gain} for the $C^{0,1}$ norm and Lemma~\ref{lem:loc-tame} for the higher $C^{q,1}$ norms,
\[
    \frac{\dd^+}{\dd t}\|\xi_t\|_{C^{q,1}}
    \le C_{q,R,a,S}(1+\|\xi_t\|_{C^{q,1}}),
    \qquad t_0\le t<T_{\max}.
\]
Since $T_{\max}<\infty$, Gronwall's lemma implies that every finite
$C^{q,1}$ norm remains bounded up to $T_{\max}$.  The equation then bounds $\partial_t\xi_t$ in the corresponding lower-order norms, so $\xi_t$ converges in every finite bounded-derivative norm to a smooth
score $\xi_*=\nabla g_*$ with $\|\xi_*\|_{C^{0,1}}\le R<r$.  Restarting the local construction at $\xi_*$ extends the solution beyond $T_{\max}$, a contradiction.  This proves the continuation criterion.
\end{proof}

\subsubsection{Global existence for smooth targets}
\label{subsec:loc-global-smooth}

The local construction and the a priori decay estimate from
Section~\ref{subsec:loc-conditional} yield global existence for smooth
targets.

\begin{proposition}[Global smooth score decay]
\label{prop:loc-smooth-global}
After decreasing $\varepsilon_*>0$ from
Proposition~\ref{prop:loc-conditional-decay}, the following holds.
Let $a$ be smooth, normalized by~\eqref{eq:loc-target-potential}, with
$\|a\|_{C^{2,1}}\le\varepsilon_*$. In the OU case assume also that all
positive-order derivatives of $a$ are bounded. Then the solution with
$h_0=-a$ exists globally and satisfies
\begin{equation}
    \|\nabla h_t\|_{C^{0,1}}
       \le e^{-\lambda_St/2}\|\nabla a\|_{C^{0,1}},
    \qquad t\ge0.
    \label{eq:loc-smooth-score-decay}
\end{equation}
\end{proposition}

\begin{proof}
Choose a continuation radius $0<r<r_{\rm loc}$. Decrease $\varepsilon_*$
so that $\varepsilon_*\le\varepsilon_{\rm loc}$ and
\begin{equation}
    \|\nabla a\|_{C^{0,1}}<\tfrac12\min\{r,r_*\}
    \quad\text{whenever }\|a\|_{C^{2,1}}\le\varepsilon_*.
    \label{eq:loc-eps-less-r}
\end{equation}
This uses only the fixed-dimensional inequality
$\|\nabla a\|_{C^{0,1}}\le C_d\|a\|_{C^{2,1}}$.
Proposition~\ref{prop:loc-local}, with initial score $-\nabla a$, gives a
maximal smooth solution on $[0,T_{\max})$. The chosen primitive of this
score is $-a$ plus a constant. Probability normalization removes that
constant because $\int e^{-a}\,\dd\rho^\star=\int1\,\dd\mathsf m=1$,
so the reconstructed solution satisfies $h_0=-a$.

Proposition~\ref{prop:loc-conditional-decay} gives
\eqref{eq:loc-smooth-score-decay} for $t<T_{\max}$.
If $T_{\max}<\infty$, then
\[
    \limsup_{t\uparrow T_{\max}}\|\nabla h_t\|_{C^{0,1}}
       \le\|\nabla a\|_{C^{0,1}}<r.
\]
The continuation criterion in Proposition~\ref{prop:loc-local} extends the
solution beyond $T_{\max}$, a contradiction. Thus $T_{\max}=\infty$.
The corresponding entropy estimate follows from
Lemma~\ref{lem:loc-entropy-score}.
\end{proof}

\subsubsection{Operator bounds at $C^{2,1}$ regularity}
\label{subsec:loc-endpoint-uniqueness}

The smooth contraction argument uses one derivative more than is available when
$a\in C^{2,1}$. For uniqueness at this regularity we therefore compare two
solutions directly in $L^\infty$. The next proposition shows that the velocity
and the remaining terms in the score equation depend Lipschitz-continuously on
the score in that weaker norm, so no extra spatial derivative is needed.

For $g$ with $\nabla g\in C^{0,1}$, let
$z_s^g:=\nabla\psi_s[g]$.  The velocity still has the definition~\eqref{eq:loc-psi-phi-w}. The remaining term
$E_S[\nabla g]$ is defined by~\eqref{eq:loc-E-map}, with
$N_S[\nabla g]$ interpreted through the cancellation identity
\eqref{eq:loc-cancel-identity}.  The following proposition shows that these
quantities are well defined at $C^{2,1}$ regularity.

\begin{proposition}[Operator bounds at $C^{2,1}$ regularity]
\label{prop:loc-endpoint-operators}
Fix $S>0$, $M<\infty$, and $R<\infty$.  There is
$C_{M,R,S}<\infty$ such that the following holds in both models.  Let
$a\in C^{2,1}$ satisfy~\eqref{eq:loc-target-potential} and
$\|a\|_{C^{2,1}}\le M$.  If $g,\widetilde g$ have scores in
$C^{0,1}$ satisfying
\[
    \|\nabla g\|_{C^{0,1}},
    \|\nabla\widetilde g\|_{C^{0,1}}\le R,
\]
then
\begin{equation}
\begin{aligned}
    &\sup_{0\le s\le S}\|z_s^g\|_{C^{0,1}}
      +\int_0^S\|z_s^g\|_{C^{1,1}}\,\dd s
      +\|W_S[\nabla g]\|_{C^{1,1}}
      +\|E_S[\nabla g]\|_{L^\infty}
      \le C_{M,R,S},
\end{aligned}
\label{eq:loc-endpoint-one-input}
\end{equation}
and
\begin{equation}
\begin{aligned}
    &\|W_S[\nabla g]-W_S[\nabla\widetilde g]\|_{C^{0,1}}
      +\|E_S[\nabla g]-E_S[\nabla\widetilde g]\|_{L^\infty} \\
    &\hspace{35mm}\le
      C_{M,R,S}\|\nabla g-\nabla\widetilde g\|_{L^\infty}.
\end{aligned}
\label{eq:loc-endpoint-operator-stability}
\end{equation}
The constants are unchanged if constants are added to $g$ or
$\widetilde g$.
\end{proposition}

\begin{proof}
We first argue for smooth functions.  The bounds are uniform under
mollification, so the $C^{2,1}$ statement follows at the end by approximation.
Put $z_s=z_s^g$.  The score and Hessian formulas
\eqref{eq:loc-score-formula}-\eqref{eq:loc-hessian-formula}, together with
Lemma~\ref{lem:loc-covariance}, give the uniform bound on this bounded set
\begin{equation}
    \sup_{0\le s\le S}\|z_s\|_{C^{0,1}}\le C_{M,R,S}.
    \label{eq:loc-endpoint-z-H1}
\end{equation}
The proof of Lemma~\ref{lem:loc-coefficients}, without normalizing the bounds to $1/2$ and
$2$, gives
\[
    \sup_{0\le s\le S}
    \bigl(\|c_s\|_{C^{1,1}}+\|\nabla c_s\|_{C^{0,1}}\bigr)
    \le C_{M,S}.
\]
Apply Lemma~\ref{lem:loc-linear-scale-tame} with $m=1$ to the scale equation
\eqref{eq:loc-scale-equation}, taking
\[
    A_s=2(c_s+z_s),\qquad B_s=2(\nabla c_s)^T,\qquad f_s=0.
\]
Estimate~\eqref{eq:loc-endpoint-z-H1} controls all $C^{0,1}$ coefficient bounds in
that lemma.  It follows that
\begin{equation}
    \int_0^S\|z_s^g\|_{C^{1,1}}\,\dd s
    \le C_{M,R,S}.
    \label{eq:loc-endpoint-z-H2}
\end{equation}
Averaging in $s$ gives the asserted $C^{1,1}$ bound for
$W_S[\nabla g]$.

We next prove the difference estimate.  Set
\[
    q:=g-\widetilde g,
    \qquad
    \eta:=\nabla q,
    \qquad
    v_s:=z_s^g-z_s^{\widetilde g}.
\]
For $0\le\theta\le1$, define
\[
    H_\theta:=a+\widetilde g+\theta q,
    \qquad
    F_\theta:=\nabla a+\nabla\widetilde g+\theta\eta.
\]
Differentiation of the tilted expectation gives
\[
    \frac{\dd}{\dd\theta}\E_{s,H_\theta}[F_\theta]
    =\E_{s,H_\theta}[\eta]
      +\Cov_{s,H_\theta}(F_\theta,q).
\]
The first term is bounded by $\|\eta\|_{L^\infty}$.  In the covariance term,
$F_\theta$ has uniformly bounded Lipschitz norm and
$\|\nabla q\|_{L^\infty}=\|\eta\|_{L^\infty}$.  Lemma~\ref{lem:loc-covariance}
therefore yields
\begin{equation}
    \sup_{0\le s\le S}\|v_s\|_{L^\infty}
    \le C_{M,R,S}\|\eta\|_{L^\infty}.
    \label{eq:loc-endpoint-v-Linf}
\end{equation}
This argument is valid in the OU case because $q$ has at most linear growth,
which is included in Lemma~\ref{lem:loc-covariance}.

Subtracting the two scale equations gives
\begin{equation}
\begin{aligned}
    \partial_sv_s
    &=\mathcal Gv_s+2(c_s+z_s^g)\cdot\nabla v_s
      +2(\nabla c_s)^Tv_s+2v_s\cdot\nabla z_s^{\widetilde g},
    \qquad v_0=\eta.
\end{aligned}
\label{eq:loc-endpoint-v-scale}
\end{equation}
Use the principal gradient estimate~\eqref{eq:loc-principal-gradient-est}, with
$2(c_s+z_s^g)\cdot\nabla$ kept in the principal operator.  Equations
\eqref{eq:loc-endpoint-z-H1} and~\eqref{eq:loc-endpoint-v-Linf} give
\begin{equation}
    \int_0^S\|v_s\|_{C^{0,1}}\,\dd s
    \le C_{M,R,S}\|\eta\|_{L^\infty}.
    \label{eq:loc-endpoint-v-H1}
\end{equation}
Indeed, the zeroth-order source term is bounded by
$C(\|\nabla c_s\|_{L^\infty}+\|\nabla z_s^{\widetilde g}\|_{L^\infty})
\|v_s\|_{L^\infty}$, and its integral is controlled by
\eqref{eq:loc-endpoint-z-H1}.  Averaging~\eqref{eq:loc-endpoint-v-H1} in $s$
proves the $C^{0,1}$ part of
\eqref{eq:loc-endpoint-operator-stability}.

It remains to control $E_S$.  The exact difference of the nonlinear
scale source term is
\[
\begin{aligned}
    F_s(z_s^g)-F_s(z_s^{\widetilde g})
    &=2(c_s+z_s^g)\cdot\nabla v_s
      +2v_s\cdot\nabla z_s^{\widetilde g}+2(\nabla c_s)^Tv_s.
\end{aligned}
\]
Consequently,~\eqref{eq:loc-endpoint-v-Linf} and
\eqref{eq:loc-endpoint-v-H1} imply
\begin{equation}
    \int_0^S
    \|F_s(z_s^g)-F_s(z_s^{\widetilde g})\|_{L^\infty}\,\dd s
    \le C_{M,R,S}\|\eta\|_{L^\infty}.
    \label{eq:loc-endpoint-F-difference}
\end{equation}
The operators $\mathsf K_{S-s}-I$ are bounded on $L^\infty$.  The cancellation
identity~\eqref{eq:loc-cancel-identity} therefore gives
\[
    \|N_S[\nabla g]-N_S[\nabla \widetilde g]\|_{L^\infty}
    \le C_{M,R,S}\|\eta\|_{L^\infty}.
\]
Also $(\mathsf K_S-I)/S$ is bounded on $L^\infty$.  The remaining terms in
\eqref{eq:loc-E-map} are controlled by the already proved $C^{0,1}$
difference bound for $W_S$ and by the fixed bounds on $\nabla a$ and
$\nabla^2a$. This proves the difference estimate for $E_S$. The
$L^\infty$ bound for $E_S[\nabla g]$ follows from the same cancellation formula,
\eqref{eq:loc-endpoint-z-H1}-\eqref{eq:loc-endpoint-z-H2}, and the product
bounds in~\eqref{eq:loc-E-map}.

We now remove the smoothness assumption. Choose periodic mollifications in
the torus case and Euclidean mollifications in the OU case. Normalize each
mollified target by subtracting the logarithm of its exponential integral
against $\mathsf m$. This constant tends to zero and does not change any
score operator. Choose the additive constants of the mollified primitives so that they converge
locally uniformly to $g$ and $\widetilde g$; none of the operators depends on
those constants.  We obtain smooth $a^n,g^n,\widetilde g^n$ such that
\[
    a^n\to a\quad\text{in }C^2_{\rm loc},
    \qquad
    \nabla g^n\to\nabla g,
    \quad
    \nabla\widetilde g^n\to\nabla\widetilde g
    \quad\text{locally uniformly},
\]
with the relevant $C^{2,1}$ and $C^{0,1}$ norms bounded by
$M+o(1)$ and $R+o(1)$.  In the OU case the primitives have a common linear
growth bound.  Hence all smooth estimates above hold with one constant
$C_{M,R,S}$.

For each fixed $s>0$, the heat-kernel or Mehler formulas and dominated
convergence give local uniform convergence
\[
    z_s^{g^n}\longrightarrow z_s^g,
    \qquad
    z_s^{\widetilde g^n}\longrightarrow z_s^{\widetilde g}.
\]
At $s=0$ this is the assumed convergence of the input scores.  The uniform
$C^{0,1}$ bounds pass to the limit by lower semicontinuity.  For $s>0$
the scale scores are smooth, and the same semigroup formulas identify the
limits of their spatial derivatives.  Fatou's lemma therefore gives
\[
    \int_0^S\|z_s^g\|_{C^{1,1}}\,\dd s
    \le \liminf_{n\to\infty}
       \int_0^S\|z_s^{g^n}\|_{C^{1,1}}\,\dd s.
\]
Averaging in $s$ gives local uniform convergence of the velocities.  Their
uniform $C^{1,1}$ bounds then imply, again by lower semicontinuity, that
$W_S[\nabla g]\in C^{1,1}$.  The uniform $C^{1,1}$ velocity bounds and local uniform convergence also
imply local uniform convergence of their first derivatives: on each compact
set, equicontinuity of those derivatives gives subsequential limits, and the
limit is identified as the derivative of the limiting velocity by integration
against test functions. Applying lower semicontinuity to the differences
proves the $C^{0,1}$ velocity stability estimate.

It remains to identify the remainder after cancellation.  The uniform $C^{0,1}$
bounds on the scale scores imply that their weak spatial gradients are bounded
in $L^\infty([0,S]\times\mathsf X)$.  Local uniform convergence of the scores
therefore identifies the weak-star limits of those gradients.  Since
$c_s^n,\nabla c_s^n$ converge locally uniformly to $c_s,\nabla c_s$, the nonlinear scale
source terms $F_s^n$ associated with the target $a^n$ satisfy
\[
    F_s^n(z_s^{g^n})\stackrel{*}{\rightharpoonup}F_s(z_s^g)
    \quad\text{in }L^\infty([0,S]\times\mathsf X),
\]
and similarly for the difference of two inputs.  In the OU case the weak-star
space is taken with respect to $\dd s\,\dd\gamma$; compactly supported tests
identify the limit and are dense in $L^1(\gamma)$.  The operators
$\mathsf K_{S-s}$ are contractions on $L^\infty$ and have the corresponding
$L^1$ preadjoints.  More explicitly, for an $L^1$ test field $\chi(s,x)$,
self-adjointness of $P_t$ with respect to $\dd x$ and of $Q_t$ with respect to
$\gamma$ gives
\[
    \int_0^S\!\int \chi_s\cdot\mathsf K_{S-s}F_s^n
    =\int_0^S\!\int \mathsf K_{S-s}\chi_s\cdot F_s^n,
\]
and $\mathsf K_{S-s}\chi_s$ is again an $L^1$ test field.  Consequently the
right-hand side of the cancellation identity converges in the weak-* topology of
$L^\infty$ to
\[
    \frac1S\int_0^S(\mathsf K_{S-s}-I)F_s(z_s^g)\,\dd s.
\]
This limit is the distribution $N_S[\nabla g]$ in
\eqref{eq:loc-N-def}: the undifferentiated difference of the averaged
velocities converges locally, so applying $\mathcal G$ commutes with its
distributional limit.  The uniform and
difference $L^\infty$ bounds pass to the limit, as do the remaining product
terms in~\eqref{eq:loc-E-map}.  This completes the passage to $C^{2,1}$ regularity.
\end{proof}

\subsubsection{Uniqueness in the $C^{2,1}$ solution class}
\label{subsec:loc-uniqueness}

The $L^\infty$ stability estimate yields uniqueness at $C^{2,1}$ regularity.
An $L^\infty$ comparison avoids differentiating the difference equation,
which would require one additional spatial derivative.

\begin{proposition}[Stability and uniqueness at $C^{2,1}$ regularity]
\label{prop:loc-endpoint-uniqueness}
Let $a\in C^{2,1}$ satisfy~\eqref{eq:loc-target-potential}.  Every solution in the sense of
Section~\ref{subsec:loc-endpoint-solutions} satisfies the score equation
\eqref{eq:loc-abstract-score} in distributions.  If $h$ and
$\widetilde h$ are two such solutions on $[0,T]$, with initial data
$h_{\rm in}$ and $\widetilde h_{\rm in}$, respectively, then
\begin{equation}
    \|\nabla h_t-\nabla\widetilde h_t\|_{L^\infty}
    \le e^{C_Tt}
       \|\nabla h_0-\nabla\widetilde h_0\|_{L^\infty},
    \qquad 0\le t\le T,
    \label{eq:loc-endpoint-stability}
\end{equation}
where $C_T$ depends on $S$, $\|a\|_{C^{2,1}}$, and the two score bounds
on $[0,T]$.  In particular, a solution with a given initial condition is
unique.
\end{proposition}

\begin{proof}
Let $\rho_t=e^{h_t}\rho^\star$ be a solution in the sense of
Section~\ref{subsec:loc-endpoint-solutions} and write
\[
    \dd\mathsf m=Z^{-1}e^{-V(x)}\,\dd x,
    \qquad
    V=0\ \text{on }\T^d,
    \qquad
    V(x)=\frac{|x|^2}{2}\ \text{in the OU case}.
\]
The Lebesgue density of $\rho_t$ is
$p_t=Z^{-1}e^{h_t+a-V}$.  Proposition~\ref{prop:loc-endpoint-operators} gives
\[
    W_S[\nabla h]\in L^\infty([0,T];C^{1,1}).
\]
The local continuity of the score
and the explicit tilted-expectation formula make $t\mapsto W_S[\nabla h_t]$ Borel, while
its spatial Lipschitz constant is uniformly bounded.  Thus the ODE
\[
    \dot X_t(x)=-W_S[\nabla h_t](X_t(x))
\]
has a unique global bi-Lipschitz flow.  The associated linear continuity
equation has at most one narrowly continuous finite-measure solution: for a
smooth terminal test function, composition with the backward flow solves the
backward transport equation, and insertion of this test function into the weak
formulation keeps its pairing with the measure constant. Such a backward
test is absolutely continuous in time and $C^1$ in space; it is admissible
by smoothing and approximation in the weak formulation. Bounded velocity
keeps its support in a fixed compact set over a finite time interval.
Consequently the pairing identity determines the measure uniquely. Hence
$\rho_t=(X_t)_\#\rho_0$.  Since $\rho_0$ has a strictly positive density, the
change-of-variables formula gives, for almost every $t$,
\[
    \frac{\dd}{\dd t}\log p_t(X_t(x))
       =\nabla\cdot W_S[\nabla h_t](X_t(x)).
\]
Equivalently,
\[
    \partial_t\log p_t=\nabla\cdot W_S[\nabla h_t]
       +W_S[\nabla h_t]\cdot\nabla\log p_t
\]
in distributions.  Substituting
$\log p_t=h_t+a-V-\log Z$ yields
\[
    \partial_t h_t
    =\nabla\cdot W_S[\nabla h_t]
      +(\nabla a-\nabla V)\cdot W_S[\nabla h_t]
      +\nabla h_t\cdot W_S[\nabla h_t],
\]
which is~\eqref{eq:loc-h-flow}.

Taking a spatial gradient and using the endpoint cancellation identity gives
\eqref{eq:loc-abstract-score} in distributions.  Every term on its right-hand
side belongs to $L^\infty([0,T]\times\mathsf X)$ by
Proposition~\ref{prop:loc-endpoint-operators} and the assumed score bound.  If
that right-hand side is denoted by $B$, the distributional identity
$\partial_t\nabla h=B$ implies, after changing the score on a null set in time,
\[
    \nabla h_t=\nabla h_0+\int_0^t B_\tau\,\dd\tau
    \quad\text{in }L^\infty.
\]
The bound
\[
    \|\nabla h_t-\nabla h_s\|_\infty
       \le\int_s^t\|B_\tau\|_\infty\,\dd\tau
\]
extends to every $s,t$ using the assumed local continuity.
Thus the score is Lipschitz in time with values in $L^\infty$; its uniform
spatial Lipschitz bound is part of the regularity required in
Section~\ref{subsec:loc-endpoint-solutions}.

Set
\[
    \xi=\nabla h,
    \qquad
    \widetilde\xi=\nabla\widetilde h,
    \qquad
    W=W_S[\xi],
    \qquad
    \widetilde W=W_S[\widetilde\xi],
    \qquad
    \delta=\xi-\widetilde\xi.
\]
Subtracting the two score equations gives
\begin{equation}
    \partial_t\delta=(W\cdot\nabla)\delta+G,
    \label{eq:loc-endpoint-difference-equation}
\end{equation}
where
\[
\begin{aligned}
    G={}&((W-\widetilde W)\cdot\nabla)\widetilde\xi
       +(\nabla W)^T\delta
       +(\nabla W-\nabla\widetilde W)^T\widetilde\xi \\
       &+E_S[\xi]-E_S[\widetilde\xi].
\end{aligned}
\]
Let $R_T$ bound the two $C^{0,1}$ score norms on $[0,T]$.
Proposition~\ref{prop:loc-endpoint-operators} and the uniform velocity bound
give
\begin{equation}
    \|G_t\|_{L^\infty}
    \le C_{T}\|\delta_t\|_{L^\infty}
    \quad\text{for almost every }t\in[0,T].
    \label{eq:loc-endpoint-G-bound}
\end{equation}
Let $X(t;x)$ be the global characteristic flow of $-W$.  The chain rule along
$X$, valid for the space-time Lipschitz representative of $\delta$, and
\eqref{eq:loc-endpoint-difference-equation} give
\[
    \frac{\dd}{\dd t}|\delta_t(X(t;x))|
    \le C_T\|\delta_t\|_{L^\infty}
\]
for almost every $t$ and almost every starting point $x$.
Bi-Lipschitz characteristic maps preserve null sets, which justifies using
the almost-everywhere differential equation along these curves.
Because the characteristic map $X(t;\cdot)$ is onto, taking the essential
supremum over $x$, integrating in time, and applying Gronwall's lemma proves
\eqref{eq:loc-endpoint-stability}.

If the initial scores agree, then $\nabla h_t=\nabla\widetilde h_t$ for every
$t$.  Their difference is therefore spatially constant.  Since both solutions are probability measures,
$\int e^{h_t}\,\dd\rho^\star=1$, and that constant is zero.
Thus $h=\widetilde h$ and $\rho=\widetilde\rho$.
\end{proof}

\subsection{Completion of the convergence theorems}\label{subsec:loc-finish}

We complete the proofs of Theorems~\ref{thm:local-torus-heat} and
\ref{thm:local-ou}.  Let $c_{d,S}:=\lambda_S/2$ and choose the smallness constant
$\varepsilon_S$ in both theorems so that
\[
    \varepsilon_S\le \frac12\varepsilon_*
\]
and so that $\varepsilon_S$ is no larger than all other smallness constants
required above.

If $a$ is smooth, and in the OU case every positive-order derivative of $a$ is
bounded, Proposition~\ref{prop:loc-smooth-global} gives the unique global
smooth normalized solution and
\[
    \|\nabla h_t\|_{C^{0,1}}
    \le e^{-c_{d,S}t}\|\nabla a\|_{C^{0,1}}.
\]
Lemma~\ref{lem:loc-entropy-score} then gives
\[
    \KL(\rho_t\mid\rho^\star)
    \le C_S e^{-2c_{d,S}t}\|\nabla a\|_{C^{0,1}}^2.
\]

\paragraph{Step 1: approximation of the target.} Let $a\in C^{2,1}$ satisfy the hypotheses of either theorem.  Choose a
standard mollifier $\eta_\delta$, periodic in the torus case and Euclidean in
the OU case, and set
\begin{equation}
    a^\delta:=\eta_\delta*a
             -\log\int e^{\eta_\delta*a}\,\dd\mathsf m.
    \label{eq:loc-mollified-a}
\end{equation}
Then $a^\delta\in C^\infty$, $\int e^{a^\delta}\,\dd\mathsf m=1$, and in the
OU case every positive-order derivative of $a^\delta$ is bounded.  Since $a$
is bounded and uniformly continuous, $\eta_\delta*a\to a$ uniformly;
hence the subtracted normalization constant tends to zero.  Moreover,
\begin{equation}
    \limsup_{\delta\downarrow0}\|a^\delta\|_{C^{2,1}}
    \le \|a\|_{C^{2,1}},
    \qquad
    \limsup_{\delta\downarrow0}\|\nabla a^\delta\|_{C^{0,1}}
    \le \|\nabla a\|_{C^{0,1}},
    \label{eq:loc-mollifier-limsup-a}
\end{equation}
while
\begin{equation}
    \partial^\kappa a^\delta\longrightarrow\partial^\kappa a
    \quad\text{locally uniformly for every }|\kappa|\le2.
    \label{eq:loc-mollifier-local-convergence}
\end{equation}
For all sufficiently small $\delta$, one has
$\|a^\delta\|_{C^{2,1}}\le\varepsilon_*$.  Let $h_t^\delta$ be the
corresponding global smooth logarithmic density ratio relative to
$\rho_\delta^\star=e^{a^\delta}\mathsf m$. Write $W_S^\delta$ for the
velocity map with target potential $a^\delta$.  Proposition~\ref{prop:loc-smooth-global}
gives, uniformly in $\delta$,
\begin{equation}
    \|\nabla h_t^\delta\|_{C^{0,1}}
    \le e^{-c_{d,S}t}\|\nabla a^\delta\|_{C^{0,1}},
    \qquad t\ge0.
    \label{eq:loc-uniform-delta-score}
\end{equation}

\paragraph{Step 2: compactness.} Fix $T<\infty$.  Equation~\eqref{eq:loc-uniform-delta-score} gives a uniform
spatial $C^{0,1}$ bound for the scores on $[0,T]$.  The uniform bound~\eqref{eq:loc-endpoint-one-input}, applied to the bounded
family $a^\delta$, and the differentiated equation
\eqref{eq:loc-abstract-score} give
\begin{equation}
    \sup_\delta\sup_{0\le t\le T}
    \|\partial_t\nabla h_t^\delta\|_{L^\infty}<\infty.
    \label{eq:loc-uniform-time-score}
\end{equation}
Normalize each smooth solution by
\[
    \int e^{h_t^\delta}\,\dd\rho_\delta^\star
    =\int e^{h_t^\delta+a^\delta}\,\dd\mathsf m=1.
\]
On the torus, the uniform score bound controls the oscillation of
$h_t^\delta$, and the normalization fixes its additive constant.  In the OU
case, if $\Lip(h)\le L$, $\|a\|_{L^\infty}\le M$, and
$\int e^{h+a}\,\dd\gamma=1$, then
\begin{equation}
    |h(x)|\le C_{L,M,d}+L|x|,
    \qquad x\in\R^d.
    \label{eq:loc-ou-linear-growth-h}
\end{equation}
Indeed, $h(0)-L|x|\le h(x)\le h(0)+L|x|$, and
\[
    1\ge e^{h(0)-M}\int e^{-L|x|}\,\dd\gamma(x),
    \qquad
    1\le e^{h(0)+M}\int e^{L|x|}\,\dd\gamma(x)
\]
bound $h(0)$ from above and below.  Thus $h_t^\delta$ is locally uniformly
bounded in the OU case as well.  The scalar equation~\eqref{eq:loc-h-flow},
together with the local bounds on $W_S^\delta[\nabla h_t^\delta]$ and its first
spatial derivatives, gives locally uniform bounds on
$\partial_t h_t^\delta$.

Applying Arzel\`a-Ascoli on $[0,N]\times K_N$, where $(K_N)_{N\ge1}$
exhausts the state space, and then taking a diagonal subsequence yields a
function $h$ defined for all $t\ge0$ such that, for every $T<\infty$ and
every compact $K$ in the state space,
\begin{equation}
    h^\delta\longrightarrow h,
    \qquad
    \nabla h^\delta\longrightarrow\nabla h
    \quad\text{uniformly on }[0,T]\times K
    \label{eq:loc-endpoint-compactness}
\end{equation}
along that subsequence.

\paragraph{Step 3: identification of the nonlocal velocity and preservation of mass.}
We next pass to the limit in the velocity.  Let
$\psi_s^\delta$ denote~\eqref{eq:loc-psi-phi-w} with $a^\delta$ in place of
$a$, and put
\[
    z_{s,t}^\delta:=\nabla\psi_s^\delta[h_t^\delta],
    \qquad
    z_{s,t}:=\nabla\psi_s[h_t].
\]
For fixed $s>0$, the tilted-expectation formula~\eqref{eq:loc-score-formula} and
\eqref{eq:loc-endpoint-compactness}-\eqref{eq:loc-mollifier-local-convergence}
imply
\[
    \sup_{0\le t\le T}\sup_{x\in K}
    |z_{s,t}^\delta(x)-z_{s,t}(x)|\longrightarrow0
\]
for every compact $K$.  In the OU case, the linear-growth estimate gives the
uniform integrable majorant $Ce^{C|z|}$ in Mehler's formula; the score factors
are uniformly bounded, and the denominators converge to strictly positive
limits.  At $s=0$ the convergence is precisely the convergence of the scores.
Estimate~\eqref{eq:loc-endpoint-z-H1}, uniformly for the bounded family of
targets $a^\delta$, gives a uniform $L^\infty$ bound in $s$.  Dominated
convergence therefore yields
\begin{equation}
    W_S^\delta[\nabla h^\delta]\longrightarrow W_S[\nabla h]
    \quad\text{locally uniformly on }[0,T]\times\mathsf X.
    \label{eq:loc-endpoint-W-convergence}
\end{equation}

Let $\varphi$ be a smooth compactly supported test function; on the torus,
compact support is unnecessary.  The smooth solutions satisfy
\[
    \int\varphi e^{h_t^\delta+a^\delta}\,\dd\mathsf m
    -\int\varphi\,\dd\mathsf m
    =-\int_0^t\int\nabla\varphi\cdot W_S^\delta[\nabla h_\tau^\delta]
      e^{h_\tau^\delta+a^\delta}\,\dd\mathsf m\,\dd\tau.
\]
The local uniform convergences above pass this identity to the limit.  In the
OU case, uniformly for $0\le t\le T$,
\[
    e^{h_t^\delta+a^\delta}\le C_Te^{C_T|x|},
\]
and the right-hand side is integrable against $\gamma$.  Dominated convergence
therefore also gives
\[
    \int e^{h_t+a}\,\dd\gamma=1.
\]
The same majorant shows that $t\mapsto e^{h_t+a}\gamma$ is narrowly
continuous.  On the torus this is immediate from compactness.  Since
$h_0^\delta=-a^\delta$, one has $h_0=-a$.  Thus
$\rho_t=e^{h_t}\rho^\star$ satisfies the weak formulation
\eqref{eq:loc-endpoint-weak-form} and the required normalization.

\paragraph{Step 4: passage to the limit and uniqueness.}
The uniform score estimate passes to the limit by lower semicontinuity.  More
precisely, local uniform convergence and a uniform $C^{0,1}$ bound imply
\[
    \|f\|_{L^\infty}\le\liminf_n\|f_n\|_{L^\infty},
    \qquad
    \Lip(f)\le\liminf_n\Lip(f_n).
\]
The first inequality follows on compact balls and then globally by increasing
the ball; the second follows by passing to the limit for pairs of points.
Applying this to $f_n=\nabla h_t^{\delta_n}$ and using
\eqref{eq:loc-mollifier-limsup-a} gives
\begin{equation}
    \|\nabla h_t\|_{C^{0,1}}
    \le e^{-c_{d,S}t}\|\nabla a\|_{C^{0,1}},
    \qquad t\ge0.
    \label{eq:loc-endpoint-score-decay}
\end{equation}
In particular, the limit is a global solution in the sense of
Section~\ref{subsec:loc-endpoint-solutions}.
  
Proposition~\ref{prop:loc-endpoint-uniqueness} shows that this solution is
unique.  It also shows that the limit is independent of the mollifier and of
the extracted subsequence.  Estimate~\eqref{eq:loc-endpoint-score-decay} is
\eqref{eq:main-heat-score-decay} in the torus case and 
\eqref{eq:main-ou-score-decay} in the OU case, after allowing the theorem
constant $C_S$.

Finally, Lemma~\ref{lem:loc-entropy-score} applies directly to the normalized limiting
density and gives
\[
    \KL(\rho_t\mid\rho^\star)
    \le C_S\|\nabla h_t\|_{C^{0,1}}^2
    \le C_Se^{-2c_{d,S}t}\|\nabla a\|_{C^{0,1}}^2.
\]
This proves~\eqref{eq:main-heat-entropy-decay} and
\eqref{eq:main-ou-entropy-decay}, and completes the proof of both theorems.


\end{document}